\documentclass[12pt]{article}

\usepackage[utf8]{inputenc}

\usepackage{url}

\usepackage{hyperref} 
\usepackage{amsmath,amsthm,amsfonts,amssymb,mathrsfs,natbib}
\usepackage{anysize}
\usepackage{enumerate,enumitem}
\usepackage{bbm} 
\usepackage{color}
\usepackage{tikz} 
\usepackage{tabularx}
\usepackage{hhline}
\usepackage{algorithm}
\usepackage{algorithmic}
\usepackage{array}      
\usepackage{booktabs}   
\usepackage{subcaption}
\usepackage{caption}

\usepackage{graphicx,wrapfig}
\usepackage[normalem]{ulem} 
\usepackage[skins]{tcolorbox}

\def\argmin{\operatorname{\arg\min}}

\newcommand{\be}{\begin{equation}}
\newcommand{\ee}{\end{equation}}
\newcommand{\br}{\begin{remmark}}
\newcommand{\er}{\end{remmark}}
\newcommand{\bq}{\begin{qu}}
\newcommand{\eq}{\end{qu}}
\newcommand{\bn}{\begin{enumerate}}
\newcommand{\en}{\end{enumerate}}
\newcommand{\bi}{\begin{itemize}}
\newcommand{\ei}{\end{itemize}}
\newcommand{\beas}{\begin{eqnarray*}}
\newcommand{\eeas}{\end{eqnarray*}}
\newcommand{\bea}{\begin{eqnarray}}
\newcommand{\eea}{\end{eqnarray}}

\newcommand{\R}{\mathbb{R}}
\newcommand{\N}{\mathbb N}

\def\CA{\mathcal A} 
\def\CB{\mathcal B} 
\def\CH{\mathcal H} 
\def\CC{\mathcal C} 
\def \CF{\mathcal F}
\def\CG{\mathcal G} 
\def\CL{\mathcal L}
\def\CN{\mathcal N} 
\def\CP{\mathcal P} 
\def\CR{\mathcal R} 
\def\CU{\mathcal{U}} 

\def\EE{\mathbb{E}} 
\def\RR{\mathbb{R}} 
\def\NN{\mathbb{N}} 

\def\hflhbl{\hat{f}_{\lambda}^{\hat{B}_{d_*}}}

\def \md {\mathrm{d}} 

\theoremstyle{plain}
\newtheorem{thm}{Theorem}
\newtheorem{lem}{Lemma}

\theoremstyle{definition}
\newtheorem{rem}{Remark}

\newtheorem{ass}{Assumption}
\newtheorem{qu}{Question}

\newcommand{\thmref}[1]{Theorem~\ref{#1}}
\def\Secref#1{Section~\ref{#1}}

\newcommand{\matK}{\hat{K}}

\begin{document}

\title{Learning Multi-Index Models with Hyper-Kernel Ridge Regression}


\author{Shuo Huang \thanks{Istituto Italiano di Tecnologia, Genoa, Italy. Email: \href{mailto:shuo.huang@iit.it}{shuo.huang@iit.it}}$\quad$
Hippolyte Labarri\`ere \thanks{MaLGa - DIBRIS - Universit\`a di Genova, Genoa, Italy. 
Email: \href{mailto:hippolyte.labarriere@edu.unige.it}{hippolyte.labarriere@edu.unige.it}
}$\quad$
Ernesto De Vito \thanks{MaLGa - DIMA - Universit\`a di Genova, Genoa, Italy.
Email:\href{mailto:ernesto.devito@unige.it}{ernesto.devito@unige.it}
}\\
Tomaso Poggio \thanks{CBMM - Massachusets Institute of Technology, Cambridge, MA, USA.
Email:\href{mailto:tp@csail.mit.edu}{tp@csail.mit.edu}
} $\quad$
Lorenzo Rosasco \thanks{MaLGa – DIBRIS – Università di Genova, Genoa, Italy; Istituto Italiano di Tecnologia, Genoa, Italy. 
Email: \href{mailto:lrosasco@mit.edu}{lrosasco@mit.edu} 
}
}
\date{}

\maketitle

\begin{abstract}

Deep neural networks excel in high-dimensional problems, outperforming models such as kernel methods, which suffer from the curse of dimensionality. However, the theoretical foundations of this success remain poorly understood. We follow the idea that the compositional structure of the learning task is the key factor determining when deep networks outperform other approaches.
Taking a step towards formalizing this idea, we consider a simple compositional model, namely
the multi-index model (MIM). In this context, we introduce and study hyper-kernel ridge regression (HKRR),
an approach blending neural networks and kernel methods. Our main contribution is a sample complexity result demonstrating that HKRR can adaptively learn MIM, overcoming the curse of dimensionality. Further, we exploit the kernel nature of the estimator to develop ad hoc optimization approaches. Indeed, we contrast alternating minimization and alternating gradient methods both theoretically and numerically. These numerical results complement and reinforce our theoretical findings.

\end{abstract}

\section{Introduction}

The search for principles underlying the success of deep networks in learning from high-dimensional problems has been the subject of much interest. At least two ideas have a long history. The first is \textbf{invariance}. Deep architectures emerge from the need to derive models insensitive to transformations that are uninformative for the task at hand. Computational primitives such as filtering and pooling at different scales can be understood as implementing these ideas. This perspective traces back to early work in computer vision \citep{fukushima1980neocognitron, lecun1989backpropagation}, itself motivated by ideas in neuroscience \citep{hubel1962receptive, riesenhuber1999hierarchical}, and we refer to \cite{serre2007feedforward, mallat2012group} for examples of more recent contributions in this line of work. A second idea is \textbf{compositionality}. High-dimensional data often have a hierarchical structure where parts at different scales interact. Language provides a natural example, with its structure in letters, syllables, words, and sentences. Deep architectures can then be designed to exploit this structure. These ideas, which go back at least to \cite{bienenstock1998compositionality, yuille2006vision}, provide another perspective on algorithmic developments such as convolutions \citep{lecun2002gradient} and attention mechanisms \citep{vaswani2017attention}. Ultimately, the relevance of either one of these principles relies on their ability to reduce the need for data, thus translating into more successful learning schemes. The study of sample complexity in statistical learning theory provides a framework within which this intuition can be formalized and tested \citep{vapnik2013nature}.

Classic sample complexity results highlight the role of data dimension and function smoothness. In the absence of any assumption, no sample complexity results can be derived \citep{vapnik2013nature, devroye2013probabilistic}. Assuming the task of interest is described by Lipschitz functions leads to sample complexity scaling exponentially with the dimension of the input data---the so-called curse of dimensionality \citep{donoho2000high}. Such dependence can be alleviated if further smoothness is assumed, yielding sample complexity that depends exponentially on the ratio of dimension to smoothness \citep{stone1982optimal}. Both these classes of problems (Lipschitz and smooth Sobolev functions) can be learned by a variety of learning approaches, including kernel methods and neural networks, hence not explaining the better performance of the latter on high-dimensional problems. Starting from the seminal work in \cite{barron2002universal}, this observation has led to investigating how to characterize the class of problems where deep networks excel; see, e.g., \cite{poggio2017and} for a recent account. 

Circling back to the initial discussion, the role of invariance in sample complexity has been discussed in \cite{poggio2017and} and analyzed, for example, in \cite{mei2021learning} in the context of group transformations. However, invariance alone seems insufficient to account for the striking empirical performance observed in practice. A functional viewpoint on compositionality was proposed in \cite{mhaskar2017and} and further developed in \cite{dahmen2025compositional} from the perspective of approximation theory. Sample complexity bounds were derived in \cite{schmidt2020nonparametric,kohler2021rate}, laying the groundwork for a theoretical understanding of compositional structure. It is within this line of work that our contribution is situated.  Our study is further motivated by the work of \cite{radhakrishnan2022feature}, which points to a simpler compositional structure and proposes a kernel-based approach to learn it, called Recursive Feature Machine (RFM), drawing on ideas from sufficient dimensionality reduction \citep{fukumizu2009kernel}. As we discuss next, we propose an alternative approach within the same context.

The approach we study blends ideas from kernel methods and neural networks. It draws inspiration from \cite{poggir90}, where an extension of radial basis function networks (RBF), called hyper-RBF, was proposed. Instead of a single kernel and its RKHS, we consider a family of kernels and their corresponding RKHSs. Each kernel is obtained by composing a fixed common kernel with a linear transformation that maps inputs to a lower-dimensional space. A solution is then obtained through regularized empirical risk minimization with least squares. For any fixed transformation, the approach reduces to kernel ridge regression (KRR). But now, rather than being fixed, the best transformation is learned during training. The resulting method is called hyper-kernel ridge regression (HKRR), and reduces to hyper-RBF when radial kernels are used. It can be seen as a special form of neural network or as a kernel method augmented with a built-in linear representation learning step. In particular,  classic approximation schemes for kernel methods,  such as Nystr\"om approximations \citep{rudi2015less}, can be exploited.

HKRR provides a natural framework for learning multi-index models (MIMs), given their structure as the composition of a linear transformation and a smooth nonlinear function. Our main contribution is the characterization of the sample complexity of HKRR for learning MIMs. We show that in this case, the dependence is exponential in the ratio between the true transformation dimension and the smoothness, and only polynomial in the input data dimension. We further show that the transformation dimension does not need to be known a priori but can be tuned by hold-out cross-validation, preserving the same sample complexity up to logarithmic factors. We complete our statistical analysis by showing that the HKRR estimator can be compressed using Nystr\"om subsampling  \citep{rudi2015less}, without degrading the sample complexity. The proofs largely draw on techniques developed for kernel methods, extended to handle the compositional nature of hyper-kernels. A second contribution is to investigate the solution of the HKRR optimization problem both theoretically and numerically. In particular, we contrast two different approaches. The first leverages the connection to KRR and alternates closed-form updates for the estimator (given a transformation) with transformation updates via gradient descent (VarPro), in the spirit of variable projection methods \citep{golub1973}. The second strategy alternates gradient descent (AGD) steps, akin to the PALM algorithm in \cite{bolte2014proximal}. The HKRR optimization problem is non-convex, but both strategies can be shown to converge to a critical point. Numerically, however, the AGD approach appears more stable and ultimately outperforms VarPro. We attribute this behavior to the nonlocal nature of the latter, as we illustrate numerically. Overall,  our results show that HKRR can be viewed as a useful augmentation of kernel methods, while providing a sound algorithmic approach to study simple compositionality and representation learning models.


Some notation and background are given in \Secref{sec:pre}.
\Secref{sec:mim} introduces the HKRR problem and two algorithms, VarPro and AGD.
\Secref{sec:thms} presents the sample complexity of HKRR and the convergence analysis of both algorithms.
Experimental results and conclusions are reported in \Secref{sec:experiment} and \Secref{sec:conclusion}.

\section{Background}\label{sec:pre}

In this section, we collect some basic definitions and notation. 
\paragraph{Statistical learning and sample complexity.} Let $\rho$ be a joint probability distribution on $X \times Y \subset \R^D \times \R$.  The learning problem with the square loss consists in minimizing, over all measurable functions, 
the expected risk
\begin{equation*}
    \CR(f)=\EE[(f(x)-y)^2]
\end{equation*}
given  $(x_i,y_i)_{i=1}^m\overset{\mathrm{i.i.d.}}{\sim}\rho^m$.  
The quality of a learning solution $\hat f$ is measured by the 
 {excess risk} $\CR(\hat{f})-\CR(f_*)$, where  $f_*$ denotes a  risk minimizer. 
 For the square loss, a minimizer is the so-called regression function defined as $f_*(x)=
 \EE[y~|~x]$ almost surely. 
The sample complexity of a learning algorithm is the number of samples required by a corresponding empirical solution to achieve a prescribed accuracy with a prescribed confidence level. More precisely, given $\epsilon > 0$ and $\delta \in [0,1]$, we say that a procedure outputting solutions $\hat{f}$ given $m$ points has sample complexity $m(\epsilon, \delta)$, if for all $m \ge m(\epsilon, \delta)$,
\(
\CR(\hat{f}) - \CR(f_*) \le \epsilon
\)
with probability at least $1 - \delta$. Here, $\epsilon$ and $1 - \delta$ are the accuracy level and the confidence level, respectively. The function $m(\epsilon, \delta)$ can typically be inverted to express the results in terms of error bounds. Given ${  m }$ and $\delta$, an error bound is a function $\epsilon( {  m}, \delta)$ such that
\(
\CR(\hat{f}) - \CR(f_*) \le \epsilon({  {m}}, \delta),
\)
with probability at least $1 - \delta$. In the following, we will take this latter point of view and review classical algorithms relevant to our study.

\paragraph{ERM and kernel methods.} Fixed a hypothesis space $\CH$ of measurable functions $f: X \to \R$,  the {empirical risk minimization} (ERM) over  $\CH$ is given by 
 \(\hat{f} = \mathop{\argmin}_{f\in \CH}\widehat{\CR}(f)=\mathop{\argmin}_{f\in \CH}
\frac1m\sum_{i=1}^m\bigl(f(x_i)-y_i\bigr)^2. \)
In particular, kernel methods consider 
$\CH$ to be an RKHS, that is a Hilbert space of functions with a reproducing kernel $k:X\times X\to\R$ satisfying 
$ k_x=k(\cdot,x)\in\CH,$
and \(
f(x)=\langle f,k_x\rangle_\CH ,\)
for all $x\in X$ and $f\in \CH$ 
\citep{aronszajn1950theory}. 
Further,  KRR corresponds to minimizing the regularized empirical risk 
\begin{equation}\label{equ:KRR1}
\hat{f}_\lambda =  \underset{f\in \CH}{\argmin} \, \widehat{\CR}_\lambda(f), \quad \widehat{\CR}_\lambda(f) = \frac{1}{m} \sum_{i=1}^m (f(x_i)-y_i)^2 +\lambda \|f\|^2_{\CH}, \quad \lambda>0.
\end{equation} 
By the representer theorem \citep{scholkopf2001generalized}, 
\(
\hat f_\lambda=\sum_{i=1}^m\alpha_i\,k(x_i,\cdot), 
\)
so that  KRR reduces to a finite-dimensional problem
\begin{equation}\label{equ:KRR2}
    \alpha^* := \underset{\alpha\in\R^m}{\argmin}\frac{1}{m}\|\matK\alpha-\mathbf{y}\|^2+\lambda\alpha^T\matK\alpha,
\end{equation}
where $\mathbf{y}=(y_i)_{i=1}^m \in \R^m$ and  $\matK\in\R^{m\times m}$ with $(\hat{K})_{i,j}= k(x_i,x_j)$ is the empirical kernel matrix. 
More efficient computation is possible using Nystr{\"o}m approximation
\citep{rudi2015less}, which considers $\widetilde{m}<m$ inducing points
$(\widetilde{x}_i)_{i=1}^{\widetilde{m}}
\subset (x_i)_{i=1}^{m}$ and restricts the estimator to the
sample-dependent subspace of functions of the form
\(
f(\cdot)
=
\sum_{i=1}^{\widetilde{m}}
\widetilde{\alpha}_i k(\widetilde{x}_i,\cdot).
\)
The Nystr{\"o}m KRR is then given by
\begin{equation}
\label{equ:nys}
\widetilde{\alpha}^{*}
:=
\underset{\widetilde{\alpha}\in\R^{\widetilde{m}}}{\argmin}
\frac{1}{m}
\left\|
\matK_{m\widetilde{m}}\widetilde{\alpha}-\mathbf{y}
\right\|^{2}
+
\lambda
\widetilde{\alpha}^{T}
\matK_{\widetilde{m}\widetilde{m}}
\widetilde{\alpha}.
\end{equation}
Here,
$(\matK_{m\widetilde{m}})_{i,j}
= k(x_i,\widetilde{x}_j)$ and
$(\matK_{\widetilde{m}\widetilde{m}})_{i,j}
= k(\widetilde{x}_i,\widetilde{x}_j)$.
Various approaches also represent regressors using
finite kernel expansions generated from the sample, but employ
alternative regularization schemes. Examples include direct penalties
on the expansion coefficients
\citep{sun2011least,shi2013learning}, geometry-aware graph penalties
\citep{belkin2006manifold}, and mixed-norm penalties over component
RKHSs in multiple-kernel learning \citep{kloft2011lp}.

It is useful to contrast kernel methods with classic one-hidden-layer neural networks.
 \begin{rem}[Neural and RBF networks]
One-hidden-layer neural networks consider functions of the form 
$
f(x)= \sum_{j=1}^u c_i \sigma(w_j^\top x+b_j), 
$
where $\sigma:\R\to \R$ is a nonlinearity, e.g. the ReLU $\sigma(z) = \max\{0,z\}$, and $c_j, b_j\in \R
, w_j\in \R^D$, $j=1, \dots, u$  are parameters to be determined. Each term $\sigma(w_j^\top x+b_j)$ is called a neuron, $u$ is the number of neurons/units,  and $(w_j, b_j)_j$ are called hidden weights. Radial basis function (RBF) networks consider functions of the form 
\(
f(x)= \sum_{j=1}^u c_i \phi(\|w_j- x\|), 
\)
where $\phi:\R\to \R$ is a nonlinearity, e.g. the Gaussian $\phi(z) =  e^{-z}$, and again $c_j\in \R, w_j\in \R^D$, $j=1, \dots, u$  are  parameters to be determined.
 \end{rem}

\section{Hyper-kernel ridge regression}\label{sec:mim}

In this section, we describe HKRR, an approach blending ideas from kernel methods and neural networks. HKRR is based on regularized ERM like KRR, but considers a class of functions defined by a family of parameterized kernels, rather than one fixed kernel. Similar to neural networks, the solution is a linear combination of nonlinearities with parameters to be determined during training. The corresponding optimization problem is nonconvex, but its structure suggests ad-hoc gradient approaches.


\subsection{Hyper-kernel ridge regression}

Let \( k:\R^d \times \R^d \to \R \) be a fixed ``mother" reproducing kernel with associated RKHS  \( \CH_k \).
Define
$
\CB = \{B \in \R^{d \times D} : \|B\|_\infty \le 1\},
$
where $d < D$ and
\(
\|B\|_\infty := \sup_{\|x\| \leq 1} \|Bx\|
\). A hyper-kernel \( k_B: \R^D \times \R^D \to \R \) is defined by the composition of the kernel \( k \) with a linear map \( B \in \CB \), namely,
\[
k_B(x, x') = k\bigl(Bx, Bx'\bigr), \quad x, x' \in \R^D.
\]
The RKHS with reproducing kernel \( k_B \) is denoted by \( \CH_B \) for each \( B \in \CB \). Note that our definition of the hyper-RKHS differs from that in \cite{liu2021generalization}.
The intuition is that each map \( B \) provides a low-dimensional linear representation of the data, while the mother kernel \( k \) defines a space of nonlinear functions on this reduced space. Considering hyper-kernels allows us to learn an estimator that composes an optimal linear representation and a corresponding nonlinear function.  HKRR achieves this by solving the regularized ERM problem
\begin{equation}\label{equ:HKRR}
\begin{split}
     &\min_{B \in \CB} \min_{f \in \CH_B} \widehat{\CR}_\lambda(f), 
     \quad 
     \quad 
     \widehat{\CR}_\lambda(f) =\frac{1}{m} \sum_{i=1}^m (f(x_i)-y_i)^2 + \lambda \|f\|_{\CH_B}^2 .
\end{split}
\end{equation}
For any given $B \in \CB$, the inner optimization over  $\CH_B$ is a standard KRR problem (\eqref{equ:KRR1}) with kernel $k_B$.  As discussed in Section~\ref{sec:pre}, the problem is strongly convex and admits a unique minimizer
$\hat{f}_\lambda^B =   \mathop{\argmin}_{f \in \CH_B} \widehat{\CR}_\lambda(f),$ 
which can be computed using the representer theorem (\eqref{equ:KRR2}), 
and more efficiently via a Nystr\"om approximation~(\eqref{equ:nys}). The outer optimization problem over $\CB$ is non-convex and corresponds to 
$$
\underset{B \in \CB} {\argmin}\,\hat{H}_\lambda(B), \quad \quad \hat{H}_\lambda(B):=  \min_{f \in \CH_B} \widehat{\CR}_\lambda(f)= 
\frac{1}{m}\sum_{i=1}^m\left(\hat{f}_\lambda^B(x_i) -y_i\right)^2 +\lambda \|\hat{f}_\lambda^B\|_{\CH_B}^2.
$$
If  $\hat{B}_d$  is a solution of the above problem, then the HKRR estimator is 
\(\hat{f}_{\lambda}^{\hat{B}_d}\) and relies on the choice  of the mother kernel $
k$, the regularization parameter $\lambda$, and the dimension $d$ of the  linear maps in $\CB$. 
In \Secref{sec:experiment}, we will investigate how these choices influence the corresponding learning performances.
We first discuss the practical computation of the HKRR estimator.

\subsection{Computing an HKRR solution}\label{sec:algo}
As already mentioned, the representer theorem~(\ref{equ:KRR2}) allows us to reduce problem~(\ref{equ:HKRR}) to a finite-dimensional optimization. 
In practice, we adopt the Nystr\"om approximation as \eqref{equ:nys}, 
where $(\tilde{x}_i)_{i=1}^{\tilde m}$ are sampled uniformly without replacement 
from the training set. This procedure is referred to as the plain Nystr\"om method~\citep{rudi2015less}, 
and leads to
\begin{equation}\label{equ:opt_pb}
    \min_{B \in \CB} \min_{\alpha \in \mathbb{R}^{\tilde m}} 
    \hat{\CL}(B,\alpha), 
    \quad 
    \hat{\CL}(B,\alpha) = \frac{1}{m}\big\|\matK^B_{m\tilde m}\alpha - \mathbf{y}\big\|^2 
    + \lambda\, \alpha^\top \matK^B_{\tilde m\tilde m} \alpha,
\end{equation}
where $\matK^B_{m\tilde m}\in\mathbb{R}^{m\times \tilde m}$ and 
$\matK^B_{\tilde m\tilde m}\in\mathbb{R}^{\tilde m\times \tilde m}$ are defined by
\(
(\matK^B_{m\tilde m})_{i,j} = k(Bx_i,B\tilde{x}_j),
(\matK^B_{\tilde m\tilde m})_{i,j} = k(B\tilde{x}_i,B\tilde{x}_j).
\)
Let $\hat f_\lambda^{\hat B_{d,\tilde m}}$ denote the solution of \eqref{equ:opt_pb}, 
with the index $\tilde m$ highlighting the use of the Nystr\"om approximation.

We next discuss some aspects of  HKRR optimization and refer to  Appendix \ref{sec:opt_details}, and in particular to Lemma~\ref{lem:Lalpha}, for details. We begin by noting that, for each $B \in \CB$, the inner minimization admits an explicit solution
\begin{equation}\label{equ:cfs_KRR}
    \alpha(B) =
    \left((\matK^B_{m\tilde m})^\top \matK^B_{m\tilde m} + \lambda m\, \matK^B_{\tilde m\tilde m} \right)^{-1} (\matK^B_{m\tilde m})^\top \mathbf{y}.
\end{equation}
Plugging this expression into~\eqref{equ:nys}, we have for each $B \in \CB$ that
\begin{equation}\label{equ:H}
\hat{H}_\lambda(B) = \underset{\alpha \in \mathbb{R}^{\tilde m}}{\min}~\hat{\CL}(B,\alpha) = \frac{1}{m} \mathbf{y}^\top \mathbf{y} - \frac{1}{m} \mathbf{y}^\top \matK^B_{m\tilde m} \left((\matK^B_{m\tilde m})^\top \matK^B_{m\tilde m} + \lambda m\, \matK^B_{\tilde m\tilde m} \right)^{-1} (\matK^B_{m\tilde m})^\top \mathbf{y}.
\end{equation}
If the kernel is smooth, then \( \hat{H}_\lambda \) is differentiable, and so is \( \hat{\CL}(B,\alpha) \) for any \( \alpha \in \mathbb{R}^{\tilde m} \), and hence \( \hat{\CL} \) itself. However, \( \hat{H}_\lambda \) is not convex, and neither is \( \hat{\CL} \). We will see that, if the mother kernel \( k \) is analytic, then  \( \hat{H}_\lambda \)  satisfies the Kurdyka–\L{}ojasiewicz property~\cite{attouch2013convergence}, which 
will allow the derivation of some optimization guarantees; see Section~\ref{subsec:opt} (Theorem~\ref{thm:cvg_informal}).

Given the above discussion, we next propose two methods to compute an (Nystr\"om) HKRR solution. The first method is \textbf{ Variable Projection (VarPro)}, see Algorithm~\ref{alg:MIGD}. It exploits the closed-form solution to update $\alpha$ (see ~\eqref{equ:cfs_KRR}), while applying gradient descent steps on $B$ to minimize $\hat{H}(B)$, see~\eqref{equ:H}.  This approach is well known in the optimization literature~\citep{golub1973,golub2003separable}, and allows the use of other optimization schemes such as L-BFGS~\citep{poon2023smooth}. We note that this idea has also been adapted in a related, though slightly different, setting in~\cite{follain2024enhanced}, introducing BKerNN. The second method is  \textbf{Alternating Gradient Descent (AGD)}, see  Algorithm \ref{alg:AGD}. 
It is based only on gradient information and successively updates $B$ and $\alpha$ through gradient descent steps. This algorithm is similar to PALM (for Proximal Alternating Linearized Minimization) introduced in \cite{bolte2014proximal}, allowing multiple steps in $\alpha$ to improve its performance. Such alternating gradient schemes have already been applied in the literature in nonconvex settings, for example, for matrix factorization or two-layer neural networks \citep{lu2019pa,ward2023convergence}. In Appendix~\ref{sec:opt_details}, we provide further details on the above methods, including line search strategies for automatically tuning the learning rates $s_\alpha$ and $s_B$, and how to handle the constraint on matrix $B$. Some convergence results are provided in \thmref{thm:cvg_informal}, while empirical performances are investigated in Section~\ref{sec:MIGD_vs_AGD}.  We end this section discussing some comparison with other works in the literature.

\renewcommand{\algorithmicrequire}{\textbf{Require:}} 

\begin{table}[htbp]
\centering
\begin{tabular}{m{0.4\textwidth} m{0.5\textwidth}}
\toprule
\begin{minipage}[t]{\linewidth}
  \captionsetup{type=algorithm,justification=raggedright,
                singlelinecheck=false,margin=0pt,
                aboveskip=2pt,belowskip=0pt}
  \captionof{algorithm}{VarPro (informal)}\label{alg:MIGD}
  \vspace{-0.35em}\noindent\rule{\linewidth}{0.4pt}\vspace{0.4em}

  \begin{algorithmic}[1]
    \REQUIRE $B^0$, $s_B>0$
    \STATE $\alpha^{0}=\argmin_{\alpha\in \mathbb{R}^{\tilde m}}\;\hat{\mathcal L}(B^{0},\alpha)$
    \FOR{$i=0,1,\ldots$}
      \STATE $B^{i+1}=B^{i}-s_B\nabla_B \hat{\mathcal L}(B^{i},\alpha^{i})$
      \STATE $\alpha^{i+1}=\argmin_{\alpha\in \mathbb{R}^{\tilde m}}\;\hat{\mathcal L}(B^{i+1},\alpha)$
    \ENDFOR
    \STATE \textbf{return} $(B^{i+1},\alpha^{i+1})$
  \end{algorithmic}
\end{minipage}
&
\begin{minipage}[t]{\linewidth}
  \captionsetup{type=algorithm,justification=raggedright,
                singlelinecheck=false,margin=0pt,
                aboveskip=2pt,belowskip=0pt}
  \captionof{algorithm}{AGD (informal)}\label{alg:AGD}
  \vspace{-0.35em}\noindent\rule{\linewidth}{0.4pt}\vspace{0.4em}

  \begin{algorithmic}[1]
    \REQUIRE $B^0$, $\alpha^0$, $s_\alpha>0$, $s_B>0$, $n_\alpha\in\mathbb{N}^*$
    \FOR{$i=0,1,\ldots$}
      \STATE $B^{i+1}=B^{i}-s_B\nabla_B\hat{\mathcal L}(B^{i},\alpha^{i})$
      \STATE $\alpha^{i,0}=\alpha^{i}$
      \FOR{$j=0,1,\ldots,n_\alpha-1$}
        \STATE $\alpha^{i,j+1}=\alpha^{i,j}-s_\alpha\nabla_\alpha\hat{\mathcal L}(B^{i+1},\alpha^{i,j})$
      \ENDFOR
      \STATE $\alpha^{i+1}=\alpha^{i,n_\alpha}$
    \ENDFOR
    \STATE \textbf{return} $(B^{i+1},\alpha^{i+1})$
  \end{algorithmic}
\end{minipage}
\\
\bottomrule
\end{tabular}
\end{table}

\subsection{Related approaches}
In this section, we discuss the connection to some approaches that directly influence our study. 
An inspiration for the HKRR approach is Hyper-RBF networks proposed in~\cite{poggir90}, from which the term ``hyper" is borrowed. Hyper-RBF networks extend standard RBF networks, considering functions of the form
$f(x) = \sum_{i=1}^m \phi\big(\|B(x - w_i)\|\big),$
with $B$ a linear transformation to be learned. In~\cite{poggir90}, neither the representer theorem nor Nystr\"om inducing points were considered, and the \emph{centers} $w_1, \dots, w_m$, together with the coefficients $\alpha_1, \dots, \alpha_m$ and the matrix $B$, were optimized using stochastic gradient with no optimization guarantees. In comparison, we consider a more general class of hyper-kernels; we do not optimize the centers, but use Nystr\"om inducing points; and finally, we consider different gradient-based methods for which convergence guarantees are provided. Another inspiration for our work is the recursive feature machine (RFM) proposed in~\cite{radhakrishnan2024linear,radhakrishnan2024mechanism}; see also~\cite{zhu2025iteratively}. RFM is based on hyper-RBF kernels and defines an estimator similar, though with a slightly different form. Indeed, noting that
\(
\|Bx\| = \big(x^\top M x\big)^{{1}/{2}},\) with \( M = B^\top B,
\)
RFM considers functions of the form
\(f(x) = \sum_{i=1}^m k\big((x - x_i)^\top M (x - x_i)\big),\)
where \( k \) is a radial basis function that is also a reproducing kernel. The centers \( x_1, \dots, x_m \) are taken to be the input data points, as in kernel methods and HKRR. The coefficients are computed for an initial \( B \) (or rather \( M \)) via KRR using a closed-form expression. The key feature of RFM lies in the computation of \( M \), which is given by the average gradient outer product (AGOP) operator \citep{xia2002adaptive}:
\(
M = \frac{1}{n} \sum_{i=1}^n \nabla f(x_i) \nabla f(x_i)^\top,
\)
with \( f \)  a KRR solution. KRR and AGOP computations are then alternated. Aside from the more specific nature of the hyper-kernels considered, RFM is close to our VarPro algorithm.  The gradient step update of \( B \) in VarPro is replaced by the AGOP update. The AGOP operator has a long history in statistics in the context of sufficient dimension reduction \citep{samarov1993exploring,hristache2001direct}. However, unlike VarPro, the RFM iteration does not currently have an ERM and hence an optimization interpretation.  Finally, HKRR was also considered in  \cite{chen2023kernel}, developing ideas of  \cite{fukumizu2009kernel}. We will discuss this more in the next section.

\section{Theoretical results of HKRR for learning MIMs}\label{sec:thms}
In this section, we present a bound on the excess risk of HKRR for learning MIMs (\thmref{thm2}), derive the convergence rate of the Nystr\"om approximation for HKRR in \thmref{thm:ny}, and provide a theoretical analysis of adaptively estimating the unknown latent dimension $d_*$ and the regularization parameter $\lambda$ via cross-validation (\thmref{thm:cv_lbd}). The convergence analysis of AGD and VarPro is also established in \thmref{thm:cvg_informal}.

\subsection{Learning MIMs with HKRR: excess risk bound}
Consider MIMs, 
where the regression function takes the form
\begin{equation}\label{equ:mim}
   f_*(x)= g_*(B_*x), \qquad \rho_X\text{-a.e.\ } x\in X,
\end{equation}
with $B_*$ a $d_* \times D$ matrix such that $d_* < D$, $\|B_*\|_\infty\leq 1$, 
and $g_*$ a measurable function defined on $\R^{d_*}$. 
Estimating MIMs is challenging due to both the nonlinearity of the function $g_*$ 
and the difficulty of determining the linear map $B_*$. 
The following assumptions are needed to derive the rate of excess risk.
\begin{ass}\label{ass:0}
We assume that:
\begin{enumerate}[left=5pt,label=1.\arabic*]
 \item (Bounded data). The input space $X$ is a closed subset of $\R^D$, with 
 \(\|x\| \leq 1\) and \(|y| \leq M\) for some $M > 0$.
 \item (Smoothness). For some integer $r \geq 1$, the mother kernel satisfies 
 \(k \in C^r(\R^{d_*}\times \R^{d_*})\).
 \item (Source condition). For some $d_* < D$, there exists
 $B_* \in \R^{d_* \times D}$ with $\|B_*\|_\infty \leq 1$, such that $f_*$ lies in 
$\mathrm{Range}\bigl(L_{k_{B_*}}^{\theta/2}\bigr)$ for some $\theta \in (0,1]$. Here, $L_{k_{B_*}}: L_2(X,\rho_X) \to L_2(X,\rho_X)$ is an integral operator given by $(L_{k_{B_*}} f)(x)
  = \int_X k_{B_*}(x,x')\,f(x')\,\mathrm{d}\rho_X(x')$.

\end{enumerate}
\end{ass}
The condition that the input space $X$ is contained in the unit ball can always be enforced for bounded inputs by rescaling.
The boundedness of the outputs is also a standard assumption. The setting we consider is in the field of classical distribution-free non-parametric learning \citep{gyorfi2002distribution}. This contrasts with the stricter distributional assumptions adopted in other works, see, e.g., \cite{mousavi2022neural,bietti2025learning}.
The smoothness of the kernel $k$ provides a sufficient condition to control the covering numbers (see Assumption~\ref{ass:cv}). The Mat\'ern kernel is an example satisfying this assumption \citep{williams2006gaussian}. The source condition is well studied in classical kernel methods (see, e.g., \cite{cucker2007learning,de2021regularization}). It states the relationship between the target $f_*$ and the space determined by the integral operator defined by the kernel $k_{B_*}$. The parameter $\theta$ controls the smoothness of $f_*$. A larger $\theta$ implies smoother functions and a smaller function space, and therefore a better approximation rate.

The following theorem establishes the excess risk rate of HKRR defined in \eqref{equ:HKRR}. Here, the dimension $d_*$ is assumed to be known a priori, while the adaptive result for unknown $d_*$ is stated in Theorem~\ref{thm:cv_lbd}.
The proof is provided in Appendix~\ref{subsec:proof1}. 

\begin{thm}\label{thm2}
Suppose Assumption~\ref{ass:0} holds. Let $0<\delta<2/e$, $ \zeta < r/(d_*+r)$ and $\lambda = \lambda_m= m^{-\zeta}$. Then with  probability at least $1-\delta$, there holds
\[
\CR(\hflhbl)-\CR(f_*) \leq C_1 Dd_* \log ^2({2}/{\delta}) m^{-\theta \zeta}
\]
for all $m\geq m_\delta$, where $m_\delta$  is independent of $D$, $d_*$ and $C_1$ is a constant independent of $D$, $d_*$ and $\delta$. 
\end{thm}
\begin{rem}
Explicit expressions for $m_\delta$ (see \eqref{eq:14} with $s^* = d_*/r$) 
and for $C_1$ (see \eqref{eq:15}) are given in the proofs. 
For sufficiently large $m_\delta$, the factor $Dd_*$ can be improved to $(Dd_*)^{1/(s^*+1)}$ (Remark~\ref{better}). 
Moreover, Theorem~\ref{cor:1} in Appendix~\ref{subsec: proofthms}  yields a weaker bound, valid for all $m \geq 1$, of order 
$m^{-{r\theta}/{(1+\theta)(d_*+r)}}$.
\end{rem}



\begin{rem}[Beating the curse of dimensionality]
The minimax excess risk for estimating an $r$-smooth function $f:\RR^D \to \RR$ from $m$ samples scales as $m^{-2r/(2r+D)}$ \citep{gyorfi2002distribution}, which deteriorates exponentially with the input dimension $D$. 
In contrast, HKRR for MIM achieves a rate that depends exponentially only on the true transformation dimension $d_* \ll D$ and only polynomially on $D$, thereby mitigating the curse of dimensionality.
\end{rem}

\begin{rem}[Suboptimal rate]
\thmref{thm2} yields an excess risk bound of order $m^{-2r/(2r+2d_*)}$, 
which introduces an extra factor of $2$ in the $d_*$-term compared with the conjectured optimal rate. 
This suboptimality likely arises from relying on $L_\infty$-based covering number bounds over $\bigcup_B \CH_B$. 
Sharper analysis based on $L_2$-norm estimates or local Rademacher complexity \citep{bartlett2005local} is left for future work.
\end{rem}

\subsection{Nystr\"om approximation}
Recall that the solution of the Nystr\"om problem in \eqref{equ:opt_pb} is denoted by $\hat f_\lambda^{\hat B_{d_*,\tilde m}}$.
To describe the relationship between $\tilde m$ and $m$, we define a random variable
$\CN_{B_*,x}(\lambda)=\langle k_{B_*x},(\Sigma_{B_*}+\lambda I)^{-1}k_{B_*x}\rangle_{\CH_{B_*}}$ for $\lambda>0$ 
with $x\sim\rho_X$, where $\Sigma_{B_*}$ is the covariance operator of $k_{B_*}$ (see~\eqref{eq:covariance_op}), and set 
$\CN_{B_*,\infty}(\lambda)=\sup_{x\in X}\CN_{B_*,x}(\lambda)$. 

The following result shows that, with $\tilde m < m$ points, the plain Nystr\"om estimator can achieve the same excess risk rate as in Theorem~\ref{thm2}, up to constants. 
The proof is given in Appendix~\ref{apx:thm_proof}.

\begin{thm}\label{thm:ny}
Under the assumptions of Theorem~\ref{thm2}, with probability at least $1-\delta$,  
\[
\CR(\hat f_\lambda^{\hat B_{d_*,\tilde m}})-\CR(f_*)
 \le C_2Dd_*\log^2\!({2}/{\delta})m^{-\theta\zeta},
\]
where $C_2$ is given in \eqref{eq:C_nys}, provided  
$\tilde m \ge 67\log\!\tfrac{4\kappa}{\lambda\delta} \,\vee\, 
5\,\CN_{B_*,\infty}(\lambda)\log\!\tfrac{4\kappa}{\lambda\delta}$ 
for $\kappa=\sup_x k(x,x)$.
\end{thm}

\begin{rem}
Since $\CN_{B_*,\infty}(\lambda)\le \kappa/\lambda$ for all $\lambda>0$ \citep{caponnetto2007optimal,rudi2015less}, under the assumptions of Theorem~\ref{thm2} we have $m > \tilde m \sim \, m^{\zeta}$, where $\zeta$ can be chosen arbitrarily close to $r/(d_*+r)<1$.
\end{rem}

\begin{rem}
We also provide the rate of the approximate leverage score (ALS) Nystr\"om \citep{rudi2015less} with varying numbers of subsampling points; see Appendix~\ref{apx:thm_proof} for details. In fact, ALS requires fewer samples than the plain Nystr\"om method since $\CN_{B_*}(\lambda)\le \CN_{B_*,\infty}(\lambda)$, where {$\CN_{B_*}(\lambda)= \EE_{x \sim \rho_X}[\CN_{B_*,x}(\lambda)]$}.
\end{rem}

\subsection{Adaptivity}\label{sec:adaptivity}
The latent dimension $d_*$ is unknown in practice. To obtain adaptive guarantees, $d$ is tuned over $\{1,\ldots,D\}$. Given \(N\in \NN, \lambda_1, \lambda_N > 0\) and
\(Q = \left({\lambda_N}/{\lambda_1}\right)^{{1}/{(N-1)}}\), the regularization parameter $\lambda$ is chosen from the geometric grid
\(
\Lambda = \{\lambda_j = \lambda_1 Q^{j-1}\}_{j=1}^N\) 
assuming that the interval $[\lambda_1,\lambda_N]$ contains the optimal $\lambda$. Let
\( \Gamma = \bigl\{(d,\lambda)\mid d\in\{1,\dots,D\},\ \lambda\in\Lambda\bigr\},\)
so that \(|\Gamma| = DN.\) Let \(\{(x'_i,y'_i)\}_{i=1}^{m'}\sim\rho^{m'}\) be an independent validation set. We select
 \begin{align*}
    (\hat{d}, \hat{\lambda}) = \mathop{\argmin}\limits_{(d,\lambda)\in \Gamma} \frac{1}{m'} \sum_{i=1}^{m'} \Bigl( T_M \hat{f}_\lambda^{\hat{B}_d}(x'_i)-y'_i \Bigr)^2.
\end{align*}
Here, \(T_M\) is a truncation operator given by
\(
T_M f(x)
= \operatorname{sign}\bigl(f(x)\bigr)\,\min\bigl\{\lvert f(x)\rvert,\,M\bigr\},
\)
which handles the unboundedness of functions obtained by HKRR. 
The resulting estimator is denoted by \(\hat f_{\hat\lambda}^{\hat B_{\hat d}}.\)
The next theorem states that it
achieves the same rate (up to constants) as the estimator in \thmref{thm2}. The idea of its proof is classical; see, e.g., \cite{devroye2013probabilistic,chirinos2024learning}, 
and it is given explicitly in Appendix~\ref{apx:thm_proof}.
\begin{thm}\label{thm:cv_lbd}
For $\delta \in (0,1)$ and a suitable $q\in [1,Q]$, the following holds with probability at least $1-\delta$ that
\begin{align*}
    \CR(T_M \hat{f}_{\hat{\lambda}}^{\hat{B}_{\hat{d}}}) -\CR(f_*)  & \leq 2 q^\theta C_1 Dd_* \log ^2({2}/{\delta}) m^{-\theta \zeta}+ \frac{52M^2}{m'} \log \frac{2DN}{\delta}.
\end{align*}
\end{thm}

The above theorem shows how to choose hyperparameters adaptively and optimally.
{Furthermore, our experiments (Figure~\ref{fig:varyingd_combined}) highlight the impact of different choices of $d$ and reveal an interesting phenomenon: overparameterizing $d$ can sometimes yield better results. This observation suggests the conjecture that $\hat{d} > d_*$.}

\subsection{Optimization guarantees}\label{subsec:opt}

We next study the convergence properties of
Algorithms \ref{alg:MIGD} and \ref{alg:AGD} introduced in Section \ref{sec:algo}; see Appendix \ref{sec:opt_details} (Theorems \ref{thm:cvg_MIGD} and \ref{thm:cvg_AGD}) for further details. The proofs rely on the Kurdyka-\L{}ojasiewicz property \citep{attouch2013convergence}, which in turn requires the kernel $k$ to be analytic.

\begin{thm}[Convergence of AGD and VarPro (informal)]\label{thm:cvg_informal}
    Let $k$ be an analytic kernel. Suppose that the sequences $\left(B^i\right)_{i\in\N}$ and $\left(\alpha^i\right)_{i\in\N}$ generated by Algorithm \ref{alg:MIGD} or \ref{alg:AGD} are such that  the minimal eigenvalue $\lambda_{min}\left(\matK^{B^i}_{\tilde m\tilde m}\right)\geq \sigma$ for some $\sigma>0$. Then, {the sequence $\left(B^i,\alpha^i\right)_{i\in\N}$ converges to a critical point} of $\Psi:B,\alpha\mapsto\CL(B,\alpha)+i_\CB(B)$ as $i$ goes to infinity, and both sequences have finite length. In addition, there exists a constant $C>0$ such that after $N$ iterations, either $\left(B^N,\alpha^N\right)$ is a critical point of $\Psi$ or 
    \begin{equation*}
        \min_{0\leq i\leq N}\left\|\nabla_B\hat\CL\left(B^i,\alpha^i\right)\right\|^2\leq\frac{C}{N}.
    \end{equation*}
\end{thm}
The above result ensures that both methods converge to some critical point of the objective function as long as the sequence $\left(B^i\right)_{i\in\N}$ does not shrink the minimal eigenvalue of the kernel matrix $\matK^{B^i}_{\tilde m\tilde m}$. In other words, we require the data points to be linearly independent under the hyper-kernel $k_{B^i}$ at each iteration. This assumption guarantees that the sequence $\left(\alpha^i\right)_{i\in\N}$ is well defined and bounded, which allows us to analyze the algorithms using Kurdyka-\L{}ojasiewicz property.

\subsection{Comparison with other works}


Hyper-kernel RKHSs have been studied for dimension reduction, see e.g., \cite{fukumizu2009kernel,fukumizu2014gradient,chen2023kernel}. 
In \cite{fukumizu2009kernel}, they studied the conditional cross-covariance operator between input and output RKHSs. It was shown that the operator equals to one induced by hyper-kernel input RKHS and output RKHS when $B$ spans a central mean subspace \citep{chiaromonte2002sufficient}. They further connected the operator to the expected risk and yield an ERM framework. 
Building on this idea, \cite{chen2023kernel} proposed a related HKRR method and proved 
that it can recover the true subspace dimension asymptotically.

The explicit excess risk rates for learning MIMs in different approaches are also studied in the literature. For example,
\cite{klock2021estimating} used $k$-nearest neighbors and piecewise polynomials to learn the link function and employed the response-conditional least squares (RCLS) algorithm to estimate the latent matrix via inverse regression. Their generalization bound is $O(m^{-2r/(2r+d_*)})$, plus the error from learning the latent matrix. By contrast, our approach achieves $O(m^{-r/(r+d_*)})$ and provides two alternating minimization algorithms with both theoretical guarantees and empirical validation. 
\cite{bach2017breaking} established generalization bounds for MIMs with Lipschitz property. He considered hypothesis spaces as neural networks with finite variation norm and activation $\sigma(x)=(x)_+^\alpha$, $\alpha>0$. For ReLU ($\alpha=1$), the rate is $O((\log D)^{2/d_*+3} m^{-1/(d_*+3)} \log m)$. Our results emphasize the blessing of smoothness: the rate improves with $r$, and even for $r=1$ we obtain a sharper bound. 
Finally, we mention a related line of research on single-index models (SIMs, $d_*=1$), 
or ridge functions, investigated through convolutional neural networks; see, e.g., 
\cite{feng2023generalization,mao2023approximating,zhou2020theory} for approximation error analyses. 

The computational complexities of gradient-based algorithms have also been studied for learning the SIM and MIM recently.  A quantity characterizing complexity is the {information exponent}~\citep{arous2021online}, see,
 e.g., online SGD \citep{arous2021online}, GD \citep{ba2022high,moniri2023theory}, and SGD \citep{mousavi2022neural,damian2023smoothing}. 
An alternative notion of complexity is given by the leap exponent \citep{abbe2023sgd,dandi2023two,bietti2025learning}. 
For a broader discussion of MIMs, see the survey \cite{bruna2025survey}.

\section{Numerical experiments}\label{sec:experiment}

In this section, we study the performance of the methods introduced in Section \ref{sec:algo} on simulated datasets, using the Gaussian kernel $k:x,x'\mapsto\exp\left(-\gamma\left\|x-x'\right\|^2\right)$. Details on the experimental setup can be found in Appendix \ref{app:setting_exp}.
\paragraph{Non-convexity of HKRR.}\label{sec:MIGD_vs_AGD}Given the nonconvex nature of the objective function minimized in HKRR, the performance of first-order methods such as VarPro (Algorithm \ref{alg:MIGD}) and AGD (Algorithm \ref{alg:AGD}) can be severely impacted by a poor initialization. In particular, VarPro directly exploits the structure of the problem in \eqref{equ:opt_pb} by computing a closed-form solution at each iteration. This leads to a \textbf{faster convergence} than AGD, especially for few Nystr\"om centers, since only a matrix of size $\tilde m \times \tilde m$ must be inverted. However, this advantage comes at a cost: because VarPro optimizes solely over $B$—the variable responsible for the non-convexity—it is prone to being \textbf{trapped in local minima} and is thus highly sensitive to initialization. By contrast, AGD explores the landscape of $\hat\CL$ jointly in both $B$ and $\alpha$, which can help it escape critical points where VarPro stagnates. This behavior is illustrated in Figure \ref{fig:losses_col}: in one scenario (top graph), both methods converge ultimately to the same solution, with VarPro reaching it more quickly; in the other scenario  (bottom graph), AGD manages to escape a critical point in which VarPro remains stuck. A simple two-dimensional problem illustrates the above intuition, see Figure~\ref {fig:path_col1}, Figure~\ref{fig:path_col2}, and Appendix~\ref{app:2d_case} for further details.



\begin{figure}[H]
\centering

\begin{subfigure}{0.29\textwidth}
    \centering
    \includegraphics[width=\linewidth]{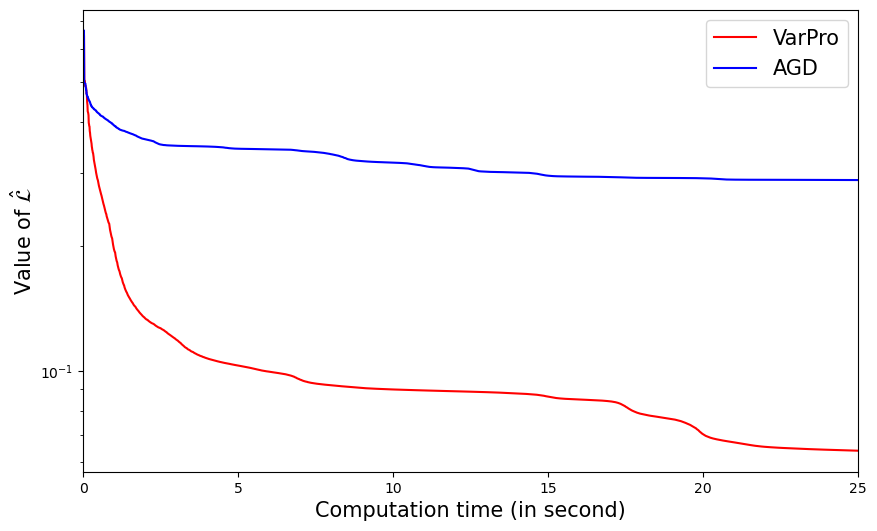}\\[4pt]
    \includegraphics[width=\linewidth]{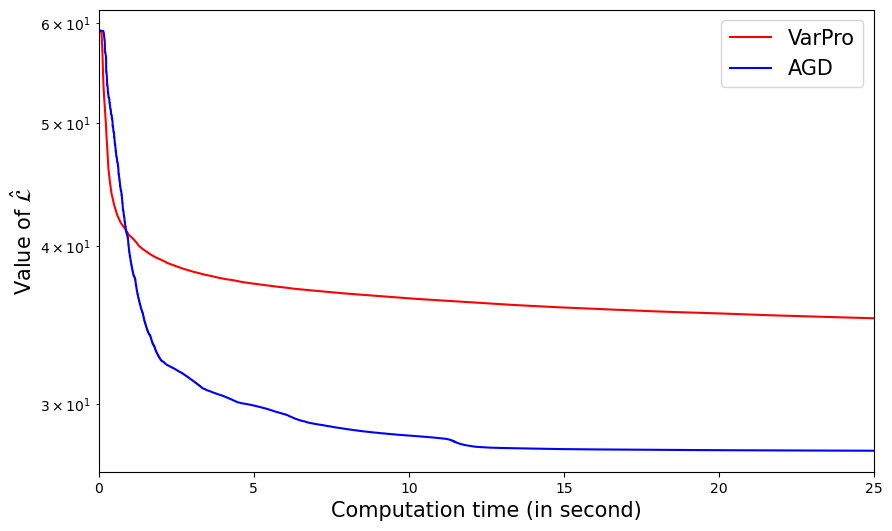}
    \caption{}
    \label{fig:losses_col}
\end{subfigure}
\hfill
\begin{subfigure}{0.29\textwidth}
    \centering
    \includegraphics[width=\linewidth]{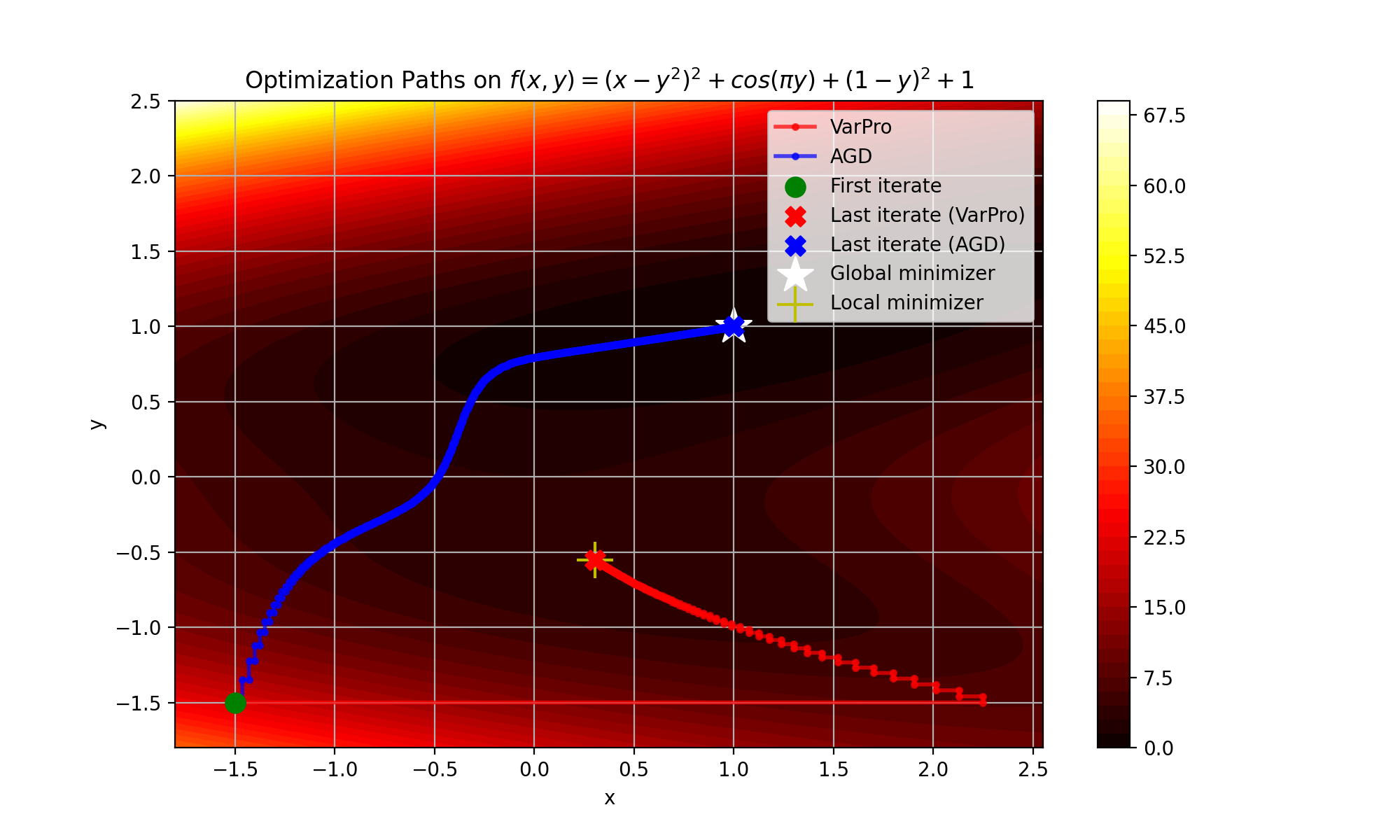}\\[4pt]
    \includegraphics[width=\linewidth]{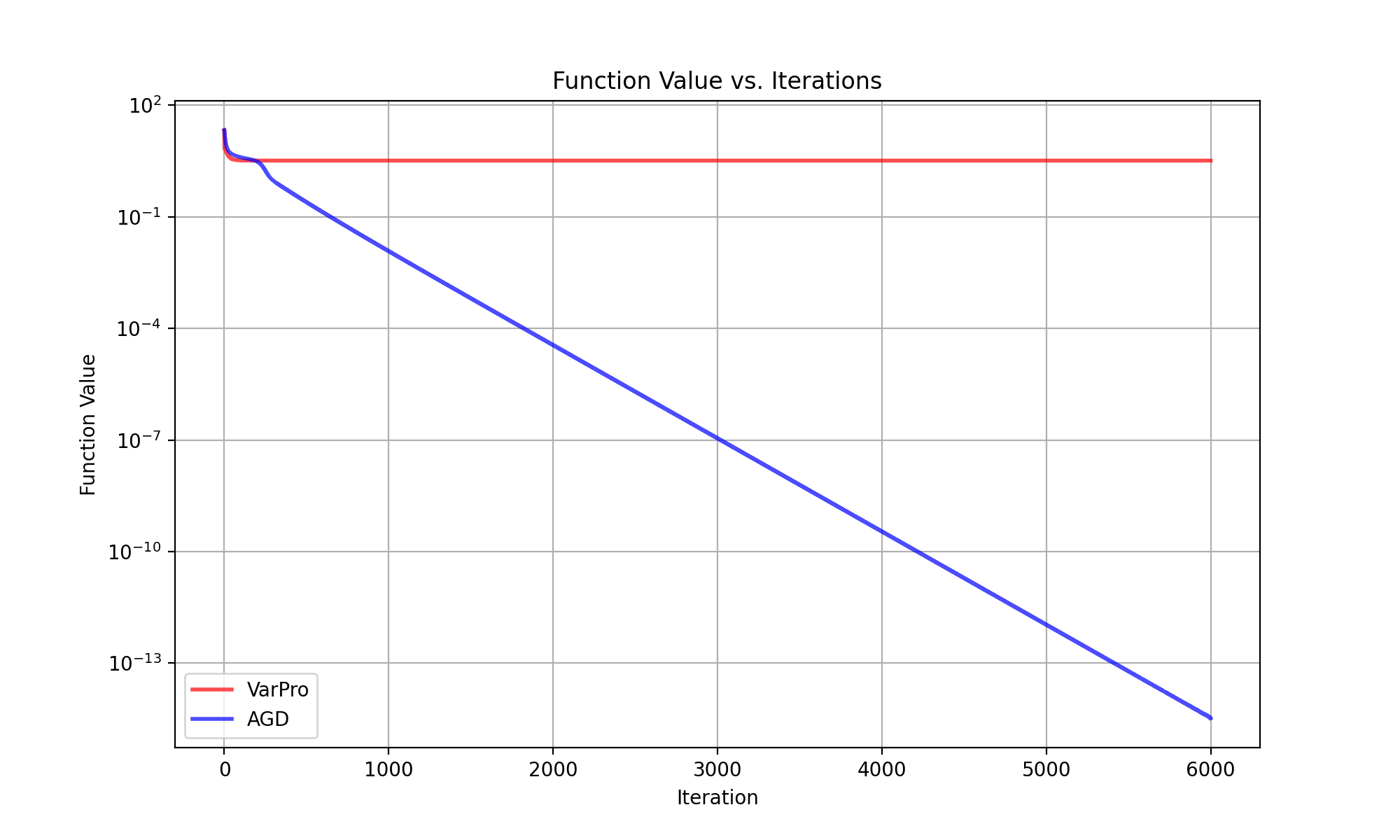}
    \caption{}
    \label{fig:path_col1}
\end{subfigure}
\hfill
\begin{subfigure}{0.29\textwidth}
    \centering
    \includegraphics[width=\linewidth]{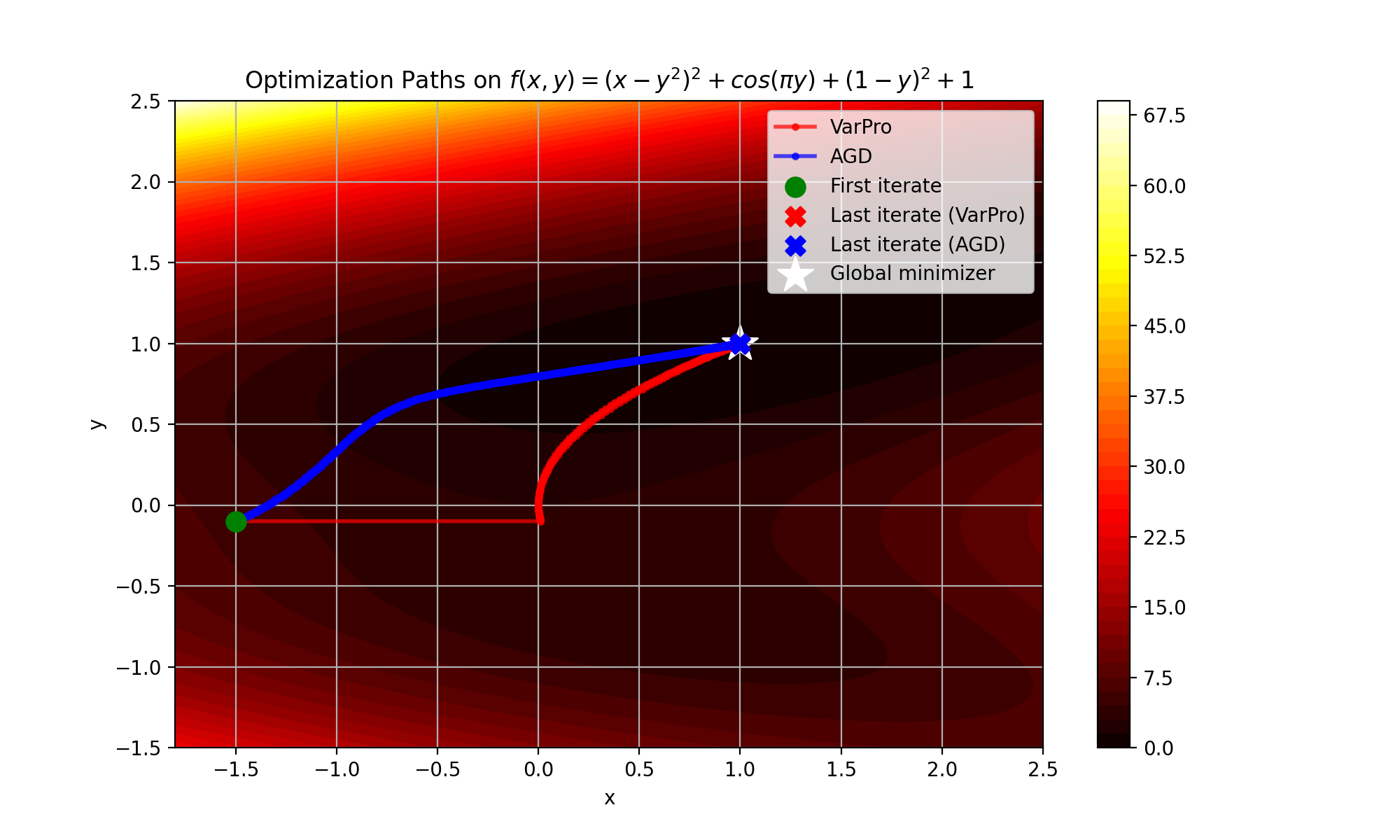}\\[4pt]
    \includegraphics[width=\linewidth]{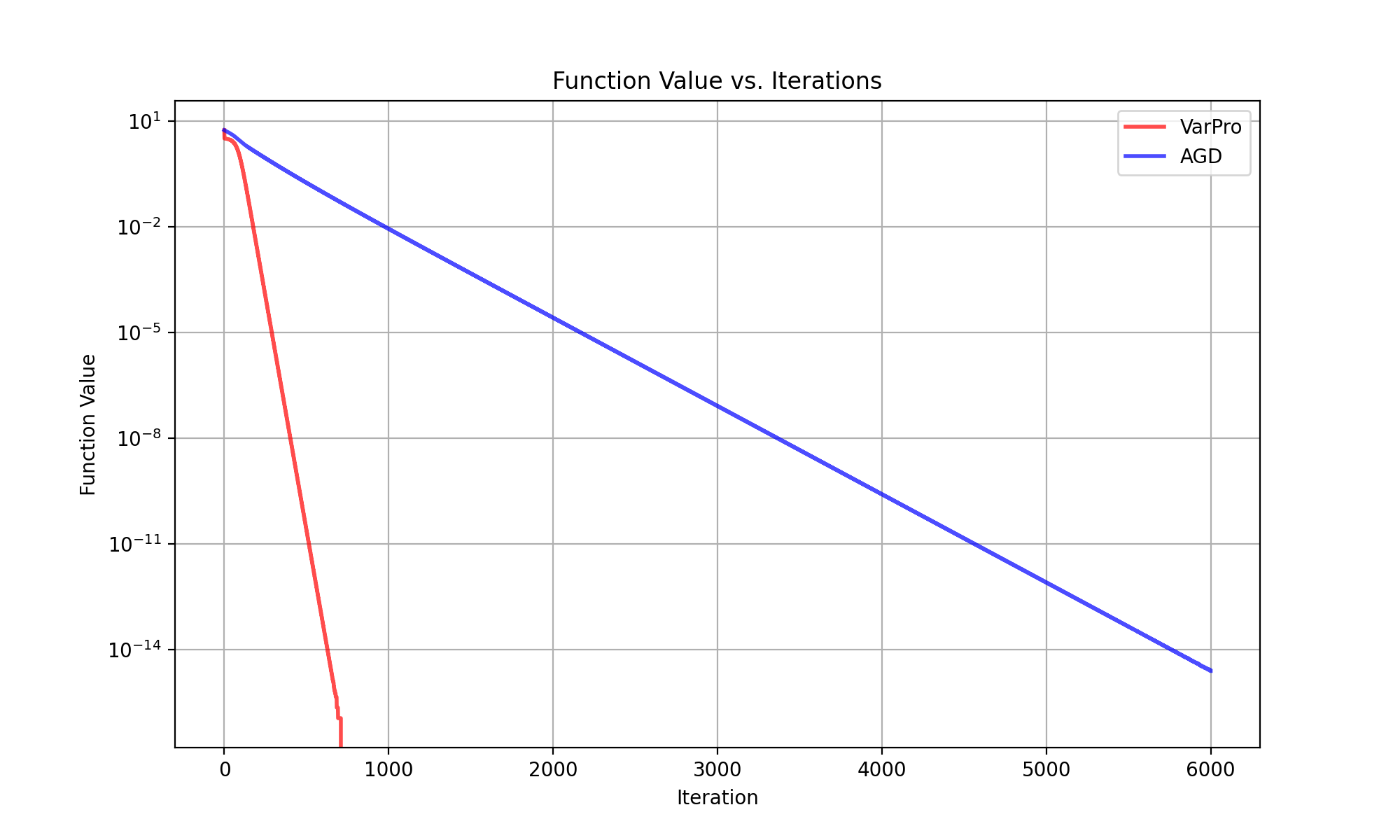}
    \caption{}
    \label{fig:path_col2}
\end{subfigure}

\caption{Comparison between VarPro (red) and AGD (blue). 
(\subref{fig:losses_col}) Training losses across time for two random initializations of $B^0$. 
(\subref{fig:path_col1}) Two-dimensional toy example with initialization $(-1.5,-1.5)$: AGD escapes a local minimum where VarPro remains stuck. 
(\subref{fig:path_col2}) Initialization $(-1.5,-0.1)$: both methods converge to minima, with VarPro being significantly faster. 
See Appendix~\ref{app:2d_case} for additional details.}
\label{fig:losses_paths_3col}
\end{figure}

\paragraph{Initialization and hyper-parameter tuning.} 
To avoid the effect of poor initialization of $B$, 
the proposed strategy is to sample several matrices from $\mathcal{B}$, 
with 10 matrices sampled in the presented experiments. 
The initialization $B^0$ is then selected by cross-validation. 
This involves computing the coefficients that minimize $\hat{\mathcal{L}}(B,\alpha)$ 
and testing each pair of matrix and coefficients on a validation set. 
Since $\hat{\mathcal{L}}$ involves a regularization parameter $\lambda$, it must be initialized either by coupled cross-validation with $B^0$ or arbitrarily. 
For the Gaussian kernel used in these experiments, 
an additional scaling parameter $\gamma$ must also be tuned. 
We adopt the well-known heuristic
\(
\gamma = \frac{1}{2\tilde \mu^2},
\)
where $\tilde \mu = \text{median}\{\|B(x_i-x_j)\| : i \neq j\}$ 
is computed separately for each sampled matrix $B$.

\paragraph{On the role of the latent dimension.} 
Beyond the conventional hyper-parameters of KRR, HKRR introduces the latent dimension $d_*$. 
Since this value is unknown in practice, it is crucial to understand how its estimate $d$ affects performance. 
Figure~\ref{fig:varyingd_combined} shows that underestimating $d$ ($d<d_*$) severely reduces accuracy, 
while \textbf{overparameterization} is more robust: choosing $d>d_*$ often even outperforms the true value $d_*$. 
However, with a limited computational budget, very large $d$ may degrade approximation quality. 
For larger datasets, setting $d=20$ consistently yields better results than $d=d_*$,  whereas setting $d=D=50$ is inefficient due to the fixed budget.


\begin{figure}[htbp]
\centering
\includegraphics[width=0.45\textwidth]{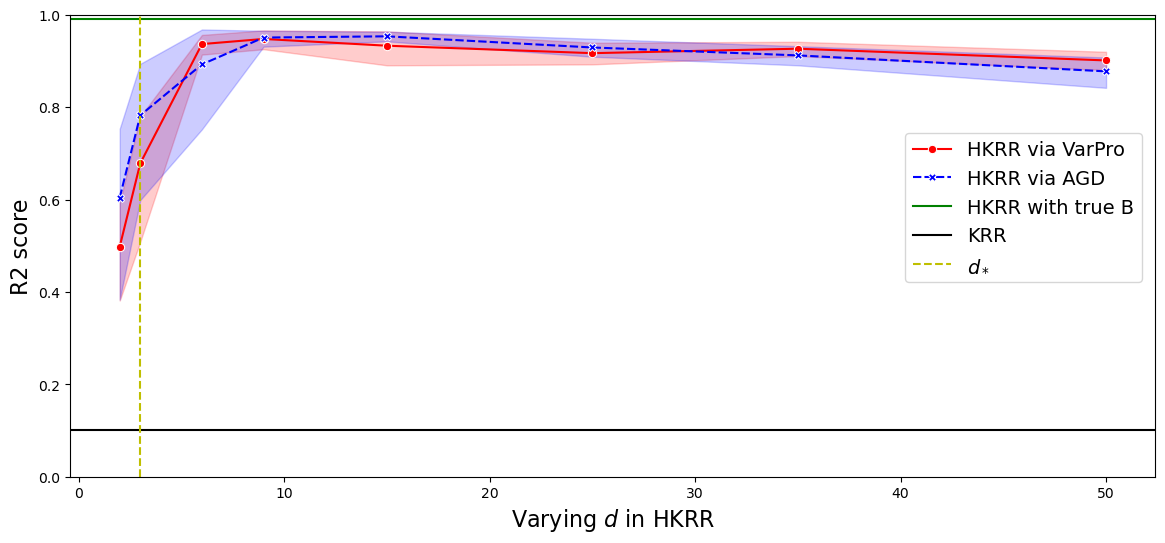}~
\includegraphics[width=0.45\textwidth]{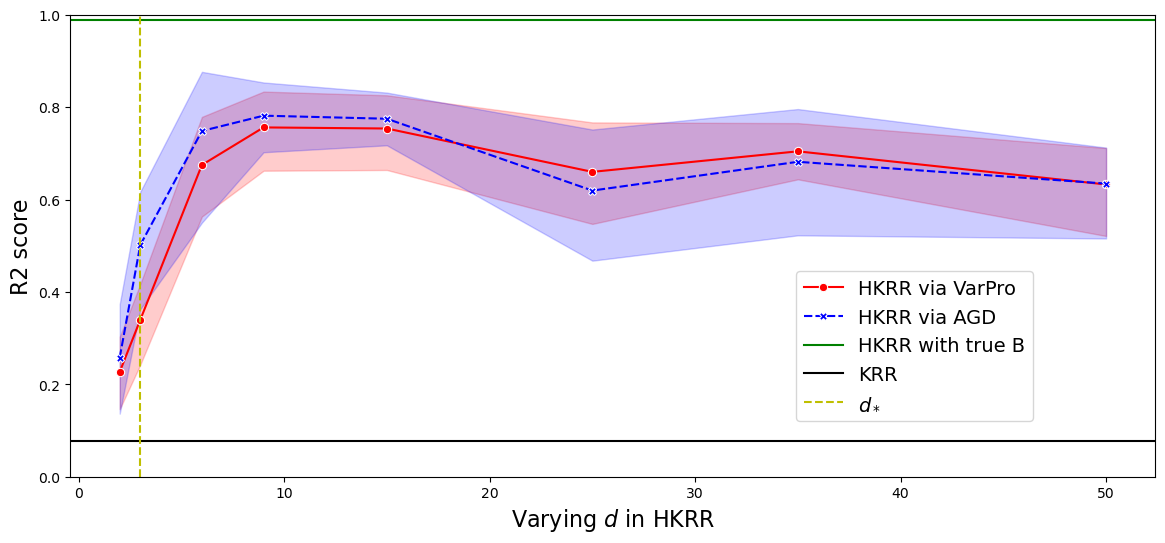}\\[4pt]
\includegraphics[width=0.45\textwidth]{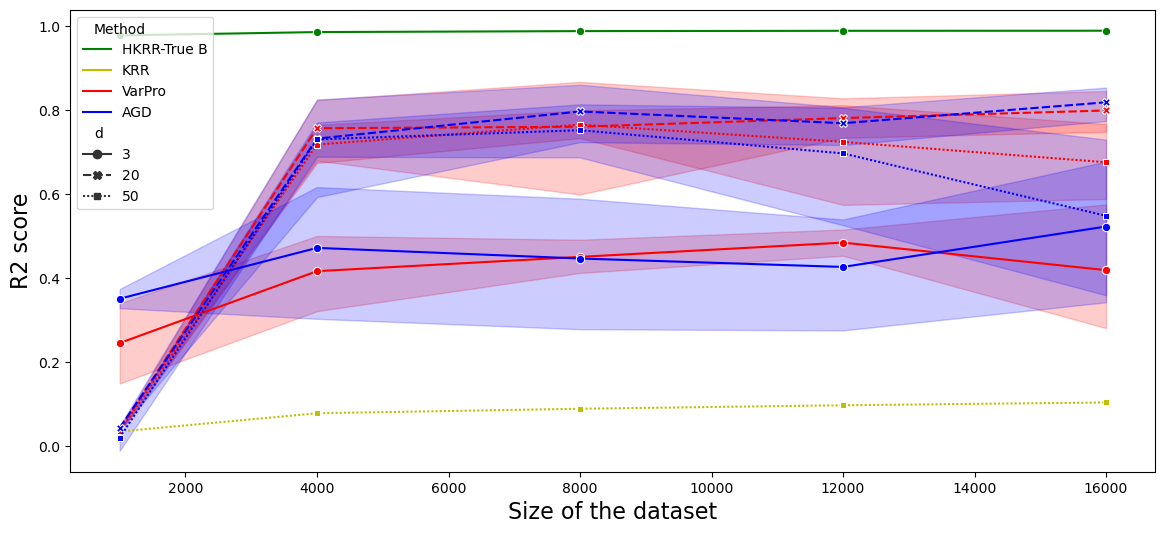}
\caption{R2 score on test sets for $B$ and $\alpha$ learned by VarPro (red) and AGD (blue). 
Top: performance w.r.t.\ the parameter $d$ for Dataset~1 (left) and Dataset~2 (right), with true latent dimension $d_*=3$, $D=50$. 
Bottom: performance for $d\in\{3,20,50\}$ as the training size increases for Dataset 1. See Appendix \ref{app:setting_exp} for further details.}
\label{fig:varyingd_combined}
\end{figure}

\section{Conclusion}\label{sec:conclusion}
In this work, we investigated hyper-kernel ridge regression as a step towards exploring the compositional principle underlying deep learning. HKRR is an approach combining ideas from kernel methods and neural networks, related to recently proposed methods such as RFM.
Our main contribution is the analysis of the sample complexity of HKRR when learning MIMs.
Unlike standard KRR, HKRR can adapt to the MIM structure to escape the curse of dimensionality.
From an algorithmic perspective, we exploit the structure of HKRR to analyze two approaches, VarPro and AGD, drawing ideas from convex optimization and for which local convergence guarantees can be established. Numerical results illustrate and corroborate our findings. Altogether, these results suggest that HKRR can be regarded as a useful augmentation of kernel methods, and point to new directions for developing efficient algorithms that bridge kernel and neural network approaches.

A natural direction for future work is to refine our analysis to obtain sharper bounds.  It would be especially interesting to consider more general forms of compositional functions beyond MIMs, and see if kernel methods and neural network ideas can again be combined to provably learn such models.


\section{Acknowledgment}
The research by E.D.V has been partially supported by the MIUR grant PRIN 202244A7YL, by the PNRR project “Harmonic Analysis and Optimization in Infinite-Dimensional Statistical Learning - Future Artificial Intelligence  Fair – Spoke 10" (CUP J33C24000410007) and  by the MIUR Excellence Department Project awarded to Dipartimento di Matematica, Universit\`a di Genova (CUP D33C23001110001). E.D.V. is a member of the Gruppo Nazionale per l’Analisi Matematica, la Probabilità e le loro Applicazioni (GNAMPA) of the Istituto Nazionale di Alta Matematica (INdAM). 
T.P. acknowledges support in part from the Center for Minds, Brains and Machines (CBMM), funded by NSF STC award CCF-1231216. 
L. R. acknowledges the financial support of: the European Commission (Horizon Europe
grant ELIAS 101120237), the Ministry of Education, University and Research (FARE grant ML4IP R205T7J2KP), the European Research Council (grant SLING 819789), the US Air Force Office of Scientific Research (FA8655-22-1-7034), the Ministry of Education, the grant BAC FAIR PE00000013 funded by the EU - NGEU and the MIUR grant (PRIN 202244A7YL). This work represents only the view of the authors. The European Commission and the other organizations are not responsible for any use that may be made of the information it contains.

\newpage
\clearpage

\bibliographystyle{abbrv}
\bibliography{biblio}

\newpage
\appendix
 
\clearpage


\section{Preliminary lemmas and basic error bounds for Theorem \ref{thm2}}\label{subsec: proofthms}

In this appendix, we prove Theorem \ref{thm2} and some accompanying results. Many of these results hold true under weaker conditions than Assumption~\ref{ass:0}, so we treat these results and conditions separately.  The proof of Theorem \ref{thm2} is given in Subsection \ref{subsec:proof1}.

In the following, if   $S$ is a compact space, the Banach space of continuous functions on $S$ endowed with the sup norm $\|\cdot\|_\infty$ is denoted by $C(S)$.  We also need to recall some basic quantities and fact associated to every RKHS.

\subsection{RKHS and related operators }\label{sec:RKHS}

We recall that if $k$ is continuous and bounded, then the following operators are 
well defined, bounded, and positive: 
\begin{enumerate}[left=5pt,label=\alph*)]
\item The integral operator $L_k: L_2(X,\rho_X) \to L_2(X,\rho_X)$ 
\begin{equation*}
    L_k(g)(x) = \int_X k(x,x') g(x') \, \md \rho_X(x'), 
    \quad g \in L_2(X,\rho_X).
\end{equation*}

\item The covariance operator $\Sigma: \CH \to \CH$
\begin{equation}\label{eq:covariance_op}
\Sigma f = \int_X \langle f, k_x\rangle_\CH \, k_x \, \md \rho_X(x)
= \left[ \int_X (k_x \otimes k_x) \, \md \rho_X(x)\right] f, 
\quad f \in \CH,
\end{equation}
where for all $x \in X$, $k_x := k(x,\cdot) \in \CH$, and 
$(k_x \otimes k_x): \CH \to \CH$ is the positive rank-one operator
\[
(k_x \otimes k_x)(f) = \langle f, k_x\rangle_\CH \, k_x.
\]
\end{enumerate}

Moreover, the relationship between the $L_2(\rho_X)$ norm and the RKHS norm is given by: 
for $g \in \CH$,
\begin{equation}\label{equ:l2_H}
\|g\|_{\rho_X}^2 
= \|\Sigma^{\frac{1}{2}} g\|_\CH^2,
\end{equation}
where $\Sigma^{\frac{1}{2}}$ is the square root of the positive operator $\Sigma$, 
defined via spectral calculus.

\subsection{Covering number of composite classes and hyper RKHS}

Let $V$ be a vector space endowed with a norm $\|\cdot\|_V$. 
The ball of radius $R$ and centered at the origin is denoted by 
$\mathbb B_{V,R}=\{f \in V : \|f\|_V \leq R\}$. 
Given a subset $\CG \subset V$ with compact closure, for all $\epsilon>0$, 
$\CN_V(\CG, \epsilon)$ is the covering number of $\CG$, defined as the minimal 
$J \in \mathbb{N}$ such that there exist $g_1,\ldots,g_J \in \CG$ satisfying 
$\CG = \bigcup_{j=1}^J \{g \in \CG : \|g-g_j\|_V \leq \epsilon\}$.  

If $\CH$ is an RKHS on $S$ with a continuous kernel, then $\CH$ is a subspace of $C(S)$, 
and its ball of radius $R$ is compact in $C(S)$~\citep{cucker2007learning}. 
We denote by $\CN(\mathbb B_{\CH,R}, \epsilon)$ the corresponding covering number, omitting the index $C(S)$ for simplicity.
We need the following condition on the data space. 

\begin{ass}\label{ass:bound}
The input space $X$ is a compact subset of $\R^D$ such that 
$\sup_{x \in X} |x| \leq 1$, and for some $M > 0$, $|y| \leq M$. 
\end{ass}

The assumption that $X$ is bounded is needed to control the covering number, 
see Lemma~\ref{lem:coveringNo}. 
The assumption that the outputs are bounded implies 
$| f_*(x) = \int_Y y \, \md \rho(y|x) | \leq M$. 

Given an integer $d$, we  recall that 
$\CB = \{B \in \R^{d \times D} : \|B\|_\infty \le 1\}$,
so that 
\[
\Omega=\{ B x\in\R^d \mid x\in X, B \in \CB\} \subset \R^d
\]
is compact,  since the map $(x,B)\mapsto Bx$ is continuous and $X\times \CB$ is compact. 

We impose the following condition on the mother RKHS. 
\begin{ass}\label{ass:1-1}
The mother space $\CH$ is an RKHS on $\Omega$ with a continuous kernel $k$, and
for all $g \in \CH$,
\[
|g(x') - g(x)| \leq C_{\CH} \, \|g\|_{\CH} \, \|x' - x\|, 
\qquad x,x' \in \Omega,
\]
for some constant $C_{\CH} > 0$. 
\end{ass}
If $k$ is defined on an open set $U \supset \Omega \times \Omega$ and $k \in C^1(U)$, 
then Assumption~\ref{ass:1-1} always holds. 
The above assumption states that the elements of $\CH$ are Lipschitz functions 
with a Lipschitz constant that is uniform on any ball of $\CH$. 
Furthermore, for all $g \in \CH$,
\[
\|g\|_{\infty} \leq \kappa^{\frac{1}{2}} \|g\|_{\CH},
\]
where $\kappa = \sup_{x \in \Omega} k(x,x)$, which is finite since $\Omega$ is compact.  

Recall that, for any $B \in \CB$, the hypothesis space $\CH_B$ is the RKHS with reproducing kernel 
\[
k_B(x, x') = k\bigl(Bx, Bx'\bigr), \quad x, x' \in X,
\]
which is continuous and bounded by $\kappa$. Hence $\CH_B \subset C(X)$ and, 
for all $f \in \CH_B$,
\begin{equation}\label{eq:6}
\|f\|_{\infty} \leq \kappa^{\frac{1}{2}} \|f\|_B,
\end{equation}
where $\|f\|_B = \|f\|_{\CH_B}$ and $\mathbb B_{B,R} = \mathbb B_{\CH_B,R}$. 

Moreover, it holds that
\[
\CH_B = \{ f : X \to \R \mid f = g \circ B \text{ for some } g \in \CH \},
\]
and
\[
\|f\|_B = \min \{ \|g\|_{\CH} \mid f = g \circ B, \, g \in \CH \}.
\]
Since the minimum is achieved, for every $f \in \mathbb B_{B,R}$ there exists $g \in \mathbb B_{\CH,R}$ such that $f = g \circ B$.

The following lemma provides a bound on the covering number of $\mathop{\bigcup}\limits_{B \in \CB} \mathbb B_{B,R} \subset C(X)$. 

\begin{lem}\label{lem:coveringNo}
Assume~\ref{ass:bound} and~\ref{ass:1-1}. Fix $\epsilon>0$ and $R>0$. Then  
\begin{equation}\label{equ: cv}
    \CN\left(\mathop{\bigcup}\limits_{B \in \CB} \mathbb B_{B,R}, \epsilon\right) 
    \leq \left(\frac{6C_{\CH}R}{\epsilon}\right)^{Dd} 
    \CN\!\left(\mathbb B_{\CH,R},\tfrac{\epsilon}{2}\right).
\end{equation}
\end{lem}

\begin{proof}
Let $g_{j}$ be the covering centers of $\mathbb B_{\CH,R}$ with radius $\epsilon/2$, 
and let $B_{i}$ be the covering centers of $\CB$ with radius $\epsilon/(2C_{\CH}R)$ 
(where $\CB$ is regarded as a compact subset of the space $V$ of $d \times D$ matrices 
endowed with the uniform norm). Then for any $g \in \mathbb B_{\CH,R}$ and $B\in \CB$, 
there exist $g_{j}\in \mathbb B_{\CH,R}$ and $B_{i} \in \CB$ such that 
$\|g-g_{j}\|_\infty \leq \epsilon/2$ and $\|B-B_{i}\|_\infty \leq \epsilon/(2C_{\CH}R)$. 

If we denote $f_\ell = g_{j} \circ B_{i}$, then $f_\ell \in \mathop{\bigcup}\limits_{B \in \CB} \mathbb B_{B,R}$ 
because $g_j \in \mathbb B_{\CH,R}$ and $B_{i} \in \CB$. For $f=g \circ B$ we have
\begin{align*}
       \|f-f_\ell\|_\infty 
       &= \|g\circ B- g_{j}\circ B_{i}\|_\infty \\
       &\leq \|g \circ B-g_{j} \circ B\|_\infty 
       + \|g_{j} \circ B-g_{j} \circ B_{i}\|_\infty\\
       &\leq \sup_{x\in X}|g(Bx)-g_{j}(Bx)| 
       + \sup_{x\in X}|g_{j}(Bx)-g_{j}(B_{i}x)| \\ 
       &\leq \sup_{x'\in\Omega} |g(x')-g_{j}(x')| 
       + C_{\CH} \|g_j\|_{\CH} \sup_{x\in X}\|(B-B_i)x\| \\
       &\leq \|g-g_{j}\|_\infty 
       + C_{\CH} \|g_j\|_{\CH}\|B-B_i\|_\infty  \sup_{x\in X}\|x\| \\
       &\leq \frac{\epsilon}{2} 
       + C_{\CH} R \frac{\epsilon}{2 C_{\CH}R} \sup_{x\in X}\|x\| 
       = \epsilon.
\end{align*}
Here, we used the property $\|x\| \leq 1$, and the fact that $g_j$ is Lipschitz with 
constant $C_{\CH}\|g_j\|_{\CH}$.

Therefore, we obtain an $\epsilon$-cover of 
$\mathop{\bigcup}\limits_{B \in \CB} \mathbb B_{B,R}$ with centers $f_\ell$, 
induced by an $\epsilon/2$-cover of $\mathbb B_{\CH,R}$ and an 
$\epsilon/(2C_\CH R)$-cover of $\CB$. By the metric entropy of $\CB$, we have
\begin{align*}
    \CN\left(\mathop{\bigcup}\limits_{B \in \CB} \mathbb B_{B,R}, \epsilon\right)  
   &\leq  \CN_{M_{dD}} \!\left( \CB ,\tfrac{\epsilon}{2C_{\CH}R}\right) 
   \CN\!\left( \mathbb B_{\CH,R}, \tfrac{\epsilon}{2} \right)\\
   &\leq \left(\frac{6C_{\CH}R}{\epsilon}\right)^{Dd} 
   \CN\!\left(\mathbb B_{\CH,R},\tfrac{\epsilon}{2}\right),
\end{align*}
where the last inequality follows from the classical bound
\[
\CN_{M_{dD}} \left( \CB ,\epsilon\right)\leq \left(\frac{3}{\epsilon}\right)^{Dd},
\]
see~\cite[Thm.~5.3]{zhang2023mathematical}.
\end{proof}

\subsection{Error decomposition}

We recall that in the multi-index model  
\[
f_*(x) = g_*(B_*x), \qquad \rho_X\text{-a.e. } x \in X,
\]
for some $d_* \times D$ matrix $B_*$ with $\|B_*\|_{\infty} \leq 1$ and some measurable function $g_*$, which we can assume to be defined on $\Omega$. 

Note that if $f_* \in \CH_{B_*}$, then
\[
\CR(f_*) = \min_{f:X\to\R} \CR(f) \leq \inf_{B \in \CB_*} \inf_{f \in \CH_B} \CR(f) = \CR(f_*),
\]
so that
\begin{equation}\label{eq:4}
\inf_{B \in \CB_*} \inf_{f \in \CH_B} \CR(f) = \CR(f_*),
\end{equation}
indicating that HKRR provides a suitable criterion for the MIM. 

In the following, we set $\CB_* = \CB$ with $d = d_*$, and recall that
\begin{alignat*}{2}
   \hat{f}_\lambda^B &= \mathop{\argmin}_{f \in \CH_B} \widehat{\CR}_\lambda(f), \quad && B \in \CB_*,\\
   \hat{B}_{d_*} &\in \mathop{\argmin}_{B \in \CB_*} \widehat{\CR}_\lambda(\hat{f}_\lambda^B),\\
   f_{\lambda}^{B_*} &= \mathop{\argmin}_{f \in \CH_{B_*}} \CR_\lambda(f).
\end{alignat*}
Note that both $\hat{f}_\lambda^B$ and $f_{\lambda}^{B_*}$ exist and are unique. For simplicity, we assume that $\hat{B}_{d_*}$ also exists; otherwise, it suffices to consider an $\epsilon$-minimizer. 

We now state the following error decomposition for the excess risk of $\hflhbl$, where the main challenge lies in identifying a suitable intermediate term to incorporate, since there are many possible choices of $f_\lambda^B$ and $\hat{f}_\lambda^B$ corresponding to different $B$.

\begin{lem}\label{lem:errdcp}
Fix $R>0$. Then
\begin{align} \label{equ:errordecm}
  \CR(\hflhbl) - \CR(f_*) &\leq  \underbrace{ \sup_{f \in \mathop{\bigcup}\limits_{B \in \CB} \mathbb B_{B,R}} \left( ( \CR(f)-\CR(f_*) )-(\widehat{\CR}(f)-\widehat{\CR}(f_*)) \right) }_{\uppercase\expandafter{\romannumeral1}} \notag \\ 
 & \quad +   \underbrace{\widehat{\CR}({f}_{\lambda}^{B_*})-\widehat{\CR}(f_*) - ({\CR}({f}_{\lambda}^{B_*})-\CR(f_*))}_{ \uppercase\expandafter{\romannumeral2} }  + \underbrace{\|{f}_{\lambda}^{B_*} - f_*\|_{\rho_X}^2 + \lambda \|{f}_{\lambda}^{B_*}\|_{B_*}^2}_{\uppercase\expandafter{\romannumeral3}   } 
\end{align}
for all training sets such that  $\|\hflhbl\|_{\hat{B}_{d_*}} \leq R$. 
\end{lem}
\begin{proof} The excess risk can be rewritten and decomposed as
\begin{align}
    &\CR(\hflhbl) - \CR(f_*) \notag \\
&=  \CR(\hflhbl) - \widehat{\CR}(\hflhbl) + \widehat{\CR}(\hflhbl) - \widehat{\CR}({f}_{\lambda}^{B_*}) + \widehat{\CR}({f}_{\lambda}^{B_*}) - \CR({f}_{\lambda}^{B_*}) + \CR({f}_{\lambda}^{B_*}) - \CR(f_*)  \notag\\
&\leq \CR(\hflhbl) - \widehat{\CR}(\hflhbl) + \widehat{\CR}_\lambda(\hflhbl) - \widehat{\CR}_\lambda({f}_{\lambda}^{B_*}) \notag \\ & \quad +  \widehat{\CR}({f}_{\lambda}^{B_*})  - \CR({f}_{\lambda}^{B_*}) + \lambda \|{f}_{\lambda}^{B_*}\|_{B_*}^2 + \CR({f}_{\lambda}^{B_*}) - \CR(f_*) \notag \\
& \leq (\CR(\hflhbl) - \widehat{\CR}(\hflhbl)) + (\widehat{\CR}({f}_{\lambda}^{B_*}) - {\CR}({f}_{\lambda}^{B_*}) )+ \lambda \|{f}_{\lambda}^{B_*}\|_{B_*}^2 + \CR({f}_{\lambda}^{B_*}) - \CR(f_*) \label{equ:dcm1}\\
&= \left\{\CR(\hflhbl)-\CR(f_*) - (\widehat{\CR}(\hflhbl)-\widehat{\CR}(f_*))\right\}  + \left\{\widehat{\CR}({f}_{\lambda}^{B_*})-\widehat{\CR}(f_*) - ({\CR}({f}_{\lambda}^{B_*})-\CR(f_*))\right\}\notag \\
& \quad +\left\{\|{f}_{\lambda}^{B_*} - f_*\|_{\rho_X}^2 + \lambda \|{f}_{\lambda}^{B_*}\|_{B_*}^2 \right\} \notag \\
 & \leq { \sup_{f \in \mathop{\bigcup}\limits_{B \in \CB} \CH_{R,B}} \left\{( \CR(f)-\CR(f_*) )-(\widehat{\CR}(f)-\widehat{\CR}(f_*))\right\}  }\notag \\ 
 & \quad +  \left\{\widehat{\CR}({f}_{\lambda}^{B_*})-\widehat{\CR}(f_*) - ({\CR}({f}_{\lambda}^{B_*})-\CR(f_*))\right\}  +\left\{\|{f}_{\lambda}^{B_*} - f_*\|_{\rho_X}^2 + \lambda \|{f}_{\lambda}^{B_*}\|_{B_*}^2 \right\}. \notag 
\end{align}
Inequality in~\eqref{equ:dcm1} follows from the fact that, by definition of $\hat{B}_{d_*}$,
\(
\widehat{\CR}_\lambda(\hflhbl) - \widehat{\CR}_\lambda(f_{\lambda}^{B_*}) \leq 0.
\)
\end{proof}
Note that by the definition of $\hflhbl$, 
\(
\widehat{\CR}_\lambda(\hflhbl) \leq \widehat{\CR}_\lambda(0),
\)
so that, under Assumption~\ref{ass:bound}, 
\begin{equation}\label{eq:10}
\|\hflhbl\|_{\hat{B}_{d_*}} \leq \frac{M}{\sqrt{\lambda}},
\end{equation}
hence we can always choose $R=M/\sqrt{\lambda}$. 

The first two components in \eqref{equ:errordecm} are estimation errors, the first of which is typically more challenging to control since it depends on the complexity of the hypothesis space. The final component is the approximation error $\CA(\lambda)= \inf_{f \in \CH_{B_*}} \CR(f) -\CR(f_*)+ \lambda \|f\|_{B_*}^2$, a quantity that has been extensively studied in the classical KRR \citep{cucker2007learning, de2021regularization}.

In what follows, we will concentrate on analyzing these estimation errors.

\subsection{Estimation error \uppercase\expandafter{\romannumeral1}}

The proof of this lemma follows a similar approach to that in \cite{cucker2007learning},
which is based on the following condition. 
\begin{ass}\label{ass:cv}
The covering number of $\mathbb B_{\CH,R}$ satisfies 
\begin{equation*}
\log \CN(\mathbb B_{\CH,R},\epsilon) \leq c_1 (R/\epsilon)^{s^*}
\end{equation*}
for some $c_1>0, s^*>0$.
\end{ass}
Assumption \ref{ass:cv} describes the complexity of hypothesis space using the concept of covering number, which is commonly used in literature \citep{cucker2007learning,zhou2002covering}. There are different metrics to measure the complexity of RHKSs. Covering number quantifies the compactness of a space by measuring how many subsets with a fixed radius are needed to cover it. While entropy number \citep{steinwart2008support,carl1990en}  represents the inverse concept by fixing the number of balls and determining the smallest radius needed to achieve that coverage. Eigenvalue decay \citep{caponnetto2007optimal}, on the other hand, describes the smoothness or compactness of the space through the rate at which eigenvalues of covariance operators diminish.  These measures are deeply interrelated, with covering numbers and entropy offering geometric and growth-based perspectives, while eigenvalue decay provides a spectral view of the hypothesis space. \cite[Chapter 5]{steinwart2008support} provides a more detailed discussion there, see also Assumption~\ref{ass:smooth}.

\begin{lem}\label{lem:item1}
Assume~$\ref{ass:bound}$,~$\ref{ass:1-1}$ and~$\ref{ass:cv}$. Fix $ f_0 \in \bigcup_{B \in \CB} \mathbb B_{B,R}$ and $\delta\in (0,1)$, the following holds with confidence at least $1-\frac{\delta}{2}$,
\begin{equation*}
\begin{split}
    \sup_{ f \in \mathop{\bigcup}\limits_{B \in \CB} \mathbb B_{B,R}}\left(\CR( f)-\CR(f_*)-(\widehat{\CR}( f)-\widehat{\CR}(f_*))\right) & \leq \frac{1}{2}\left( \CR(f_0)-\CR(f_*)\right)  + \notag \\
& + C_3 \max\{1, R^2\}\, Dd_*\,  \max\left\{1, \log\frac{2}{\delta}\right\}\left(\frac{1}{m}\right)^{\frac{1}{s^*+1}},
\end{split}
\end{equation*}
where 
\begin{equation}\label{equ:bnd1}
    C_3 = 360  \max\{1,M+\kappa\} \left( 1  +c_1+ \log(3 C_{\CH}) \right).
\end{equation}
\end{lem}
\begin{rem}\label{better}
 By inspecting the proof, it holds that for $m$ large enough ($m>m_0$ where $m_0$ is given by~\eqref{eq:13}), the factor $Dd_*$ can be replaced by $(Dd_*)^{\frac{1}{s^*+1}}$.
\end{rem}
\begin{proof} Without loss of generality, we can assume that $R\geq 1$. 
 Choose a function class
 \begin{align*}
     \mathcal{F} = \{F(x,y)| F(x,y) = (f(x)-y)^2-(f_*(x)-y)^2, f  \in \bigcup_{B \in \CB}\mathbb B_{B,R}\}.
 \end{align*}
 Then $\EE(F) = \CR(f)-\CR(f_*)$ and $\frac{1}{n}\sum_{i=1}^n F(x_i,y_i)=\widehat{\CR}(f)-\widehat{\CR}(f_*)$. 
By~\eqref{eq:6} 
$$\|f\|_\infty \leq \kappa \|f\|_B \leq \kappa R, $$ 
 and $|f_*(x)| \leq M,$ then 
 \begin{align*}
     &|F(z)|=|(f(x)-f_*(x))(f(x)+f_*(x)-2y)| \\
   &   \leq (\kappa R+M)(\kappa R+3M),
 \end{align*}
  $|F(z)-\EE(F)|\leq 2 (\kappa R+M)(\kappa R+3M)$ and $\EE(F^2) \leq \|f-f_*\|_{\rho_X}^2 (\kappa R+M)(\kappa R+3M)= (\kappa R+M)(\kappa R+3M) \EE(F)$.
 For $f_1, f_2 \in \mathop{\bigcup}\limits_{B \in \CB}\mathbb B_{B,R}$, we have 
 \begin{align*}
     |F_1(x,y)-F_2(x,y)|\leq 2(M+\kappa R)\|f_1-f_2\|_\infty.
 \end{align*}
It follows that 
 a $\frac{\epsilon}{2(M+\kappa R)}-$cover of $\mathop{\bigcup}\limits_{B \in \CB} \mathbb B_{B,R}$ yields an $\epsilon-$cover of $\mathcal{F}$, that is, 
$$\CN\left(\CF, \epsilon\right) \leq \CN\left(\mathop{\bigcup}\limits_{B \in \CB}\mathbb B_{B,R}, \frac{\epsilon}{2(M+\kappa R)}\right).$$
By taking $\alpha = \frac{1}{4}$ of Lemma 3.19 in \cite{cucker2007learning}, then with probability at least 
\begin{equation}\label{equ:equation_epsilon}
\begin{split}
     &1-\CN\left(\mathop{\bigcup}\limits_{B \in \CB}\mathbb B_{B,R}, \frac{\epsilon}{8(M+\kappa R)}\right)\exp \left\{ -\frac{3m\epsilon}{160 (\kappa R+M)(\kappa R+3M)}\right\} 
\end{split}
\end{equation}
there holds, for any $ f_0 \in \mathop{\bigcup}\limits_{B \in \CB}\mathbb B_{B,R}$
\begin{align*}
    &\mathop{\sup}\limits_{f \in \mathop{\bigcup}\limits_{B \in \CB}\mathbb B_{B,R}} \left( \CR(f)-\CR(f_*)-(\widehat{\CR}(f)-\widehat{\CR}(f_*)) \right)\leq \sqrt{\epsilon} \sqrt{\CR(f_0)-\CR(f_*)+\epsilon}\\
    &\leq \frac{1}{2}\left( \CR(f_0)-\CR(f_*)\right)+\epsilon.
\end{align*}
Fixed $\delta\in (0,1)$, we choose $\epsilon$ such that
\[
\CN\left(\mathop{\bigcup}\limits_{B \in \CB}\mathbb B_{B,R}, \frac{\epsilon}{8(M+\kappa R)}\right)\exp \left\{ -\frac{3m\epsilon}{160 (\kappa R+M)(\kappa R+3M)}\right\} \leq \delta/2.
\]
By \eqref{equ:equation_epsilon} and Lemma \ref{lem:coveringNo}, we need to solve 
\begin{equation}\label{eq:8}
\left(\frac{48(M+\kappa R)C_{\CH}R}{\epsilon}\right)^{Dd_*} \CN \left(\mathbb B_{\CH,R},\frac{\epsilon}{16(M+\kappa R)} \right) \cdot \exp \left\{ -\frac{3 m\epsilon}{160  (\kappa R+M)(\kappa R+3M)}\right\} \leq  \frac{\delta}{2}.
\end{equation}
Let $x=16R (M+\kappa R)/\epsilon >0$. Set
\[
A = Dd_* \log(3 C_{\CH}) +\log \frac{2}{\delta},\, \quad
B = \frac{3R }{10(\kappa R+3M)},\]
taking into account condition~(\ref{ass:cv}), the above inequality becomes 
\[Dd_* \log x + c_1 x^{s^*}- B m x^{-1} + A \leq 0.\]
Since $\log x \leq x^{s^*}$, the above inequality is satisfied if
\[
(Dd_*+c_1 ) x^{s^*} -  B m x^{-1} + A \leq 0 ,
\]
which is equivalent to
\[
x^{s^*+1}+a  x - mb\leq 0, \quad a=\frac{A}{Dd_*+c_1},\ b=\frac{B}{Dd_*+c_1}.
\]
The function $\varphi(x) = x^{s^*+1} + a x$ is continuous and strictly increasing on $(0,+\infty)$, 
tends to $0$ as $x \to 0^+$, and diverges to $+\infty$ as $x \to +\infty$. 
Hence there is a unique $x_m \in (0,+\infty)$ such that $\varphi(x_m) = mb$, 
and the above inequality is satisfied for all $0 < x < x_m$. 

Since $\varphi(1) = 1+a$, it holds that
\begin{equation}
  \label{eq:11}
  \begin{cases}
  x_m \geq 1 & \text{if } 1+a \leq mb, \\
  x_m < 1 & \text{if } 1+a > mb.
  \end{cases}
\end{equation}

If $1+a \leq mb$, since $s^*+1 > 1$, for all $x \geq 1$ we have
\[
\varphi(x) \leq (1+a)\, x^{s^*+1},
\]
so that
\[
x_m \geq \left(\frac{mb}{1+a}\right)^{\tfrac{1}{s^*+1}} \geq 1.
\]

If $1+a > mb$, then for all $x \leq 1$,
\[
\varphi(x) \leq (1+a)\, x,
\]
so that
\[
1 > x_m \geq \frac{mb}{1+a}.
\]

Hence inequity \ref{eq:8} is satisfied if
\[
\epsilon \geq 16 R (M+\kappa R) 
\left( \frac{10 (\kappa R + 3M)}{3R} \right)^{t_m} 
\left( Dd_* + c_1 + Dd_* \log(3C_{\CH}) + \log \frac{2}{\delta} \right)^{t_m} 
\left(\frac{1}{m}\right)^{t_m},
\]
where
\begin{equation}\label{eq:13}
t_m =
\begin{cases}
  1, & m < m_0, \\[4pt]
  \tfrac{1}{s^*+1}, & m \geq m_0,
\end{cases}
\qquad 
m_0 = \frac{3R}{10 (\kappa R + 3M)\, (Dd_* + c_1 + Dd_* \log(3C_{\CH}) + \log \tfrac{2}{\delta})}.
\end{equation}
Taking into account that $R\ge 1$, the above inequality is implied by
\[
\epsilon\geq 160  R^2  (M+\kappa )^{t_m+1} \left( Dd_* +c_1+ Dd_* \log(3 C_{\CH}) +\log \frac{2}{\delta} \right)^{t_m} \left(\frac{1}{m}\right)^{t_m}.
\]
Since $a+b\leq 2 ab$ for all $a,b\ge 1$ and $Dd_*\geq 1 $ then
\[
Dd_* +c_1+ Dd_* \log(3 C_{\CH}) +\log \frac{2}{\delta} \leq 2\, Dd_*\, (1+c_1 + \log(3 C_{\CH}) \max\{1, \log \frac{2}{\delta} \},
\]
so that bound in~\eqref{equ:bnd1} is a consequence of the fact that $\frac{1}{1+s^*}\leq t_m\le 1$. 
\end{proof}
This lemma shows that the largest error can be bounded in terms of 
$\CR(f)-\CR(f_*)$ for any $f \in \bigcup_{B \in \CB} \mathbb B_{B,R}$. 
In particular, by taking $f = \hflhbl$, the excess risk of HKRR appears in the upper bound, 
which is essential for the full excess risk analysis. 
Moreover, the radius $R$ in~\eqref{equ:bnd1} depends on assumptions about the hypothesis space 
and may vary across different settings. 
Specifically, under the assumptions of Theorem~\ref{thm:l_infty} and Theorem~\ref{cor:1} 
we can set $R = M/\sqrt{\lambda}$, 
while under the assumptions of Theorem~\ref{thm2}, $R$ can be chosen as 
$R \simeq \sqrt{\CA(\lambda)/\lambda} + 1$ (see Lemma~\ref{lem:improvedR}). 

\subsection{Estimation error \uppercase\expandafter{\romannumeral2}}\label{subsec:proof2} Note that to bound the item $\widehat{\CR}({f}_{\lambda}^{B_*})-\widehat{\CR}(f_*) - ({\CR}({f}_{\lambda}^{B_*})-\CR(f_*))$, the primary error arises from the difference between $\widehat{\CR}$ and $\CR$, which reflects the discrepancy between integration and discretization.

Recall that the approximation error is defined by
\[
\CA(\lambda)= \inf_{f \in \CH_{B_*}} \CR(f) -\CR(f_*)+ \lambda \|f\|_{B^*}^2= \CR(f^{B^*}_\lambda) + \lambda  \|f^{B^*}_\lambda\|_{B_*}^2
\]

The following is a restatement of a result from \cite{cucker2007learning}. We provide a proof for the sake of
completeness.

\begin{lem}\label{lem:proof2}
Assume~$\ref{ass:bound}$, the following holds with probability at least $1-\delta/2$, 
    \begin{equation*}
    \begin{split}
      &  \widehat{\CR}({f}_{\lambda}^{B_*})-\widehat{\CR}(f_*) - ({\CR}({f}_{\lambda}^{B_*})-\CR(f_*)) \\ & \leq  \left( \frac{14\kappa^2 \log(2/\delta)}{3m\lambda}+1\right) \CA(\lambda)+\frac{42M^2 \log(2/\delta)}{m}.
    \end{split}
    \end{equation*}
\end{lem} 
\begin{proof}
Consider a random variable $\xi$ with $f_\lambda^{B_*} \in \CH_{B_*}, \,\|f_\lambda^{B_*}\|_{\infty}\leq R'$ as
\[\xi(x,y) = (f_\lambda^{B_*}(x)-y)^2-(f_*(x)-y)^2. \]
 then $|\xi| \leq ( R'+3M)^2 =C_4^{'}$, $|\xi -\EE(\xi)|\leq 2C_4^{'}$, $\EE(\xi) = \CR(f_\lambda^{B_*})-\CR(f_*) \geq 0$ and $\EE(\xi^2) \leq C_4^{'}  \EE(\xi)$. Then by Bernstein's inequality, we have
\[\widehat{\CR}({f}_{\lambda}^{B_*})-\widehat{\CR}(f_*) - ({\CR}({f}_{\lambda}^{B_*})-\CR(f_*)) \leq \epsilon\]
holds with confidence $1-\delta/2$ with
\[\frac{\delta}{2}= \exp\left\{-\frac{m\epsilon^2}{2C_4^{'}\EE(\xi)+\frac{4}{3}C_4^{'}\epsilon}\right\}.\]
Solving the quadratic equation for $\epsilon$ tells us with confidence at least $1-\delta/2$
\begin{align*}
    &\widehat{\CR}({f}_{\lambda}^{B_*})-\widehat{\CR}(f_*) - ({\CR}({f}_{\lambda}^{B_*})-\CR(f_*)) \\
     & \leq\frac{ \frac{2}{3}C_4^{'} \log \frac{2}{\delta} + \sqrt{\frac{4}{9}(C_4^{'})^2(\log \frac{2}{\delta})^2 + 2mC_4^{'} \log \frac{2}{\delta} \EE(\xi)}}{m}\\
& \leq\frac{4C_4^{'} \log \frac{2}{\delta}}{3m} + \sqrt{\frac{2C_4^{'} \log \frac{2}{\delta} \EE(\xi)}{m}}.
\end{align*}
Applying the elementary inequality with the dual number $p'$ and $p$
\[
ab \leq \frac{1}{p}a^p + \frac{1}{p'}b^{p'} \quad \forall a, b > 0 
\]
to $p'=p=2$, \( a = \left(\frac{2C_4^{'} \log(2/\delta)}{m}\right)^{1/2} \), and \( b = \left(\EE(\xi)\right)^{1/2} \), we get
\[
\sqrt{\frac{2C_4^{'} \log \frac{2}{\delta} \EE(\xi)}{m}} \leq \frac{C_4^{'} \log \frac{2}{\delta}}{m} + \frac{1}{2}\EE(\xi).
\]
Hence, with confidence at least \( 1 - \delta/2 \), we have
\[
\frac{1}{m} \sum_{i=1}^{m} \xi(z_i) - \mathbb{E}(\xi) \leq \frac{4C_4^{'} \log \frac{2}{\delta}}{3m} + \frac{C_4^{'} \log \frac{2}{\delta}}{m} + \EE(\xi). 
\]
Note that for all \( \lambda > 0 \),
\[
\|f_\lambda^{B_*}\|_{B_*} \leq \sqrt{\CA(\lambda)/\lambda }\quad \text{and} \quad \|f_\lambda^{B_*}\|_\infty \leq \kappa \sqrt{\CA(\lambda)/\lambda}.
\]
In fact, since $ f_*$ is a minimizer of $\CR(f)$, we know that
$$
\lambda \|f_\lambda^{B_*}\|_{B_*}^2 \leq \CR(f_\lambda^{B_*}) - \CR(f_*) + \lambda \|f_\lambda^{B_*}\|_{B_*}^2 = \CA(\lambda).
$$
And the second follows from $ \|f_\lambda^{B_*}\|_\infty \leq \kappa \|f_\lambda^{B_*}\|_{B_*}. $
Thus, by taking $R'=\kappa \sqrt{\CA(\lambda)/\lambda}$,  it follows that { $C_4^{'} \leq 2\kappa^2\CA(\lambda)/\lambda+18M^2$} and
\begin{equation}\label{equ:bnd2}
\begin{split}
    \frac{1}{m} \sum_{i=1}^{m} \xi(z_i) - \mathbb{E}(\xi) \leq \left( \frac{14\kappa^2 \log(2/\delta)}{3m\lambda}+1\right) \CA(\lambda)+\frac{42M^2 \log(2/\delta)}{m}. 
\end{split}
\end{equation}
\end{proof}
This inequality offers a tighter bound with respect to the sample size \(m\) compared to classical concentration inequalities. This explains why we add the intermediate terms \(\widehat{\mathcal{R}}(f_*)\) and \(\mathcal{R}(f_*)\) in the error decomposition step (\ref{equ:dcm1}), which allows for a more refined bound on the variance of the random variables.

\subsection{Basic error bound}

As a consequence of the above results and the trivial bound~(\ref{eq:10}), 
we obtain our first main result.

\begin{thm}\label{thm:l_infty}
Assume~\ref{ass:bound}, \ref{ass:1-1}, and \ref{ass:cv}. 
Let $\delta > 0$. Then, with confidence at least $1-\delta$, 
\begin{equation}\label{equ:thm1.1}
\begin{split}
\CR(\hat{f}_\lambda^{\hat{B}_{d_*}})-\CR(f_*) 
&\leq 2 C_3 \max\!\left\{1,\frac{M^2}{\lambda} \right\} 
\left(\frac{1}{m}\right)^{\tfrac{1}{s^*+1}} \\
&\quad + \left( 4+\frac{28\kappa^2 \log(2/\delta)}{3m\lambda}\right)\CA(\lambda) 
+ \frac{84M^2 \log(2/\delta)}{m},
\end{split}
\end{equation}
where $C_3$ is given by~\eqref{equ:bnd1}.
\end{thm}

\begin{proof}
By~\eqref{eq:10}, bound in~\eqref{equ:errordecm} holds with $R = M/\sqrt{\lambda}$. 
Taking $f= \hflhbl$ on the right-hand side of Lemma~\ref{lem:item1} bounds 
\uppercase\expandafter{\romannumeral1}, 
Lemma~\ref{lem:proof2} bounds 
\uppercase\expandafter{\romannumeral2}, 
and the definition of $\CA(\lambda)$ yields \uppercase\expandafter{\romannumeral3}. 
Hence, with confidence at least $1-\delta$,
\[
\begin{split}
\CR(\hflhbl)-\CR(f_*) 
&\leq 2 C_3 \max\!\left\{1,\tfrac{M^2}{\lambda} \right\}  
\, Dd_*\, \max\!\left\{1, \log\frac{2}{\delta}\right\}
\left(\frac{1}{m}\right)^{\tfrac{1}{s^*+1}} \\
&\quad + 2\left(\frac{14\kappa^2 \log(2/\delta)}{3m\lambda}+1\right)\CA(\lambda) 
+ \frac{84M^2 \log(2/\delta)}{m} + 2\CA(\lambda).
\end{split}
\]
\end{proof}
Note that in the above error bound, $\CA(\lambda)$ is the approximation error, which depends on both the hypothesis space $\CH_{B_*}$ and the properties of the target function $f_*$. 
It decreases as $\lambda$ increases. 
By contrast, $s^*$ describes the complexity of the ambient hypothesis space $\CH$ (not $\CH_{B_*}$); it is typically determined by the intrinsic input dimension $d_*$ (rather than the ambient dimension $D$) and the regularity of $\CH$. 

The approximation error under so-called source conditions has been extensively studied in the context of classical kernel methods. 
The following is a standard formulation, where $L_{k_{B_*}}$ denotes the integral operator
\begin{equation}\label{eq:12}
L_{k_{B_*}}: L^2(X,\rho_X)\to L^2(X,\rho_X), 
\qquad 
(L_{k_{B_*}} f)(x) = \int_X k_{B_*}(x,x')\,f(x')\,\mathrm{d}\rho_X(x'),
\end{equation}
which is positive, so that for any $\theta > 0$, the fractional power $L_{k_{B_*}}^\theta$ is well defined by spectral calculus.  

\begin{ass}\label{ass:source} 
There exists $\theta \in (0,1]$ such that $f_* \in \mathrm{Range}\bigl(L_{k_{B_*}}^{\theta/2}\bigr)$.
\end{ass}

Assumption~\ref{ass:source} states that $f_*$ is not arbitrary, but belongs to a smoother subspace determined by the integral operator $L_{k_{B_*}}$ associated with the kernel $k_{B_*}$. 
The parameter $\theta$ quantifies the smoothness of $f_*$: larger values of $\theta$ correspond to smoother target functions, smaller hypothesis spaces, and therefore better approximation rates.

Under the above assumption, we have the following classical result 
\citep{cucker2007learning, de2021regularization}.

\begin{lem}\label{lem:source}
Under Assumption~\ref{ass:source}, 
\begin{equation}
\CA(\lambda) 
= \|f_{\lambda}^{B_*} - f_*\|_{\rho_X}^2 
+ \lambda \|f_{\lambda}^{B_*}\|_{B_*}^2 
\;\leq\; \lambda^\theta \|L_{B_*}^{-\tfrac{\theta}{2}} f_*\|_{\rho_X}^2.
\end{equation}
\end{lem}

We are now ready to state a main result, whose rate is slower than Theorem~\ref{thm2}. 
To simplify the statement, we introduce the following smoothness assumption on the mother kernel.

\begin{ass}\label{ass:smooth}
The mother kernel $k$ is defined on an open set $U \supset \CH_{R_0}\times \CH_{R_0}$ 
for some $R_0>1$, and $k \in C^r(U)$ for some $r \in \mathbb{N}$ with $r \geq 1$.  
\end{ass}

\begin{rem}
By Assumption~\ref{ass:bound} and the definition of $\CB$, we have
\[
\Omega \subset \mathbb B_{\CH,1} \subset \mathbb B_{\CH,R_0},
\]
so that we can apply \cite[Th.~6.26]{steinwart2008support}.
\end{rem}

\begin{rem}
As shown in \citep[Th.~6.26 and the subsequent remark]{steinwart2008support}, under Assumption \ref{ass:smooth}, $s^*$ 
in Assumption~\ref{ass:cv} equals to $s^* = d_*/r$.
Moreover, since $r \geq 1$, it also implies Assumption~\ref{ass:1-1}. 

\end{rem}

Examples of kernels satisfying these assumptions include the Matérn kernel with parameter $2.5$, 
polynomial kernels of degree greater than $2$, and others. 
The following theorem is an immediate consequence of the error decomposition and the results above, 
and its proof is therefore omitted.

\begin{thm}\label{cor:1}
Assume~$\ref{ass:bound}$, $\ref{ass:source}$ and $\ref{ass:smooth}$. Fix  $0<\delta<1$, with probability at least $1-\delta$  
\begin{align}\label{equ:cor1}
\CR(\hflhbl)-\CR(f_*) & \leq   2 C_3 \max\left
\{1,\frac{M^2}{\lambda} \right\} \, Dd_*\,  \max\{1, \log\frac{2}{\delta}\}\left(  \frac{1}{m} \right)^{\frac{r}{d_*+r}} +\notag\\
& \quad +\frac{28\kappa^2 \log(2/\delta)}{3m\lambda^{1-\theta}}+\frac{84M^2 \log(2/\delta)}{m}+4 \lambda^\theta,
\end{align}
where $C_3$ is given by \eqref{equ:bnd1}. 
By taking $\lambda= M^2 m^{- \frac{1}{(1+\theta)(1+d_*/r)}}$, there holds
\[\CR(\hflhbl)-\CR(f_*)\leq C_4 \left( Dd_* \right)  \log \frac{2}{\delta}\left(\frac{1}{m} \right)^{\frac{r\theta}{(1+\theta)(d_*+r)}}\]
with $C_4=M^2 \left( C_3+  28 M^2\kappa^2 +88 \right) $. 
\end{thm}

Classical results with the same assumptions on hypothesis space are of order $m^{n\theta/(1+\theta)(D+n)}$ \citep{cucker2007learning}.  When the input dimension $D$ is exceptionally large, as is often the case with the rise of big data, the rate is adversely impacted by the curse of dimensionality. However, as noted in \thmref{cor:1}, the exponential dependence on $m$ is governed by $d$ rather than $D$, with the dependence on $D$ being polynomial, which helps mitigate the curse of dimensionality. Moreover, when \(s\) is sufficiently large and \(\theta = 1\), the rate in Theorem \ref{cor:1} asymptotically reduces to \(O(m^{-1/2})\). Although this is slower than the \(O(m^{-1})\) rate of Theorem \ref{thm2} under the additional sample size condition, the proof techniques are essentially the same.

\subsection{Refined error bound: proof of Theorem~\ref{thm2}}\label{subsec:proof1}

The following lemma provides a more refined bound  than  the bound in~\eqref{eq:10} under stricter assumptions on the sample size. A detailed proof is available in \cite[Lemma 8.19]{cucker2007learning}.

\begin{lem}\label{lem:improvedR}
Under Assumptions~\ref{ass:bound}, \ref{ass:1-1}, and \ref{ass:smooth}, 
suppose $\zeta < 1/(1+s^*)$ with  $s^* = d_*/r$ and
choose $\lambda_m = m^{-\zeta}$. 
Fix $0 < \delta < 1$. 
Then, with confidence at least 
with confidence \( 1 - 3\delta/(1/(1 + s^* )- \zeta) \),
we have
\begin{equation}\label{equ:improvedR}
    \|\hflhbl\|_{\hat{B}_{d_*}} \leq c_2\sqrt{\log(2/\delta)}(\sqrt{\CA(\lambda_m)/\lambda_m} + 1) =R^*
\end{equation}
for all $m \geq m_\delta$. 
Here $c_2 > 0$ is a constant depending only on $s^*$, $\zeta$, $\kappa$, and $M$, and
\begin{equation} \label{eq:14}
  m_\delta=\max \left\{ (108/c_1)^{1/s^*}(\log (2/\delta))^{1+1/s^*}, (1/2c_3)^{2/(\zeta-1/(1+s^*))}  \right\}
\end{equation}
with $c_3 = (2\kappa+5)(108 c_1)^{1/(1+s^*)}$.
\end{lem}

Now we are ready to prove \thmref{thm2}.
\begin{proof}[Proof of \thmref{thm2}]
Assumption~\ref{ass:0} states that Assumptions~\ref{ass:bound}, \ref{ass:1-1}, 
and \ref{ass:smooth} hold. 
Recall that Assumption~\ref{ass:smooth} implies Assumption~\ref{ass:cv} 
with $s_* = d_*/r$. 
We can now apply the error decomposition in Lemma~\ref{lem:errdcp}, 
together with the corresponding bounds for each term, 
to derive the excess risk.


Take $f_0 = \hflhbl$ on the right-hand side of the inequality in Lemma~\ref{lem:item1} 
with $R = R^*$ given in~\eqref{equ:improvedR} to bound item~\uppercase\expandafter{\romannumeral1}, 
combine with Lemma~\ref{lem:proof2} to bound item~\uppercase\expandafter{\romannumeral2}, 
and use Lemma~\ref{lem:source} for item~\uppercase\expandafter{\romannumeral3}. 
Then, with confidence at least $1-\delta$,
\begin{equation}\label{equ:improved_lbd}
\begin{split}
  \CR(\hflhbl)-\CR(f_*)
  &\leq 2 C_3 D d_* c_2^2 \log^2 \!\left(\tfrac{2}{\delta}\right) 
        \left(\frac{\CA(\lambda)}{\lambda}+1\right) 
        \left(\frac{1}{m}\right)^{\frac{1}{s^*+1}} \\
  &\quad+ 2\left(\frac{14\kappa^2 \log(2/\delta)}{3m\lambda}+1\right)\CA(\lambda) 
        + \frac{84 M^2 \log(2/\delta)}{m} + 2 \CA(\lambda),
\end{split}
\end{equation}
for all $m \geq m_\delta$.
Moreover, since $\CA(\lambda) \leq \lambda^\theta$ by Lemma~\ref{lem:source} 
and $\lambda = m^{-\zeta}$, we obtain
\begin{align}\label{equ:combinedbnd}
   \CR(\hflhbl)-\CR(f_*) 
   &\leq 28 \kappa^2 \log(2/\delta)\, m^{-1+(1-\theta)\zeta}
      + \frac{84M^2 \log(2/\delta)}{m}
      + 4 m^{-\theta \zeta} \notag \\
   &\quad+ 2 C_3 D d_* c_2^2 \log^2(2/\delta)\, 
      m^{-\frac{r}{d_*+r}+(1-\theta)\zeta}.
\end{align}
Since $\delta < 2/e$, we have
\[
1 \leq \log(2/\delta) \leq \log^2(2/\delta).
\]
Therefore, the dominant terms with respect to $m$ in~\eqref{equ:combinedbnd} 
are the third and the last ones, because 
\(
-1 > -1+(1-\theta)\zeta > -\frac{r}{d_*+r}+(1-\theta)\zeta.
\)
If $\zeta < r/(d_*+r)$, then the rate becomes
\[
\CR(\hflhbl)-\CR(f_*) 
   \leq C_1 D d_* \log^2(2/\delta)\, m^{-\theta \zeta},
\]
where
\begin{equation}\label{eq:15}
C_1 = 28 \kappa^2 + 84M^2 + 4 + 2 C_3 c_2^2.
\end{equation}
This completes the proof.
\end{proof}

\section{Proofs of other excess risk bounds in Section \ref{sec:thms} }\label{apx:thm_proof}
In this part, we will give the proofs of excess risk rates for Theorems \ref{thm:ny} and \thmref{thm:cv_lbd}.

\subsection{Rate of Nystr\"om approximation: proof of Theorem \ref{thm:ny}}
Before proving Theorem \ref{thm:ny}, we introduce an alternative strategy for selecting the Nystr\"om points based on approximate leverage scores (ALS), referred to as the ALS Nystr\"om approximation. This method, together with some necessary definitions, will be included in the next theorem.
The leverage scores associated to points $(x_i)_{i=1}^m$ are
\[
(\ell_i(t))_{i=1}^m,\quad
\ell_i(t)=\left(\hat K^B(\hat K^B + t\,m\,I)^{-1}\right)_{ii},
\quad i\in\{1,\dots,m\}, \quad t>0,
\]
where $(\hat K^B)_{ij}=k(Bx_i,Bx_j)$. Computing these scores exactly can be challenging in practice; thus, one may consider approximations $(\hat\ell_i(t))_{i=1}^m$ \citep{drineas2012fast,rudi2015less}. Given $t_0>0$, $T\ge1$ and confidence level $\delta>0$, we say that $(\hat\ell_i(t))_{i=1}^m$ are $(T,t_0)$-approximate leverage scores with probability at least $1-\delta$, if
\[
\frac{1}{T}\ell_i(t) \le \hat{\ell}_i(t) \le T \ell_i(t), \quad t\ge t_0, \quad i=1,\ldots,m.
\]
The ALS
sampling selects the Nystr\"om points $(\tilde x_i)_{i=1}^{\tilde m}$ independently with replacement from the training set, where each $x_i$ is selected with probability $p_t(i) = \hat{\ell}_i(t)/\sum_j \hat{\ell}_j(t)$.

\begin{thm}[Extension of \thmref{thm:ny}]\label{thm:extension2}
Under the same assumptions as \thmref{thm2},  the following holds with probability at least $1-\delta$,
\[\CR(\hat{f}_{\lambda}^{\hat{B}_{d_*,\tilde m}})-\CR(f_*) \leq {C_2  Dd_*  \log^2 ({2}/{\delta})} \left( {m}\right)^{-\theta \zeta}\]
with $C_2$ given explicitly in \eqref{eq:C_nys} under conditions:
\begin{enumerate}[left=0pt]
    \item for plain Nystr\"om, $\tilde m \geq 67\log \frac{4\kappa}{\lambda \delta} \lor 5 \CN_{B_*,\infty}(\lambda) \log \frac{4 \kappa}{\lambda \delta}$.
    \item for ALS Nystr\"om and $(T,t_0)$-approximate leverage scores with subsampling probabilities $p_t$, 
\begin{align*}
&m  \ge  1655 \kappa + 223\kappa \log\frac{2\kappa}{\delta}, \\
& t_{0}\vee\frac{19\kappa}{m}\log\frac{2m}{\delta}
\le \lambda \le \bigl\|\Sigma_{B_*}\bigr\|,\\
&\tilde m  \ge (334 \vee  
78\,T^{2}\,\mathcal{N}_{B_*}(\lambda))\log\frac{8m}{\delta}.
\end{align*}
\end{enumerate}

\end{thm}


\begin{proof}[Proof of \thmref{thm:ny}(\thmref{thm:extension2})]
Let $P_{B,\tilde m}: \CH_B \to \CH_B$ denote the orthogonal projection from $\CH_{B,m}$ onto the subspace $\CH_{B,\tilde m}\subset \CH_{B,m} \subset \CH_B$.
  Recall that
    \(\hat f_\lambda^{\hat B_{d_*,\tilde m}}
  = \mathop{\argmin}_{B\in\CB}\min_{f\in\CH_{B,\tilde m}}\widehat R_\lambda(f).\)
We decompose the excess risk as follows:
  \begin{equation}\label{equ:dcm2}
  \begin{split}
     & \CR(\hat f_\lambda^{\hat B_{d_*,\tilde m}})-\CR(f_*)\\ 
      &= \CR(\hat f_\lambda^{\hat B_{d_*,\tilde m}})-\widehat{\CR}_\lambda(\hat f_\lambda^{\hat B_{d_*,\tilde m}})+\widehat{\CR}_\lambda(\hat f_\lambda^{\hat B_{d_*,\tilde m}})-\widehat{\CR}_\lambda(P_{B_*,\tilde m} f_\lambda^{B_*}) \\
      & \quad +\widehat{\CR}_\lambda(P_{B_*,\tilde m} f_\lambda^{B_*})-\CR(P_{B_*,\tilde m} f_\lambda^{B_*})+\CR(P_{B_*,\tilde m} f_\lambda^{B_*})-\CR(f_*)+\lambda\|P_{B_*,\tilde m} f_\lambda^{B_*}\|_\CH^2\\
      & \leq \CR(\hat f_\lambda^{\hat B_{d_*,\tilde m}})-\widehat{\CR}_\lambda(\hat f_\lambda^{\hat B_{d_*,\tilde m}}) +\widehat{\CR}(P_{B_*,\tilde m} f_\lambda^{B_*})-\CR(P_{B_*,\tilde m} f_\lambda^{B_*})\\
      & \quad +\CR(P_{B_*,\tilde m} f_\lambda^{B_*})-\CR(f_*)+\lambda\|P_{B_*,\tilde m} f_\lambda^{B_*}\|_\CH^2 \\
      & \leq \left\{\CR(\hat f_\lambda^{\hat B_{d_*,\tilde m}})-\widehat{\CR}(\hat f_\lambda^{\hat B_{d_*,\tilde m}})\right\} + \left\{\widehat{\CR}(P_{B_*,\tilde m} f_\lambda^{B_*})-\CR(P_{B_*,\tilde m} f_\lambda^{B_*})\right\}\\
      & \quad +\left\{\CR(P_{B_*,\tilde m} f_\lambda^{B_*})-\CR(f_*)+\lambda\|P_{B_*,\tilde m} f_\lambda^{B_*}\|_\CH^2 \right\}.
\end{split}
 \end{equation}
 Here, the first inequality follows because $\hat f_\lambda^{\hat B_{d_*,\tilde m}}$ minimizes $\widehat{\CR}_\lambda$ over $\CH_{B,\tilde m}$. The first two terms are called estimation errors, whose controls are identical to those in the proof of \thmref{thm2} since $\CH_{B,\tilde m } \subset \CH_B$. For the last term, we further decompose it as 
  \begin{align*}
      &\CR(P_{B_*,\tilde m} f_\lambda^{B_*})-\CR(f_*)+\lambda\|P_{B_*,\tilde m} f_\lambda^{B_*}\|_\CH^2 \\
      & \leq  \|P_{B_*,\tilde m} f_\lambda^{B_*}-f_*\|_{\rho_X}^2 + \lambda\|f_\lambda^{B_*}\|_{B^*}^2 \\
      & \leq \|P_{B_*,\tilde m} f_\lambda^{B_*}-f_\lambda^{B_*}\|_{\rho_X}^2+\|f_\lambda^{B_*}-f_*\|_{\rho_X}^2 + \lambda\|f_\lambda^{B_*}\|_{B^*}^2.
  \end{align*}
  Denote the first term as $\CC(\lambda):=\|P_{B_*,\tilde m} f_\lambda^{B_*}-f_\lambda^{B_*}\|_{\rho_X}^2$, which is the so-called computational error as given in \cite{rudi2015less}. And the left part is the approximation error $\CA(\lambda)$. Let 
  {$\Sigma _{B_*}$ be the covariance operator associated with the kernel $k_{B_*}$, see~\eqref{eq:covariance_op},}
  by the relationship between $L_2$ norm and RKHS norm (\ref{equ:l2_H}), we have
  \begin{align*}
      \|P_{B_*,\tilde m} f_\lambda^{B_*}-f_\lambda^{B_*}\|_{\rho_X}^2 &=\|\Sigma_{B_*}^{1/2}(I-P_{B_*,\tilde m})f_\lambda^{B_*}\|_{B^*}^2\\
      & \leq \|(I-P_{B_*,\tilde m})\Sigma_{B_*}^{1/2}\|^2 \|f_\lambda^{B_*}\|_{B^*}^2.
  \end{align*}
Using the estimate \(\|f_\lambda^{B_*}\|_{\mathcal{H}_{B_*}}^2 \le \CA(\lambda)/\lambda\) and applying \cite[Lemma 6]{rudi2015less}, if $\tilde m \geq 67\log \frac{4\kappa^2}{\lambda \delta} \lor 5 \CN_\infty(\lambda) \log \frac{4 \kappa^2}{\lambda \delta}$, then with probability at least $1-\delta$, $\|(I-P_{B_*,\tilde m})\Sigma_{B_*}^{1/2}\|^2 \leq 3 \lambda$. Therefore,   
\begin{align}\label{eq:apx_nys}
    \CR(P_{B_*,\tilde m} f_\lambda^{B_*})-\CR(f_*)+\lambda\|P_{B_*,\tilde m} f_\lambda^{B_*}\|_\CH^2 \leq 3 \lambda \frac{\CA(\lambda)}{\lambda} + \CA(\lambda) = 4 \CA( \lambda),
\end{align}
which coincides with the order of the approximation error in the proof of \thmref{thm2}. 

By adding the terms $\CR(f_*)$ and $\widehat{\CR}(f_*)$ to the first two components of \eqref{equ:dcm2}, 
the estimation errors can be controlled using the same argument as in \thmref{thm2}. Specifically, applying Lemma~\ref{lem:item1} by taking $f_0=\hat f_\lambda^{\hat B_{d_*,\tilde m}}$, 
\(
R=c_2\sqrt{\log(2/\delta)}\Bigl(\sqrt{{\CA(\lambda_m)}/{\lambda_m}}+1\Bigr), s^*={d_*}/{r},
\)
and  together with Lemma~\ref{lem:proof2}, we obtain
\begin{align*}
\CR(\hat f_\lambda^{\hat B_{d_*,\tilde m}})-\CR(f_*) 
&\leq 2 C_3 \,Dd_*\, c_2^2 \log^2 \!\Bigl(\tfrac{2}{\delta}\Bigr) 
     \Bigl(\tfrac{\CA(\lambda)}{\lambda}+1\Bigr)
     \Bigl(\tfrac{1}{m}\Bigr)^{\frac{1}{s^*+1}} \\
&\quad +2\Bigl(\tfrac{14\kappa^2 \log(2/\delta)}{3m\lambda}+1\Bigr)\CA(\lambda)
      +\tfrac{84M^2 \log(2/\delta)}{m}
      +8\CA(\lambda).
\end{align*}
The remainder of the proof follows the same computations as in \thmref{thm2}, yielding
\[
\CR(\hat f_\lambda^{\hat B_{d_*,\tilde m}})-\CR(f_*) 
   \leq C_2 \,Dd_*\, \log^2 \!\Bigl(\tfrac{2}{\delta}\Bigr)\, m^{-\theta\zeta},
\]
where
\begin{equation}\label{eq:C_nys}
C_2 = 28\kappa^2 + 84M^2 + 10 + 2 C_3 c_2^2.
\end{equation}

Note that the assumption for ALS Nystr\"om ensures that 
$\|(I-P_{B_*,\tilde m})\Sigma_{B_*}^{1/2}\|^2 \leq 3\lambda$ \citep{rudi2015less}, 
and hence \eqref{eq:apx_nys} follows. 
The proof proceeds exactly as in the case of plain Nystr\"om and is therefore omitted.

\end{proof}

\subsection{Adaptivity: proof of \thmref{thm:cv_lbd}}
The following gives the proof of \thmref{thm:cv_lbd}. We recall the following concentration inequality (see for example \cite{caponnetto2010cross}.  Let \( Z_1, \ldots, Z_{m'} \) be a sequence of i.i.d. real random variables with mean \( \mu \), such that \( |Z_i| \leq a \) a.s. and \( \mathbb{E}[|Z_i - \mu|^2] \leq \sigma^2 \). Then for all \( \alpha, \varepsilon > 0, \)
\begin{equation}\label{equ:oracle}
P \left( \left| \frac{1}{m'} \sum_{i=1}^{m'} Z_i - \mu \right| \geq \varepsilon + \alpha \sigma^2 \right) \leq 2e^{-\frac{6m'\alpha\varepsilon}{3 + 4\alpha a}}.
\end{equation}

\begin{proof}[Proof of \thmref{thm:cv_lbd}]
Let 
$$(\tilde d, \tilde{\lambda}) = \mathop{\argmin}_{(d,\lambda) \in \Gamma} \EE[(T_M \hat{f}_\lambda^{\hat{B}_{d}}(x')-y')^2  ].$$  
Here, the expectation is taken with respect to the pair $(x',y')$ according to $\rho$.  

For any $(d,\lambda)\in\Gamma$, we apply~\eqref{equ:oracle}  with the choice $Z_i= Z_i^{(d,\lambda)} := \left(T_M \hat{f}_\lambda^{\hat{B}_{d}}(x'_i)-y'_i\right)^2-(f_*(x'_i)-y'_i)^2$, with $i=1,\ldots,m'$. Note that $|Z_i^{(d,\lambda)}| \leq 8M^2$ and $\EE(Z_i^{(d,\lambda)})=\CR(T_M \hat{f}_\lambda^{\hat{B}_{d}})-\CR(f_*)>0$,  then $\EE((Z_i^{(d,\lambda)})^2) \leq 8M^2\EE(Z_i^{(d,\lambda)})$. Therefore, by a union bound over $\Gamma$, for all $(d,\lambda)\in \Gamma$,   with probability at least $1-\delta$, there hold
\[\frac{1}{m'} \sum_{i=1}^{m'} Z_i^{(d,\lambda)}\leq (1+8\alpha M^2) \EE(Z_i^{(d,\lambda)}) +\epsilon'\]
and
\[\EE(Z_i^{(d,\lambda)}) \leq \frac{1}{1-8\alpha M^2} \frac{1}{m'} \sum_{i=1}^{m'} Z_i^{(d,\lambda)} +\frac{\epsilon'}{1-8\alpha M^2}, \quad \text{for} \quad \alpha <\frac{1}{8M^2},\]
where $\epsilon' = \frac{3+32\alpha M^2}{6m' \alpha} \log \frac{2DN}{\delta}$.  Therefore,
\begin{align*}
    & \CR(T_M \hat{f}_{{\lambda}}^{\hat{B}_{{d}}}) -\CR(f_*) = \EE(Z_i^{{(d,\lambda)}}) \\ &\leq \frac{1}{1-8\alpha M^2} \frac{1}{m'} \sum_{i=1}^{m'} Z_i^{(d,\lambda)} +\frac{\epsilon'}{1-8\alpha M^2}\\
      &\leq\frac{1+8\alpha M^2}{1-8\alpha M^2} \EE(Z_i^{(d,\lambda)})+ \frac{2\epsilon'}{1-8\alpha M^2}.
\end{align*}
Then, since $(\hat{d},\hat{\lambda}),(\tilde{d},\tilde{\lambda})\in \Gamma$ and $(\hat{d},\hat{\lambda})$ is the minimizer of $\frac{1}{m'} \sum_{i=1}^{m'} \Bigl( T_M \hat{f}_\lambda^{\hat{B}_d}(x'_i)-y'_i \Bigr)^2$,
\[
\begin{split}
\CR(T_M \hat{f}_{\hat{\lambda}}^{\hat{B}_{\hat{d}}}) -\CR(f_*)  & \leq \frac{1}{1-8\alpha M^2} \frac{1}{m'} \sum_{i=1}^{m'} Z_i^{(\hat{d},\hat{\lambda})} +\frac{\epsilon'}{1-8\alpha M^2} \\
& \leq \frac{1}{1-8\alpha M^2} \frac{1}{m'} \sum_{i=1}^{m'} Z_i^{(\tilde{d},\tilde{\lambda})} +\frac{\epsilon'}{1-8\alpha M^2} \\
&\leq \frac{1+8\alpha M^2}{1-8\alpha M^2} \EE(Z_i^{(\tilde{d},\tilde{\lambda})})+ \frac{2\epsilon'}{1-8\alpha M^2}.
\end{split}
\]
With the  choice of $\alpha = 1/(24M^2)$, we get that 
\[
\begin{split}
\CR(T_M \hat{f}_{\hat{\lambda}}^{\hat{B}_{\hat{d}}}) -\CR(f_*)  & \leq 2 (\CR(T_M \hat{f}_{\tilde{\lambda}}^{\hat{B}_{\tilde{d}}}) -\CR(f_*) )+ \frac{52M^2}{m'} \log \frac{2DN}{\delta} \\
& \leq 2 (\CR(T_M \hat{f}_{\tilde{\lambda}}^{\hat{B}_{{d_*}}}) -\CR(f_*)) + \frac{52M^2}{m'} \log \frac{2DN}{\delta} \\
& \leq 2 (\CR(\hat{f}_{\tilde{\lambda}}^{\hat{B}_{{d_*}}}) -\CR(f_*)) + \frac{52M^2}{m'} \log \frac{2DN}{\delta} \\ 
& \leq 2 U(\tilde{\lambda})  + \frac{52M^2}{m'} \log \frac{2DN}{\delta},
\end{split}
\]
where $U(\lambda)$ is the right-hand side of \eqref{equ:improved_lbd}. The second and third steps are obtained from the definition of $(\tilde d, \tilde \lambda)$ and the fact that, for any function $f$, 
\[
\CR(T_M f) - \CR(f_*) \leq \CR(f) - \CR(f_*).
\]
As in \cite[Lemma~2]{chirinos2024learning}, by the definition of the grid $\Lambda$, there exists $q \in [1,Q]$ such that 
$\tilde{\lambda} = q \lambda_*$, where $\lambda_*$ is the optimal parameter choice according to the bound in~\eqref{equ:improved_lbd}, namely
\[
\lambda_* = \mathop{\argmin}_{\lambda>0} U(\lambda) = m^{-\zeta}.
\]
Furthermore, one can easily show that 
\[
U(q\lambda) \leq q^\theta U(\lambda)
\]
for a suitable $\theta$, so that 
\begin{align*}
    \CR(T_M \hat{f}_{\hat{\lambda}}^{\hat{B}_{\hat{d}}}) - \CR(f_*)
    &\leq 2 q^\theta U(\lambda_*) + \frac{52 M^2}{m'} \log \frac{2ND}{\delta}.
\end{align*}
\end{proof}

\section{Technical details on optimization for HKRR}\label{sec:opt_details}

In this section, we provide further insights into the optimization procedure introduced in Section \ref{subsec:opt}.

We start by presenting the complete versions of Algorithms \ref{alg:MIGD} and \ref{alg:AGD}, which include additional implementation specifics:
\begin{itemize}[left=0pt]
    \item a \textbf{non-monotone Armijo backtracking strategy} to automatically tune the learning-rates $s_\alpha$ and $s_B$. This classical approach, featuring a contraction parameter $\rho\in(0,1)$, a dilatation parameter $\delta\in(0,1]$ and a decay rate parameter $c>0$, allows avoiding fine-tuning while adapting to the local behavior of $\hat\CL$ thanks to its non-monotonicity \citep{calatroni2019backtracking}.
    \item a \textbf{projection step} in $B$ to handle the constraint $B\in\CB$. At each iteration, $B^i$ is obtained through a projected gradient step involving the projection operator $\CP_\CB$. In practice, this step consists in thresholding the singular values of the matrix $B^i$.
\end{itemize}

\paragraph{\textbf{Variable Projection (VarPro)}} The complete version of the method is specified in Algorithm \ref{alg:MIGD_detail} below.
\begin{algorithm}[H]
\caption{VarPro}\label{alg:MIGD_detail}
\begin{algorithmic}
    \STATE \textbf{Require} $B^0$, $s_{B,-1}>0$, $s_{max}>0$, $\rho\in(0,1)$, $\delta\in(0,1]$ and $c>0$.
    \STATE $\alpha^{0}=\underset{\alpha\in \R^{\tilde m}}{\argmin}~\hat\CL\left(B^{0},\alpha\right)$
    \FOR{$i$ in $0,1,...$}
    \STATE $s_{B,i} = \min\left\{\frac{s_{B,i-1}}{\rho\delta},s_{max}\right\}$
    \REPEAT
    \STATE $s_{B,i} = \rho s_{B,i}$
    \STATE $B^{i+1}=\CP_\CB \left(B^{i}-s_{B,i}\nabla_B\hat\CL\left(B^{i},\alpha^{i}\right)\right)$
    \UNTIL $\hat\CL\left(B^{i+1},\alpha^{i}\right)-\hat\CL\left(B^{i},\alpha^{i}\right)<-cs_{B,i}\left\|\nabla_B\hat\CL\left(B^i,\alpha^i\right)\right\|^2$
    \STATE $\alpha^{i+1}=\underset{\alpha\in \R^{\tilde m}}{\argmin}~\hat\CL\left(B^{i+1},\alpha\right)$
    \ENDFOR
    \RETURN $\left(B^{i+1}, \alpha^{i+1}\right)$
\end{algorithmic}
\end{algorithm}

This method involves a linesearch strategy for determining the sequence of learning-rates $\left(s_{B,i}\right)_{i\in\N}$. The following lemma guarantees the well-posedness of this procedure under continuity conditions on the kernel and a boundedness assumption on the sequence $\left(\alpha_i\right)_{i\in\N}$.

\begin{lem}\label{lem:LB}
    Let $k(\cdot,\cdot)$ be a continuous, twice differentiable kernel with continuous second order derivatives, and $\left(\alpha^i\right)_{i\in\N}$ be a bounded sequence. Then, there exists $L_1>0$ such that,
    \begin{equation}
        \forall B\in \mathcal{B},~\forall i\in\N,\quad \left\|\nabla^2_B\hat\CL\left(B,\alpha^i\right)\right\|_{2}\leq L_1,
    \end{equation}
    implying that for any $i\in\N$, $B\mapsto\hat\CL\left(B,\alpha^i\right)$ is $L_1$-smooth. It follows that the linesearch procedure used to generate $\left(s_{B,i}\right)_{i\in\N}$ in Algorithm \ref{alg:MIGD_detail} and Algorithm \ref{alg:AGD_detail} is well defined, and $s_{B,i}\geq\min\left\{\frac{2\rho(1-c)}{ L_1},s_{B,-1}\right\}$ for any $i\in\N$.
\end{lem}
\begin{proof}[Proof of Lemma \ref{lem:LB}]
    Since we consider a dataset of dimension $m<+\infty$, we can ensure that there exists $C<+\infty$ such that for any $j\in\{1,\ldots,m\}$, $\|x_j\|<C$. As a consequence, for any $B\in\CB$,
    \begin{equation*}
        \|Bx_j\|\leq\|B\|_\infty\|x_j\|\leq C.
    \end{equation*}
    From the above inequality and the continuity of $k$ and its first and second order derivatives, we get that $k(Bx_{j_1},Bx_{j_2})$, $\|k_p(Bx_{j_1},Bx_{j_2})\|$ and $\|k_{pq}(Bx_{j_1},Bx_{j_2})\|$, where $(p,q)\in\{1,2\}^2$, can be bounded independently from $B$ and $(j_1,j_2)\in\{1,\ldots,m\}^2$.\\
    It is then straightforward to show that $\left\|\nabla^2_B\hat\CL\left(B,\alpha^i\right)\right\|_{2}$ can be bounded independently from $B$ and $i$ since it only depends on $\left(k(Bx_{j_1},Bx_{j_2})\right)_{j_1,j_2=1}^m$, $\left(k_p(Bx_{j_1},Bx_{j_2})\right)_{j_1,j_2=1}^m$, $\left(k_{pq}(Bx_{j_1},Bx_{j_2})\right)_{j_1,j_2=1}^m$, $\left(y_{j}\right)_{j=1}^m$ and $\alpha^i$ which can be bounded independently from $i$ by assumption.\\
    The Lipschitz continuity of $B\mapsto\hat\CL\left(B,\alpha^i\right)$ directly ensures that the Armijo backtracking procedure in Algorithm \ref{alg:AGD_detail} is well defined. In particular, if $B^i$ is not a critical point of $\Psi:B\mapsto \hat\CL\left(B,\alpha^i\right)+i_\CB(B)$, the Armijo condition is satisfied for any step size $s\leq\frac{2(1-c)}{L_1}$. Since the $s_{B,i}$ is updated by multiplying it by $\rho$ until the condition is satisfied, we can deduce the desired lower bound.
\end{proof}

We can then exploit this lemma to prove the convergence of Algorithm \ref{alg:MIGD_detail}, using the Kurdyka-\L{}ojasiewicz property of $\hat{L}$.

\begin{thm}\label{thm:cvg_MIGD}
    Let $k(\cdot,\cdot)$ be an analytic kernel. Let $\left(B^i\right)_{i\in\N}$ and $\left(\alpha^i\right)_{i\in\N}$ be the sequences generated by Algorithm \ref{alg:MIGD_detail} and suppose that for any $i\in\N$, $\lambda_{min}\left(\matK^{B^i}_{\tilde m\tilde m}\right)\geq \sigma>0$ where $\lambda_{min}$ denotes the smallest eigenvalue. Then, \textbf{the sequence $\left(B^i,\alpha^i\right)_{i\in\N}$ converges to a critical point} of $\Psi:B,\alpha\mapsto\CL(B,\alpha)+i_\CB(B)$ as $i$ goes to infinity. In addition, the sequences $\left(B^i\right)_{i\in\N}$ and $\left(\alpha^i\right)_{i\in\N}$ have finite length, i.e.
    \begin{equation*}
        \sum_{i=0}^{+\infty}\|B^{i+1}-B^i\|<+\infty,\quad\sum_{i=0}^{+\infty}\|\alpha^{i+1}-\alpha^i\|<+\infty,
    \end{equation*}
    and there exists $C>0$ such that after $N$ iterations, $\left(B^N,\alpha^N\right)$ is a critical point of $\Psi$ or 
    \begin{equation}\label{eq:grad_MIGD}
        \min_{0\leq i\leq N}\left\|\nabla_B\hat\CL\left(B^i,\alpha^i\right)\right\|^2\leq\frac{C}{N}.
    \end{equation}
\end{thm}

\begin{proof}[Proof of Theorem \ref{thm:cvg_MIGD}]
    The proof of this theorem relies on \cite[Theorem~2.9]{attouch2013convergence} stating convergence towards critical points for algorithms minimizing functions having the Kurdyka-\L{}ojasiewicz property under several assumptions. It requires to prove the following points:
    \begin{enumerate}[label=(\alph*),left=0pt]
        \item the function $\Psi:B,\alpha\mapsto\CL(B,\alpha)+i_\CB(B)$ has the Kurdyka-\L{}ojasiewicz (KL) property,
        \item there exists $a>0$ such that for each $i\in\N$,
        \begin{equation}\label{eq:H1_AGD}
            \Psi\left(B^{i+1},\alpha^{i+1}\right)+a\left(\left\|B^{i+1}-B^i\right\|^2+\left\|\alpha^{i+1}-\alpha^i\right\|^2\right)\leq \Psi\left(B^{i},\alpha^{i}\right),
        \end{equation}
        \item there exists $b>0$ such that for each $i\in\N$, there is $g^{i+1}\in\partial\Psi\left(B^{i+1},\alpha^{i+1}\right)$ satisfying,
        \begin{equation}\label{eq:H2_AGD}
            \left\|g^{i+1}\right\|^2\leq b\left(\left\|B^{i+1}-B
            ^i\right\|^2+\left\|\alpha^{i+1}-\alpha^i\right\|^2\right),
        \end{equation}
        where $\partial\Psi$ denotes the convex subdifferential of $\Psi$ which is defined for any $\left(B,\alpha\right)\in\R^{d\times D}\times \R^{\tilde m}$ as\footnotesize
        \begin{equation*}
            \partial \Psi(B,\alpha)=\left\{ s\in\R^{d\times D}\times \R^{\tilde m}~|~\forall \left(B',\alpha'\right)\in\R^{d\times D}\times \R^{\tilde m}, \Psi\left(B',\alpha'\right)\geq \Psi\left(B,\alpha\right)+\langle s,\left(B',\alpha'\right)-\left(B,\alpha\right)\rangle\right\}.
        \end{equation*}\normalsize
        \item the sequence $\left(B^i,\alpha^i\right)_{i\in\N}$ admits a converging subsequence.
    \end{enumerate}
    Before proving each point above, we first show that the sequence $\left(\alpha^i\right)_{i\in\N}$ in Algorithm \ref{alg:MIGD_detail} is well defined and bounded. According to the assumptions of the theorem, there exists $\sigma>0$ such that for any $i\in \N$, we have
    \begin{equation*}
        \lambda_{min}\left(\matK^{B^i}_{\tilde m\tilde m}\right)\geq \sigma.
    \end{equation*}
    This ensures that the matrix $\left(\frac{1}{m}(\matK^{B^i}_{m\tilde m})^T\matK^{B^i}_{m\tilde m}+\lambda \matK^{B^i}_{\tilde m\tilde m}\right)$ is invertible at each iteration $i$ and thus that $\alpha^{i}$ is well defined. In addition, we have that for any $i\in\N$,
    \begin{equation*}
        \|\alpha^i\|\leq\left\|\left(\frac{1}{m}(\matK^{B^i}_{m\tilde m})^T\matK^{B^i}_{m\tilde m}+\lambda \matK^{B^i}_{\tilde m\tilde m}\right)^{-1}\right\|_2\left\|\matK^{B^i}_{\tilde m\tilde m}\right\|_2\|y\|\leq\frac{\left\|\matK^{B^i}_{\tilde m\tilde m}\right\|_2\|y\|}{\lambda\sigma},
    \end{equation*}
     and since $k$ is continuous (because analytic), $B^i\in\CB$ and the dataset is bounded, we can conclude that $\left\|\matK^{B^i}_{\tilde m\tilde m}\right\|_2$, and consequently the sequence $\left(\alpha^i\right)_{i\in\N}$, are bounded. In addition, $k$ is analytic and therefore continuous, twice differentiable with continuous second-order derivatives. Consequently, we can apply Lemma \ref{lem:LB}.\\
     We now prove the four properties stated at the beginning of the proof:\\
    1. The analyticity of $k$ directly ensures that $\hat\CL$ is analytic (and therefore has the KL property) in both variables. Moreover, the indicator function $i_\CB$ where $\CB=\{B\in \R^{d \times D}, \|B\|_\infty\leqslant 1\}$ is semi algebraic and also has the KL property. As a consequence, the first statement is directly satisfied.\\
    2. First, notice that since $B^0\in\CB$ and the sequence $\left(B^i\right)_{i\in\N}$ is built via the step
    \begin{equation*}
        B^{i+1}=\CP_\CB \left(B^{i}-s_{B,i}\nabla_B\hat\CL\left(B^{i},\alpha^{i}\right)\right),
    \end{equation*}
    where $s_{B,i}>0$, we have that $B^i\in\CB$ for any $i\in\N$. It follows that for any $i\in\N$ and $\alpha\in\R^{\tilde m}$, \begin{equation}\label{eq:Psi_L}
        \Psi\left(B^i,\alpha\right)=\CL\left(B^i,\alpha\right).
    \end{equation}
    Lemma \ref{lem:LB} guarantees that for each $i\in\N$, Algorithm \ref{alg:MIGD_detail} provides $s_{B,i}$ such that
    \begin{equation*}
        \hat\CL\left(B^{i+1},\alpha^{i}\right)-\hat\CL\left(B^{i},\alpha^{i}\right)<-cs_{B,i}\left\|\nabla_B\hat\CL\left(B^i,\alpha^i\right)\right\|^2.
    \end{equation*}
    Note that due to the firm non-expansiveness of $\CP_\CB$, we have that for any $i\in\N$,
    \begin{equation*}
        \left\|B^{i+1}-B^i\right\|=\left\|\CP_\CB \left(B^{i}-s_{B,i}\nabla_B\hat\CL\left(B^{i},\alpha^{i}\right)\right)-\CP_\CB\left(B^i\right)\right\|\leq s_{B,i}\left\|\nabla_B\hat\CL\left(B^{i},\alpha^{i}\right)\right\|.
    \end{equation*}
    Consequently, at any iteration $i\in\N$,
    \begin{equation}
        \hat\CL\left(B^{i+1},\alpha^{i}\right)-\hat\CL\left(B^{i},\alpha^{i}\right)<-\frac{c}{s_{B,i}}\left\|B^{i+1}-B^i\right\|^2\leq-\frac{c}{s_{max}}\left\|B^{i+1}-B^i\right\|^2.
    \end{equation}

    Since it is assumed that there exists some $\sigma>0$ such that for any $i\in\N$ we have $\lambda_{min}\left(\matK^{B^i}_{\tilde m\tilde m}\right)\geq \sigma$, we can prove that $\alpha\mapsto\hat\CL\left(B^i,\alpha\right)$ is $\lambda\sigma$-strongly convex for any $i\in\N$. From the definition of $\left(\alpha^i\right)_{i\in\N}$ and this property, we get that
    \begin{equation*}
        \hat\CL\left(B^{i+1},\alpha^{i+1}\right)+\frac{\lambda\sigma}{2}\left\|\alpha^{i+1}-\alpha^i\right\|^2\leq\hat\CL\left(B^{i+1},\alpha^{i}\right).
    \end{equation*}
    We can deduce that
    \begin{equation}
        \hat\CL\left(B^{i+1},\alpha^{i+1}\right)-\hat\CL\left(B^{i},\alpha^{i}\right)<-\frac{c}{s_{max}}\left\|B^{i+1}-B^i\right\|^2-\frac{\lambda\sigma}{2}\left\|\alpha^{i+1}-\alpha^i\right\|^2,
    \end{equation}
    which leads to the desired inequality.\\
    3. We aim at showing that for a well-chosen $b>0$, for any $i\in\N$, there exists $g^{i+1}\in\partial \Psi\left(B^{i+1},\alpha^{i+1}\right)$ (i.e. $g^{i+1}=\left(g^{i+1}_B,g^{i+1}_\alpha\right)$ where $g^{i+1}_B\in\partial_B\Psi\left(B^{i+1},\alpha^{i+1}\right)$ and $g^{i+1}_\alpha\in\partial_\alpha\Psi\left(B^{i+1},\alpha^{i+1}\right)$) such that 
    \begin{equation*}
        \left\|g^{i+1}\right\|^2\leq b\left(\left\|B^{i+1}-B
        ^i\right\|^2+\left\|\alpha^{i+1}-\alpha^i\right\|^2\right).
    \end{equation*}
    Due to the structure of $\Psi$, it is then sufficient to show that for some $v^{i+1}\in\partial i_\CB(B^{i+1})$, the choice $g^{i+1}=\left(v^{i+1}+\nabla_B\CL(B^{i+1},\alpha^{i+1}),\nabla_\alpha\CL(B^{i+1},\alpha^{i+1})\right)$ is valid for the above equation.\\
    From Algorithm \ref{alg:AGD_detail}, the sequence $\left(B^i\right)_{i\in\N}$ is defined via a step which can be seen as a proximal gradient step on $B\mapsto\CL(B,\alpha^i)+i_\CB(B)$. As a consequence, we can write that
    \begin{equation*}
        B^i-B^{i+1}-s_{B,i}\nabla_B\CL(B^i,\alpha^i)\in\partial i_\CB(B^{i+1}),
    \end{equation*}
    and we will choose $v^{i+1}=\frac{1}{s_{B,i}}\left(B^i-B^{i+1}\right)-\nabla_B\CL(B^i,\alpha^i)\in\partial i_\CB(B^{i+1})$ (due to the properties of $\partial i_\CB$ which is the normal cone onto $\CB$). It follows that
    \begin{equation}\label{eq:gi+1}
    \begin{aligned}
        \left\|g^{i+1}\right\|^2&=\left\|v^{i+1}+\nabla_B\CL(B^{i+1},\alpha^{i+1})\right\|^2+\left\|\nabla_\alpha\CL(B^{i+1},\alpha^{i+1})\right\|^2\\&=\left\|\frac{1}{s_{B,i}}\left(B^i-B^{i+1}\right)+\nabla_B\CL(B^{i+1},\alpha^{i+1})-\nabla_B\CL(B^i,\alpha^i)\right\|^2+\left\|\nabla_\alpha\CL(B^{i+1},\alpha^{i+1})\right\|^2.\end{aligned}
    \end{equation}
    Elementary computations ensure that\small
    \begin{equation*}\begin{aligned}
        \left\|\frac{1}{s_{B,i}}\left(B^i-B^{i+1}\right)+\nabla_B\CL(B^{i+1},\alpha^{i+1})-\nabla_B\CL(B^i,\alpha^i)\right\|^2\leq~&\frac{2}{{s_{B,i}}}\left\|B^{i+1}-B^i\right\|^2\\&+2\left\|\nabla_B\CL(B^{i+1},\alpha^{i+1})-\nabla_B\CL(B^i,\alpha^i)\right\|^2.\end{aligned}
    \end{equation*}\normalsize
    By using similar arguments to that in the proof of Lemma \ref{lem:LB} namely boundedness of $\left(B^i\right)_{i\in\N}$ and $\left(\alpha^i\right)_{i\in\N}$, and continuity of the second order derivatives of $k$, we can show that $\hat\CL$ is jointly Lipschitz smooth in $\left(B,\alpha\right)$ on a compact containing $\left(B^i,\alpha^i\right)_{i\in\N}$. We can deduce that there exists $L>0$ such that for any $\left(B_1,\alpha_1\right)$ and $\left(B_2,\alpha_2\right)$:
    \begin{equation*}
        \left\|\nabla_B\CL(B_1,\alpha_1)-\nabla_B\CL(B_2,\alpha_2)\right\|^2+\left\|\nabla_\alpha\CL(B_1,\alpha_1)-\nabla_\alpha\CL(B_2,\alpha_2)\right\|^2\leq L\left( \left\|B_1-B_2\right\|^2+\left\|\alpha_1-\alpha_2\right\|^2\right).
    \end{equation*}
    We can then use the above inequality and the lower bound on $s_{B,i}$ from Lemma \ref{lem:LB} to write\small
    \begin{equation*}\begin{aligned}
        \left\|\frac{1}{s_{B,i}}\left(B^i-B^{i+1}\right)+\nabla_B\CL(B^{i+1},\alpha^{i+1})-\nabla_B\CL(B^i,\alpha^i)\right\|^2\leq~ &2\left(\max\left\{\frac{L_1}{2\rho(1-c)},s_{B,-1}^{-1}\right\}+L\right)\left\|B^{i+1}-B^i\right\|^2\\
        &+2L\left\|\alpha^{i+1}-\alpha^i\right\|^2.        
    \end{aligned}
    \end{equation*}\normalsize
    In addition, since $\alpha^{i+1}$ minimizes the function $\alpha\mapsto\hat\CL\left(B^{i+1},\alpha\right)$, we directly get that
    \begin{equation*}
        \left\|\nabla_\alpha\CL(B^{i+1},\alpha^{i+1})\right\|=0.
    \end{equation*}
    Therefore, we get that 
    \begin{equation*}
        \left\|g^{i+1}\right\|^2\leq 2\left(\max\left\{\frac{L_1}{2\rho(1-c)},s_{B,-1}^{-1}\right\}+L\right)\left\|B^{i+1}-B^i\right\|^2+2L\left\|\alpha^{i+1}-\alpha^i\right\|^2,
    \end{equation*}
    which implies \eqref{eq:H2_AGD}.\\
    4. This point is trivially satisfied as both $\left(B^i\right)_{i\in\N}$ and $\left(\alpha^i\right)_{i\in\N}$ are bounded.\\
    We have proved the convergence of Algorithm \ref{alg:MIGD_detail} towards a critical point of $\Psi$. We now demonstrate that \eqref{eq:grad_MIGD} holds after $N$ iterations if $\left(B^N,\alpha^N\right)$ is not a critical point. From the definition of the method, for any $i\in\left\{0,\ldots,N\right\}$,
    \begin{equation*}
        \begin{aligned}\left\|\nabla_B\hat\CL\left(B^i,\alpha^i\right)\right\|^2&<\frac{1}{cs_{B,i}}\left(\hat\CL\left(B^i,\alpha^i\right)-\hat\CL\left(B^{i+1},\alpha^i\right)\right)\\
        &<c^{-1}\max\left\{\frac{L_1}{2\rho(1-c)},s_{B,-1}^{-1}\right\}\left(\hat\CL\left(B^i,\alpha^i\right)-\hat\CL\left(B^{i+1},\alpha^i\right)\right).\end{aligned}
    \end{equation*}
    By summing this inequality on $i\in\left\{0,\ldots,N\right\}$, we get that
    \begin{equation*}
        \begin{aligned}\sum_{i=0}^N\left\|\nabla_B\hat\CL\left(B^i,\alpha^i\right)\right\|^2&<c^{-1}\max\left\{\frac{L_1}{2\rho(1-c)},s_{B,-1}^{-1}\right\}\sum_{i=0}^N\left(\hat\CL\left(B^i,\alpha^i\right)-\hat\CL\left(B^{i+1},\alpha^i\right)\right)\\
        &\leq c^{-1}\max\left\{\frac{L_1}{2\rho(1-c)},s_{B,-1}^{-1}\right\}\hat\CL\left(B^0,\alpha^0\right):=C.\end{aligned}
    \end{equation*}
    We can then deduce \eqref{eq:grad_MIGD}.
\end{proof}

\paragraph{\textbf{Alternating Gradient Descent.}} The complete version of the scheme can be found below.

\begin{algorithm}[H]
\caption{Alternating Gradient Descent}\label{alg:AGD_detail}
\begin{algorithmic}
    \STATE \textbf{Require} $B^0\in \CB$, $\alpha^0\in \R^{\tilde m}$, $s_{B,-1}>0$, $s_{\alpha,-1}>0$, $s_{max}>0$, $\rho\in(0,1)$, $\delta\in(0,1]$, $c>0$ and $n_\alpha\in\N^*$.
    \FOR{$i$ in $0,1,...$}
    \STATE $s_{B,i} = \min\left\{\frac{s_{B,i-1}}{\rho\delta},s_{max}\right\}$
    \REPEAT
    \STATE $s_{B,i} = \rho s_{B,i}$
    \STATE $B^{i+1}=\CP_\CB \left(B^{i}-s_{B,i}\nabla_B\hat\CL\left(B^{i},\alpha^{i}\right)\right)$
    \UNTIL $\hat\CL\left(B^{i+1},\alpha^{i}\right)-\hat\CL\left(B^{i},\alpha^{i}\right)<-cs_{B,i}\left\|\nabla_B\hat\CL\left(B^i,\alpha^i\right)\right\|^2$
    \STATE $\alpha^{i,0}=\alpha^i$
    \STATE $s_{\alpha,i,-1} = s_{\alpha,i-1}$
    \FOR{$j$ in $0,1,...,n_\alpha-1$}
    \STATE $s_{\alpha,i,j} = \min\left\{\frac{s_{\alpha,i,j-1}}{\rho\delta},s_{max}\right\}$
    \REPEAT
    \STATE $s_{\alpha,i,j} = \rho s_{\alpha,i,j}$
    \STATE $\alpha^{i,j+1}=\alpha^{i,j}-s_{\alpha,i,j}\nabla_\alpha\hat\CL\left(B^{i+1},\alpha^{i,j}\right)$
    \UNTIL $\hat\CL\left(B^{i+1},\alpha^{i,{j+1}}\right)-\hat\CL\left(B^{i},\alpha^{i,j}\right)<-cs_{\alpha,i,j}\left\|\nabla_\alpha\hat\CL\left(B^i,\alpha^{i,j}\right)\right\|^2$
    \ENDFOR
    \STATE $s_{\alpha,i} = s_{\alpha,i,n_\alpha-1}$
    \STATE $\alpha^{i+1} = \alpha^{i,n_\alpha-1}$
    \ENDFOR
    \RETURN $\left(B^{i+1}, \alpha^{i+1}\right)$
\end{algorithmic}
\end{algorithm}

Similarly to Algorithm \ref{alg:MIGD_detail}, Algorithm \ref{alg:AGD_detail} leverages a backtracking strategy to set both $\left(s_{B,i}\right)_{i\in\N}$ and $\left(s_{\alpha,i,j}\right)_{(i,j)\in\N\times\{0,\ldots,n_\alpha-1\}}$. In addition to Lemma \ref{lem:LB}, we introduce the following lemma guaranteeing that these sequences are well-defined.

\begin{lem}\label{lem:Lalpha}
    Let $k(\cdot,\cdot)$ be a continuous kernel and $\left(B^i\right)_{i\in\N}$ such that $B^i\in\CB$ for any $i\in\N$. Then, there exists $L_2>0$ such that,
    \begin{equation}\label{eq:L_2}
        \forall \alpha\in \R^{\tilde m},~\forall i\in\N,\quad \left\|\nabla^2_\alpha\hat\CL\left(B^i,\alpha\right)\right\|_{2}\leq L_2,
    \end{equation}
    implying that for any $i\in\N$, $\alpha\mapsto\hat\CL\left(B^i,\alpha\right)$ is $L_2$-smooth.\\More precisely, $\left\|\nabla^2_\alpha\hat\CL\left(B^i,\alpha\right)\right\|_{2}=\left\|\frac{1}{m}(\matK^{B^i}_{m\tilde m})^T\matK^{B^i}_{m\tilde m}+\lambda \matK^{B^i}_{\tilde m\tilde m}\right\|_2\leq L_2$. It follows that the linesearch procedure used to generate $\left(s_{\alpha,i,j}\right)_{i,j\in\N\times\{1,\ldots,n_\alpha\}}$ in Algorithm \ref{alg:AGD_detail} is well defined, and $s_{\alpha,i,j}\geq\min\left\{\frac{2\rho(1-c)}{L_2},s_{\alpha,-1}\right\}$ for any $i\in\N$ and $j\in \{0,\ldots,n_\alpha-1\}$.
\end{lem}
\begin{proof}[Proof of Lemma \ref{lem:Lalpha}]
    Elementary computations show that for any $\alpha\in\R^{\tilde m}$ and $B\in\CB$,
    \begin{equation*}
        \nabla^2_\alpha\hat\CL\left(B,\alpha\right) = \frac{1}{m}(\matK^{B}_{m\tilde m})^T\matK^{B}_{m\tilde m}+\lambda \matK^{B}_{\tilde m\tilde m}
    \end{equation*}
    Notice that $\matK^{B}_{m\tilde m}$ and $\matK^{B}_{\tilde m\tilde m}$ are submatrices of $\matK^{B}$ defined as $\left(\matK^{B}\right)_{i,j} = k\left(Bx_i,Bx_j\right)$. Since $\matK^{B}$ is a positive semi-definite matrix, we have that $\|\matK^{B}_{m\tilde m}\|_2\leq\|\matK^{B}\|_2$ and $\|\matK^{B}_{\tilde m\tilde m}\|_2\leq\|\matK^{B}\|_2$.\\
    We can then use the continuity of $k$, the boundedness of $\left(x_i\right)_{i=1}^m$ and the inequality $\|B\|_\infty\leq1$ (see more details in the proof of Lemma \ref{lem:LB}) to show that there exists $C<+\infty$ such that $\|\matK^{B}\|_2\leq C$. This leads to \eqref{eq:L_2} with $L_2 = \frac{C^2}{m}+\lambda C$.\\
    Since $B^i\in\CB$ for any $i\in\N$, the Lipschitz continuity of $\alpha\mapsto\hat\CL\left(B^i,\alpha\right)$ directly ensures that the Armijo backtracking procedure in Algorithm \ref{alg:AGD_detail} is well defined. In particular, the Armijo condition is satisfied for any step size $s\leq\frac{2(1-c)}{L_2}$. Since the $s_{\alpha,i,j}$ is updated by multiplying it by $\rho$ until the condition is satisfied, we can deduce the desired lower bound.
\end{proof}
\begin{rem}[Backtracking strategy on $\left(s_{\alpha,i,j}\right)_{(i,j)\in \N\times\left\{0,\ldots,n_\alpha-1\right\}}$]
For any iteration $i\in\N$, the function $\alpha\mapsto \hat\CL\left(B^i,\alpha\right)$ is $L_i$-Lipschitz where $L_i =\left\|\frac{1}{m}(\matK^{B^i}_{m\tilde m})^T\matK^{B^i}_{m\tilde m}+\lambda \matK^{B^i}_{\tilde m\tilde m}\right\|_2$ can be computed directly. It is therefore possible to replace the linesearch procedure for setting $\left(s_{\alpha,i,j}\right)_{(i,j)\in \N\times\left\{0,\ldots,n_\alpha-1\right\}}$ by the rule $s_{\alpha,i,j} = 1/L_i$ in Algorithm \ref{alg:AGD_detail} (the convergence guarantees would be the same). It is worth noticing that this strategy involves computing the largest eigenvalue of $\frac{1}{m}(\matK^{B^i}_{m\tilde m})^T\matK^{B^i}_{m\tilde m}+\lambda \matK^{B^i}_{\tilde m\tilde m}$ which can be costly depending on the parameter $\tilde m$.
\end{rem}

By applying the same strategy as that of the proof of Theorem \ref{thm:cvg_MIGD}, we demonstrate similar convergence properties of Algorithm \ref{alg:AGD_detail}. 

\begin{thm}\label{thm:cvg_AGD}
    Let $k(\cdot,\cdot)$ be an analytic kernel. Let $\left(B^i\right)_{i\in\N}$ and $\left(\alpha^i\right)_{i\in\N}$ be the sequences generated by Algorithm \ref{alg:AGD_detail} and suppose that $\left(\alpha^i\right)_{i\in\N}$ is bounded. Then, \textbf{the sequence $\left(B^i,\alpha^i\right)_{i\in\N}$ converges to a critical point} of $\Psi:B,\alpha\mapsto\CL(B,\alpha)+i_\CB(B)$ as $i$ goes to infinity. In addition, the sequences $\left(B^i\right)_{i\in\N}$ and $\left(\alpha^i\right)_{i\in\N}$ have finite length, i.e.
    \begin{equation*}
        \sum_{i=0}^{+\infty}\|B^{i+1}-B^i\|<+\infty,\quad\sum_{i=0}^{+\infty}\|\alpha^{i+1}-\alpha^i\|<+\infty,
    \end{equation*}
    and there exists $C>0$ such that after $N$ iterations, $\left(B^N,\alpha^N\right)$ is a critical point of $\Psi$ or 
    \begin{equation}\label{eq:grad_AGD}
        \min_{0\leq i\leq N}\left\|\nabla_B\hat\CL\left(B^i,\alpha^i\right)\right\|^2\leq\frac{C}{N}.
    \end{equation}
\end{thm}
\begin{proof}[Proof of Theorem \ref{thm:cvg_AGD}]
    Similarly to the proof of Theorem \ref{thm:cvg_MIGD}, we adapt \cite[Theorem~2.9]{attouch2013convergence} to our framework to show the desired convergence results. Therefore, we need to show the $4$ assertions enumerated in the aforementioned proof.\\
    1. The Kurdyka-\L{}ojasiewicz property of $\hat\CL$ is already shown in the proof of Theorem \ref{thm:cvg_MIGD}.\\
    2. As stated in the proof of Theorem \ref{thm:cvg_MIGD}, it is trivial that $B^i\in\CB$ for any $i\in\N$, and consequently for any $\alpha\in\R^{\tilde m}$, \begin{equation}\label{eq:psi_L2}\Psi\left(B^i,\alpha\right)=\CL\left(B^i,\alpha\right).
    \end{equation}
    In addition, similar computations allow to show that
    \begin{equation}
        \hat\CL\left(B^{i+1},\alpha^{i}\right)-\hat\CL\left(B^{i},\alpha^{i}\right)<-\frac{c}{s_{max}}\left\|B^{i+1}-B^i\right\|^2.
    \end{equation}
    Lemma \ref{lem:Lalpha} ensures that at each step $i\in\N$ and substep $j\in\{1,\ldots,n_\alpha\}$,
    \begin{equation}\label{eq:ineq1_B2}
        \hat\CL\left(B^{i+1},\alpha^{i,{j+1}}\right)-\hat\CL\left(B^{i},\alpha^{i,j}\right)<-cs_{\alpha,i,j}\left\|\nabla_\alpha\hat\CL\left(B^i,\alpha^{i,j}\right)\right\|^2,
    \end{equation}
    which directly implies that
    \begin{equation*}
        \hat\CL\left(B^{i+1},\alpha^{i,{j+1}}\right)-\hat\CL\left(B^{i},\alpha^{i,j}\right)<-\frac{c}{s_{\alpha,i,j}}\left\|\alpha^{i,j+1}-\alpha^{i,j}\right\|^2\leq -\frac{c}{s_{max}}\left\|\alpha^{i,j+1}-\alpha^{i,j}\right\|^2.
    \end{equation*}
    Since $\alpha^i=\alpha^{i,0}$ and $\alpha^{i+1}=\alpha^{i,n_\alpha-1}$, we get that
    \begin{equation*}
        \hat\CL\left(B^{i+1},\alpha^{i+1}\right)-\hat\CL\left(B^{i+1},\alpha^{i}\right)<-\frac{c}{s_{max}}\sum_{j=0}^{n_\alpha-2}\left\|\alpha^{i,j+1}-\alpha^{i,j}\right\|^2,
    \end{equation*}
    and since for any sequence $(x^i)_{i\in\N}$, $\left\|\sum_{i=1}^nx^i\right\|^2\leq n \sum_{i=1}^n\left\|x^i\right\|^2$,
    \begin{equation}\label{eq:ineq1_alpha}
        \hat\CL\left(B^{i+1},\alpha^{i+1}\right)-\hat\CL\left(B^{i+1},\alpha^{i}\right)<-\frac{c}{s_{max}(n_\alpha-1)}\left\|\alpha^{i+1}-\alpha^{i}\right\|^2.
    \end{equation}
    From \eqref{eq:ineq1_B2} and \eqref{eq:ineq1_alpha}, we can prove that
    \begin{equation*}
        \hat\CL\left(B^{i+1},\alpha^{i+1}\right)-\hat\CL\left(B^{i},\alpha^{i}\right)<-\frac{c}{s_{max}}\left(\left\|B^{i+1}-B^i\right\|^2-\frac{1}{n_\alpha-1}\left\|\alpha^{i+1}+\alpha^{i}\right\|^2\right),
    \end{equation*}
    which leads to the desired conclusion, taking $a=\frac{c}{s_{max}(n_\alpha-1)}$ and using \eqref{eq:psi_L2}.\\
    3. We aim at showing that for a well-chosen $b>0$, for any $i\in\N$, there exists $g^{i+1}\in\partial \Psi\left(B^{i+1},\alpha^{i+1}\right)$ (i.e. $g^{i+1}=\left(g^{i+1}_B,g^{i+1}_\alpha\right)$ where $g^{i+1}_B\in\partial_B\Psi\left(B^{i+1},\alpha^{i+1}\right)$ and $g^{i+1}_\alpha\in\partial_\alpha\Psi\left(B^{i+1},\alpha^{i+1}\right)$) such that 
    \begin{equation*}
        \left\|g^{i+1}\right\|^2\leq b\left(\left\|B^{i+1}-B
        ^i\right\|^2+\left\|\alpha^{i+1}-\alpha^i\right\|^2\right).
    \end{equation*}
    By taking $g^{i+1}=\left(v^{i+1}+\nabla_B\CL(B^{i+1},\alpha^{i+1}),\nabla_\alpha\CL(B^{i+1},\alpha^{i+1})\right)$ with $v^{i+1}=\frac{1}{s_{B,i}}\left(B^i-B^{i+1}\right)-\nabla_B\CL(B^i,\alpha^i)$, we can apply the same reasoning as that of the proof of Theorem \ref{thm:cvg_MIGD} to demonstrate that there exists $L>0$ such that:
    \begin{equation}\label{eq:ineq_g}\begin{aligned}
        \left\|g^{i+1}\right\|^2\leq&~2\left(\max\left\{\frac{L_1}{2\rho(1-c)},s_{B,-1}^{-1}\right\}+L\right)\left\|B^{i+1}-B^i\right\|^2+2L\left\|\alpha^{i+1}-\alpha^i\right\|^2\\&+\left\|\nabla_\alpha\CL(B^{i+1},\alpha^{i+1})\right\|^2.\end{aligned}
    \end{equation}
    By rewriting the second term of the above inequality,
    \begin{equation*}
        \left\|\nabla_\alpha\CL(B^{i+1},\alpha^{i+1})\right\|^2\leq\left\|\nabla_\alpha\CL(B^{i+1},\alpha^{i+1})-\nabla_\alpha\CL(B^{i+1},\alpha^{i,n_{\alpha-1}})\right\|^2+\left\|\nabla_\alpha\CL(B^{i+1},\alpha^{i,n_{\alpha-1}})\right\|^2.
    \end{equation*}
    From the joint Lipschitz smoothness of $\hat \CL$ and the definition of the sequence $\left(\alpha^i\right)_{i\in\N}$, we get that
    \begin{equation*}
        \begin{aligned}\left\|\nabla_\alpha\CL(B^{i+1},\alpha^{i+1})\right\|^2&\leq2\left(L+\frac{1}{s_{\alpha,i,n_\alpha-1}}\right)\left\|\alpha^{i+1}-\alpha^i\right\|^2\\
        &\leq2\left(L+\max\left\{\frac{L_2}{2\rho(1-c)},s_{\alpha,-1}^{-1}\right\}\right)\left\|\alpha^{i+1}-\alpha^i\right\|^2,\end{aligned}
    \end{equation*}
    where we use the lower bound on $s_{\alpha,i,j}$ from Lemma \ref{lem:Lalpha}.
    Combining the above inequalities, we can conclude that
    \begin{equation*}
    \begin{aligned}
        \left\|g^{i+1}\right\|^2\leq &~2\left(L+\max\left\{\frac{L_1}{2\rho(1-c)},s_{B,-1}^{-1}\right\}\right)\left\|B^{i+1}-B^i\right\|^2\\&+2\left(2L+\max\left\{\frac{L_2}{2\rho(1-c)},s_{\alpha,-1}^{-1}\right\}\right)\left\|\alpha^{i+1}-\alpha^i\right\|^2
    \end{aligned}
    \end{equation*}
    which implies the desired inequality for $$b=\max\left\{2\left(L+\max\left\{\frac{L_1}{2\rho(1-c)},s_{B,-1}^{-1}\right\}\right),2\left(2L+\max\left\{\frac{L_2}{2\rho(1-c)},s_{\alpha,-1}^{-1}\right\}\right)\right\}.$$
    4. This point is trivially satisfied as both $\left(B^i\right)_{i\in\N}$ and $\left(\alpha^i\right)_{i\in\N}$ are bounded.\\
    We demonstrated above that the method converges to a critical point. Inequality (\ref{eq:grad_AGD}) can be obtained using the same computations as in the proof of Theorem \ref{thm:cvg_MIGD}.
\end{proof}

\begin{rem}[On the boundedness of $\left(\alpha^i\right)_{i\in\N}$]
Theorem \ref{thm:cvg_AGD} relies on a boundedness assumption on the sequence $\left(\alpha^i\right)_{i\in\N}$. This can be enforced in different ways:
\begin{enumerate}
    \item by applying the same hypothesis as in Theorem \ref{thm:cvg_MIGD}: suppose that for any $i\in\N$, $\lambda_{min}\left(\matK^{B^i}_{\tilde m\tilde m}\right)\geq \sigma>0$. The boundedness of the sequence can be obtained directly as Algorithm \ref{alg:AGD_detail} is a descent method and consequently the term $\lambda(\alpha^i)^T\matK^{B^i}_{\tilde m \tilde m}\alpha^i$ can not grow indefinitely.
    \item by directly adding a constraint on $\alpha$ to the problem. The convergence analysis for Algorithm \ref{alg:AGD_detail} would remain the same.
\end{enumerate}
\end{rem}

\begin{rem}[On $\delta$ and $s_{max}$]
    Algorithm \ref{alg:MIGD_detail} and Algorithm \ref{alg:AGD_detail} involve a non-monotone backtracking procedure for defining the learning rates. The non-monotonicity, which occurs when $\delta < 1$, allows for more aggressive step sizes that better adapt to the local geometry of the objective function. For technical reasons, the proofs require to set a maximum learning-rate $s_{max}>0$ which, in practice, is not necessary. Note that in the monotone case, i.e. $\delta = 1$, the sequences of learning-rates are directly bounded by the initial value.
\end{rem}

\section{Supplemental material on numerical experiments}\label{app:num_exp}

\subsection{Experiments setting}\label{app:setting_exp}

\paragraph{\textbf{Experimental setting.}} The experiments presented in Figure \ref{fig:losses_col}, Figure \ref{fig:path_col1}, Figure \ref{fig:path_col2} and Appendix \ref{app:2d_case} were performed in Python on a 2,4 GHz Intel Core i5 quad-core laptop with 8 Gb of RAM. The remaining experiments were performed on a server on \#99-Ubuntu SMP with 2 x AMD EPYC 7301 16-Core Processor and 256 Gb of RAM.\\

\paragraph{\textbf{Datasets.}} Two synthetic datasets are generated for the experiments presented in the paper. The \textbf{first dataset} is generated by sampling $x_i\sim\CU\left([-1,1]\right)$ and setting the output as follows:
\begin{equation}\label{eq:ds1}
    y_i=\underbrace{\sum_{j=1}^d\sin\left(\left(1+\frac{j}{d}\right)\pi \left(Bx_i\right)_j\right)}_{:=z_i}+\varepsilon_i,\quad\varepsilon_i\sim\CN\left(0,\sigma^2\right),
\end{equation}
where $\sigma^2$ depends on the variance of $\left(z_i\right)_{i\in\{1,\ldots,m\}}$ (i.e. $\sigma^2=\text{Var}(z)/100$). The matrix $B$ is sampled randomly in $\CB$ (for any $(i,j)\in\{1,\ldots,d\}\times\{1,\ldots,D\}$, $B_{i,j}\sim\CU([0,1])$).\\
The \textbf{second dataset} is generated taking $x_i\in\CU([-10,10])$ and
\begin{equation}\label{eq:ds2}
    y_i = \underbrace{\sum_{j=1}^d\sin\left(0.5\left(Bx_i\right)_j-j\right)+\frac{1}{2}\left(Bx_i\right)_{j+1}\cos\left(0.4\left(Bx_i\right)_{j+2}-j+1\right)}_{:=z_i}+\varepsilon_i,
\end{equation}
where $\varepsilon_i\sim\CN(0,\sigma^2)$ and $\sigma^2=\text{Var}(z)/100$. We sample $B$ randomly as done for the first dataset.

\paragraph{\textbf{Methodology.}} 
The convergence graphs of Figure \ref{fig:losses_col} were obtained by running VarPro and AGD on the first dataset introduced above, setting $D=300$ and $d=2$, and the true value of $B$ manually to $B_{i,j}=1$ for $(i,j)\in\{(1,1),(2,2)\}$, i.e. selecting only the first two components. We used $\tilde m =5$ Nystr\"om centers from the training set of size $m=300$.

The experiments presented at the top of Figure \ref{fig:varyingd_combined} were performed using VarPro and AGD on $5$ Nyström centers for $\lambda=10^{-7}$ and adjusting the parameter $\gamma$ to the initial matrix $B^0$ (sampled randomly in $\CB$):
\begin{equation*}\gamma=\frac{1}{2\tilde \mu^2},\end{equation*}
where $\tilde \mu = \text{median}\{\|B^0(x_i-x_j)\|,i\neq j\}$. We then solve HKRR without the Nystr\"om approximation setting $B$ as the approximation given by the method, and setting the parameters $\gamma$ and $\lambda$ with cross-validation (on a validation set of size $m=600$). The resulting $B$ and $\alpha$ are used to compute the R2 score on a test set of size $8m=4800$. For the graph on the left, we consider the first dataset with $D=50$ and $d_*=3$, while the left one is obtained based on the second dataset with $D=50$ and $d_*=3$. The methods are stopped after $60$ seconds of computations and run $5$ times per set of parameters. 

The bottom graph in Figure \ref{fig:varyingd_combined} was obtained by applying the same process, running AGD and VarPro on $\tilde m = 25$ Nystr\"om centers and setting $\lambda=10^{-8}$. For each dataset size $m$ in $\{1000,4000,8000,12000,16000\}$, we performed each algorithm $10$ times, stopping after $80$ seconds.

\subsection{A 2D example for alternating minimization}\label{app:2d_case}

Recall that we consider the function $f:x,y\mapsto\left(x-y^2\right)^2+\cos(\pi y)+(1-y)^2+1$. Despite its simplicity, this function shares several similarities with $\hat\CL:B,\alpha\mapsto\frac{1}{m} \left\|\matK^B_{m\tilde m}\alpha-\mathbf{y}\right\|^2 + \lambda \alpha^T\matK_{\tilde m\tilde m}^B\alpha$:
\begin{itemize}
    \item it is strongly convex w.r.t. its first variable $x$ and minimizing $x\mapsto f(x,y)$ can be done directly ($x^*=y^2$).
    \item it is non convex in its second variable, with potentially local minimizers.
\end{itemize}
Because of this structural similarity to the HKRR objective function, we use Variable Projection (VarPro) and Alternating Gradient Descent (AGD) to minimize $f$. Note that $f$ has a unique global minimum in $\left(x^*,y^*\right)=(1,1)$.

Figure \ref{fig:path_y2} shows the behavior of both methods applied to $f$. For the chosen initialization point, it highlights the advantage of taking in account the geometry of $f$ in both $x$ and $y$ since VarPro provides a local minimizer while AGD goes to a global one. This phenomenon occurs not only for cherry-picked initialization points as shown in Figure \ref{fig:map_y2}: the attraction basin of the global minimizer of $f$ is significantly larger for AGD on this function.

\begin{figure}[ht]
\centering
\includegraphics[width=0.35\textwidth]{2dpath_y2.png}~
\includegraphics[width=0.25\textwidth]{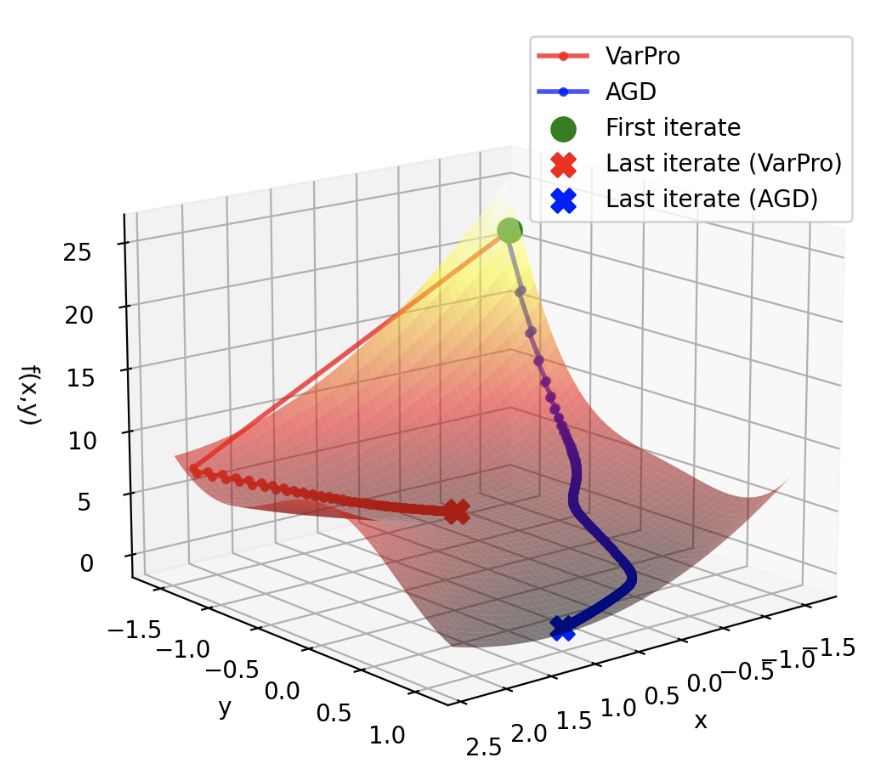}
\includegraphics[width=0.35\textwidth]{loss_y2.png}
\caption{Left and center: Trajectories of the iterates of VarPro (in red) and AGD (in blue) for $f:x,y\mapsto\left(x-y^2\right)^2+\cos(\pi y)+(1-y)^2+1$ (taking $(x_0,y_0)=(-1.5,-1.5)$). Right: Value of the loss function w.r.t. the number of iterations.}
\label{fig:path_y2}
\end{figure}

\begin{figure}[ht]
\centering
\includegraphics[width=0.4\textwidth]{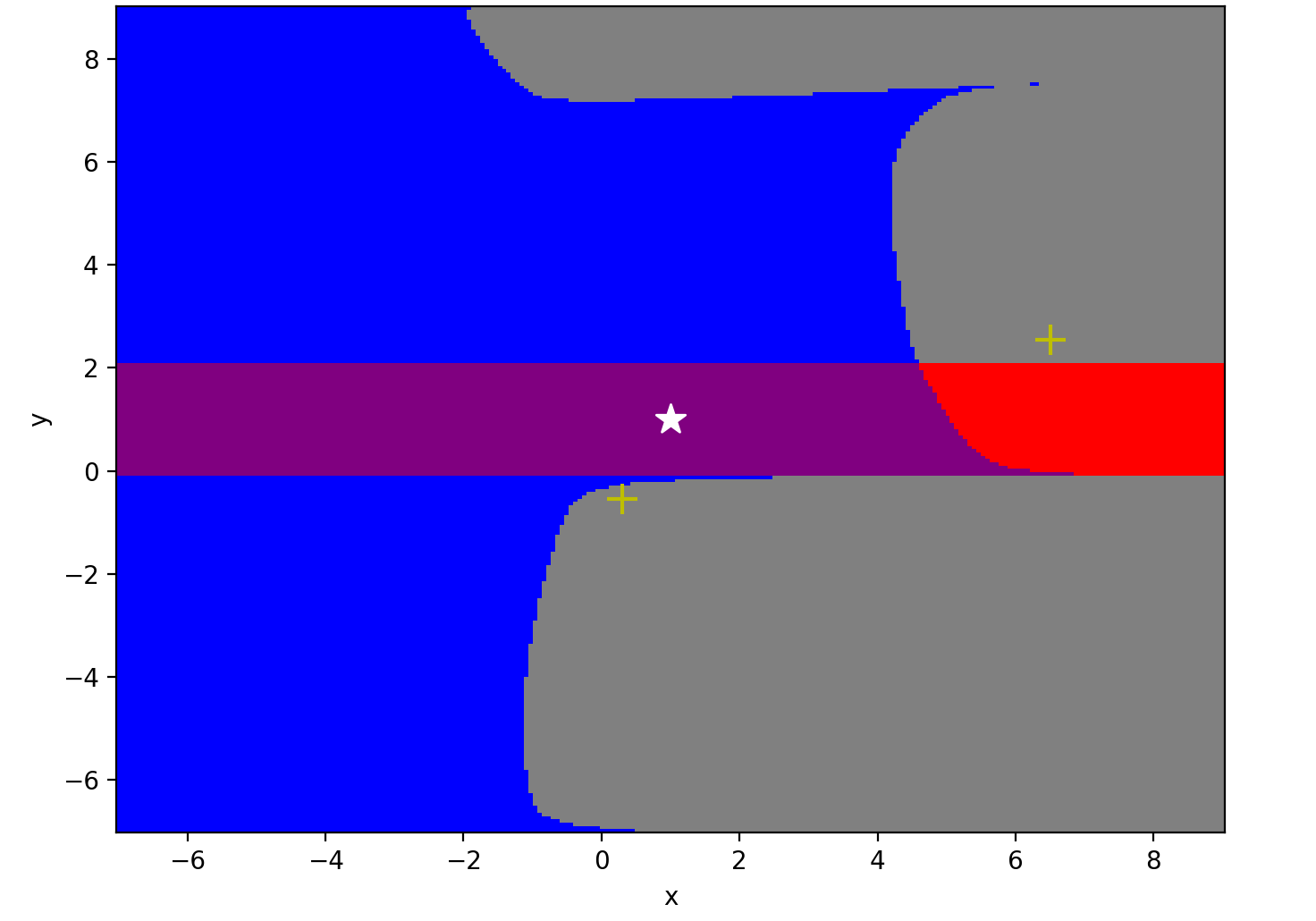}
\caption{Convergence map of VarPro and AGD for minimizing $f:x,y\mapsto\left(x-y^2\right)^2+\cos(\pi y)+(1-y)^2+1$. Purple = both methods converge to the global minimum from the corresponding initialization point; Red =  only VarPro converges to the global minimum; Blue = only AGD converges to the global minimum; Gray = no method converges to the global minimum. The white star is the global minimizer of $f$ and the yellow '+' crosses are local minimizers.}
\label{fig:map_y2}
\end{figure}

We can observe a similar behavior on the function $f:x,y\mapsto\left(x-\sigma(y)\right)^2+\cos(\pi y)+(1-y)^2+1$ where $\sigma:y\mapsto \frac{1}{1+e^{-y}}$ is the sigmoid function as illustrated in Figure \ref{fig:path_sigmoid} and \ref{fig:map_sigmoid}.

\begin{figure}[H]
\centering
\includegraphics[width=0.35\textwidth]{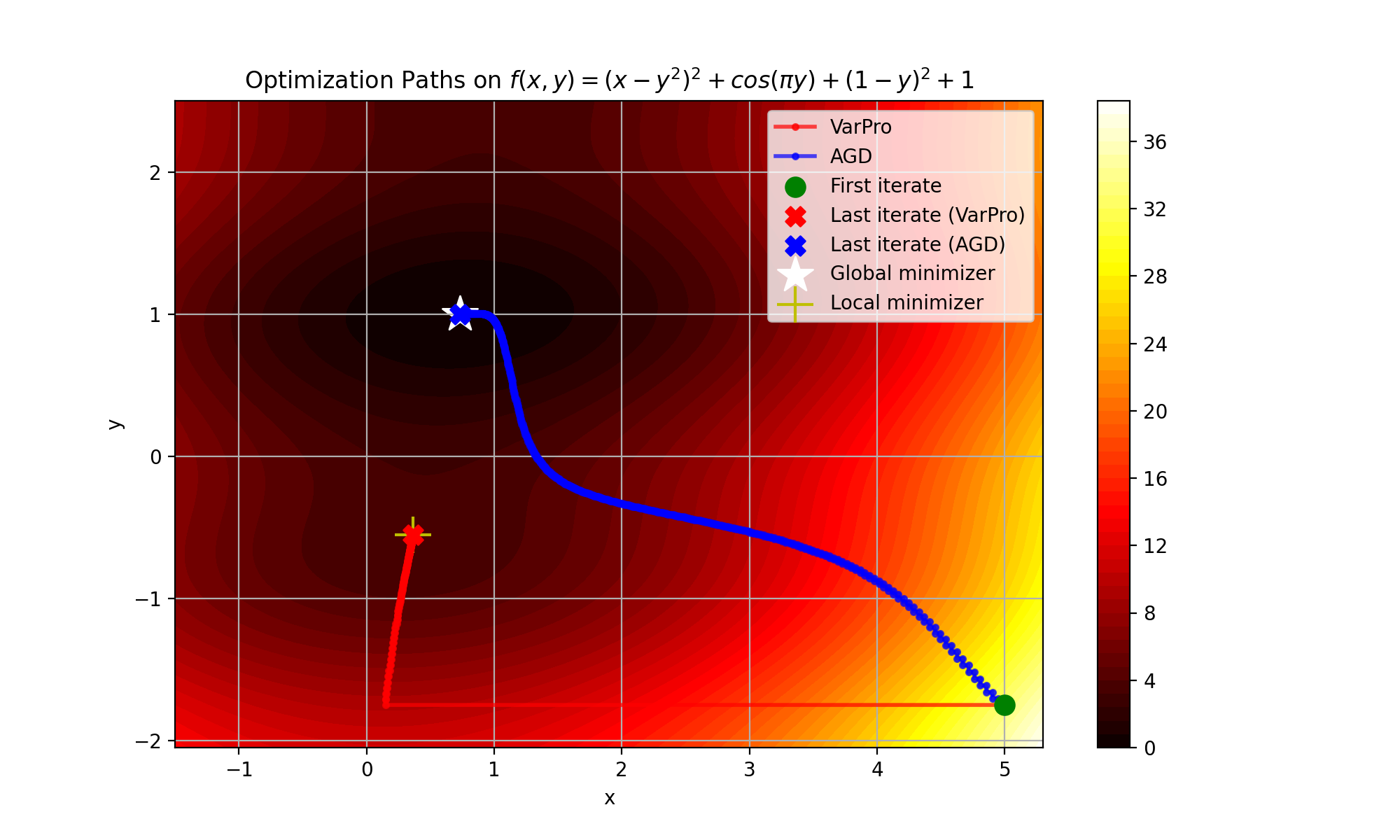}~
\includegraphics[width=0.25\textwidth]{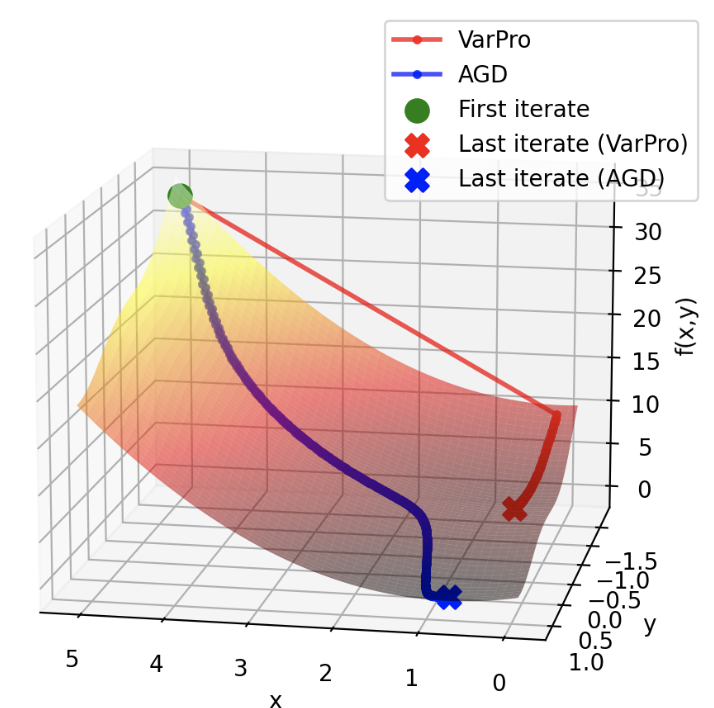}
\includegraphics[width=0.35\textwidth]{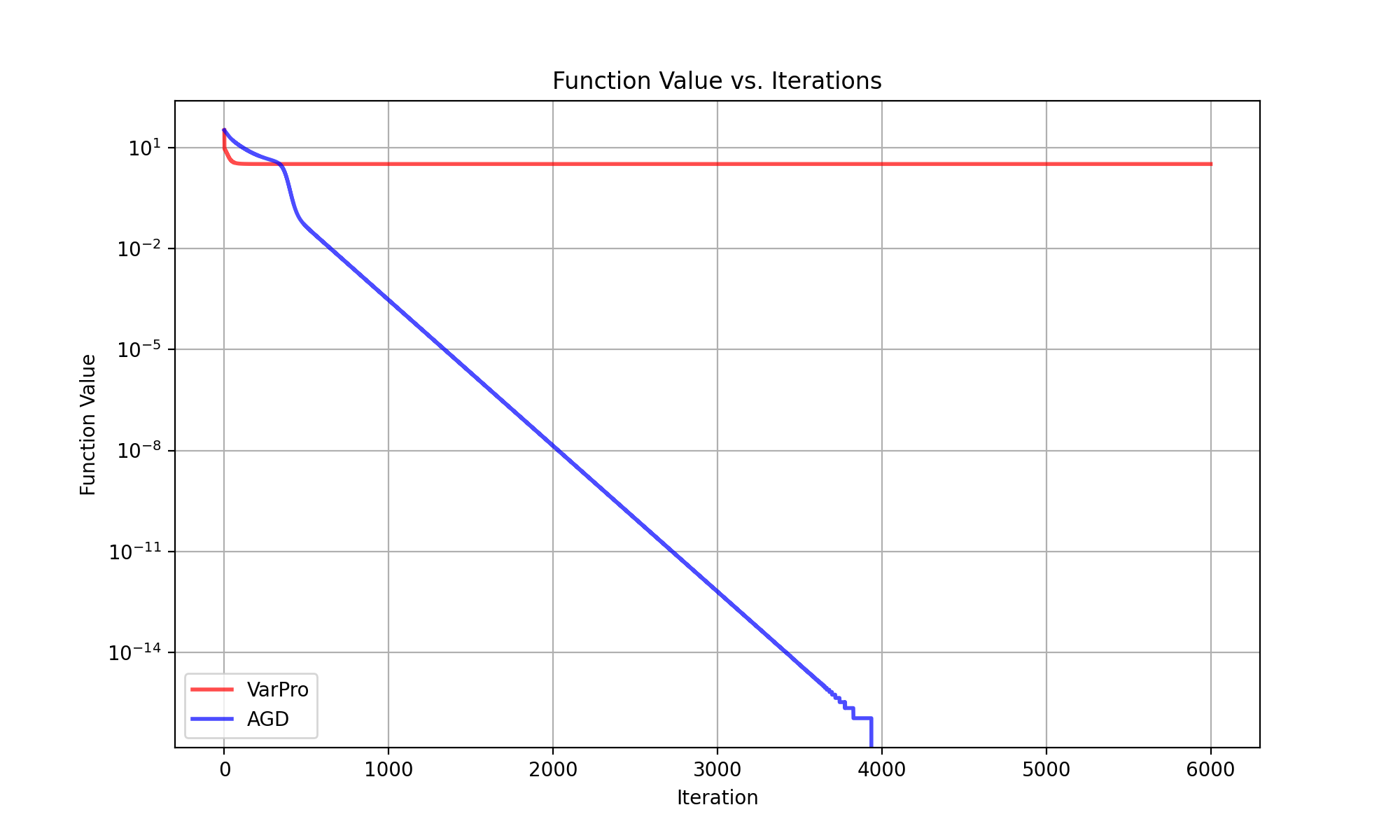}
\caption{Left and center: Trajectories of the iterates of VarPro (in red) and AGD (in blue) for $f:x,y\mapsto\left(x-\sigma(y)\right)^2+\cos(\pi y)+(1-y)^2+1$ (taking $(x_0,y_0)=(5,-1.75)$). Right: Value of the loss function w.r.t. the number of iterations.}
\label{fig:path_sigmoid}
\end{figure}

\begin{figure}[H]
\centering
\includegraphics[width=0.4\textwidth]{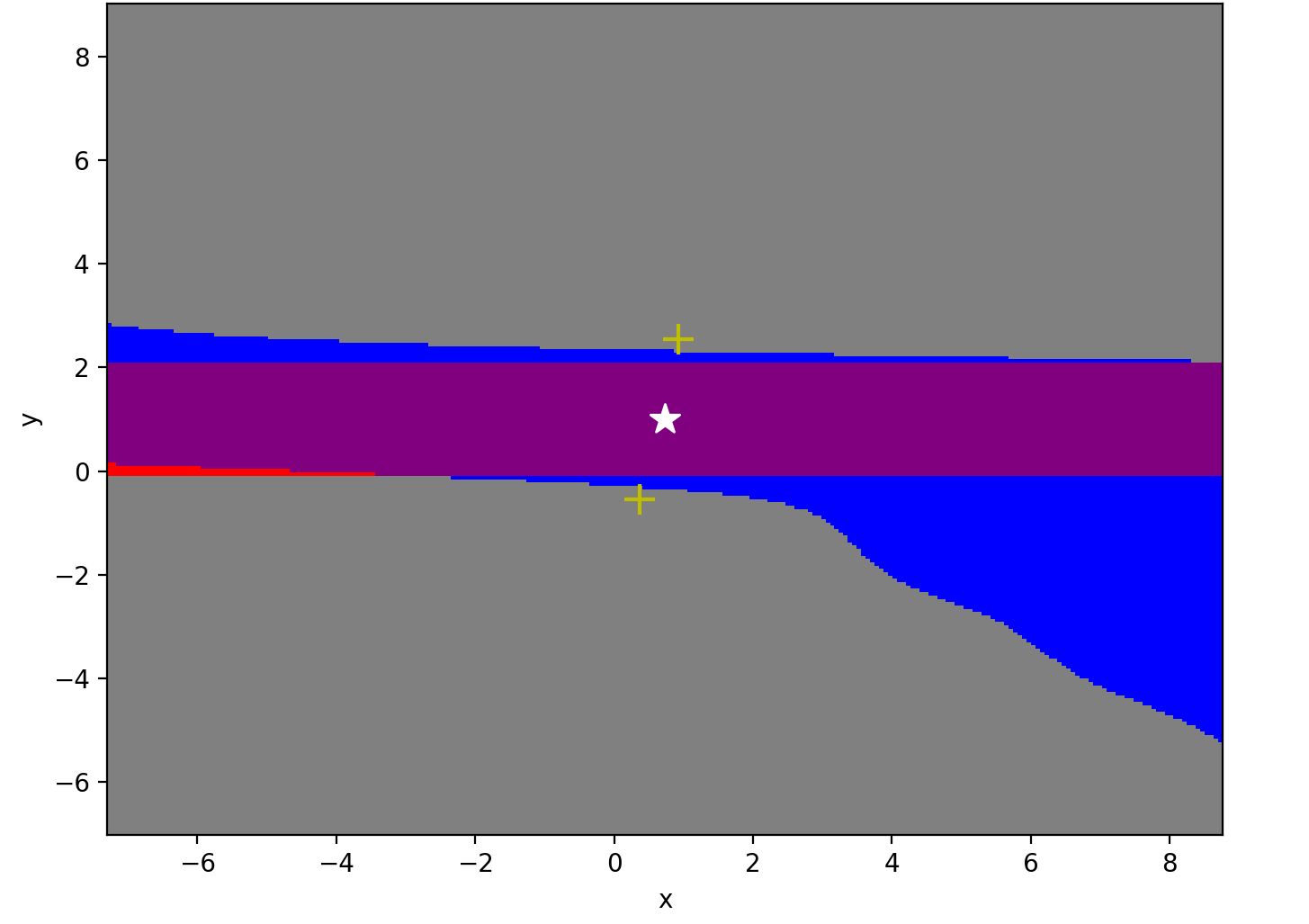}
\caption{Convergence map of VarPro and AGD for minimizing $f:x,y\mapsto\left(x-\sigma(y)\right)^2+\cos(\pi y)+(1-y)^2$. Purple = both methods converge to the global minimum from the corresponding initialization point; Red =  only VarPro converges to the global minimum; Blue = only AGD converges to the global minimum; Gray = no method converges to the global minimum. The white star is the global minimizer of $f$ and the yellow '+' crosses are local minimizers.}
\label{fig:map_sigmoid}
\end{figure}

\end{document}